\documentclass{article}

\usepackage{microtype}
\usepackage{graphicx}
\usepackage{subcaption}
\usepackage{booktabs} % for professional tables
\usepackage{multirow}
\usepackage{bbm}
\usepackage{hyperref}
\usepackage{enumitem}
\usepackage[ruled,vlined]{algorithm2e}
\usepackage{amsmath,amssymb,bm}

\newcommand{\supp}{\operatorname{supp}}

\usepackage[accepted]{icml2026}

\usepackage{amsmath}
\usepackage{amssymb}
\usepackage{mathtools}
\usepackage{amsthm}

\usepackage[capitalize,noabbrev]{cleveref}

\theoremstyle{plain}
\newtheorem{theorem}{Theorem}[section]

\newtheorem{lemma}[theorem]{Lemma}

\theoremstyle{definition}

\newtheorem{assumption}[theorem]{Assumption}
\theoremstyle{remark}

\usepackage[textsize=tiny]{todonotes}

\icmltitlerunning{Sparse Bayesian Deep Functional Learning with Structured Region Selection}

\begin{document}

\twocolumn[
  \icmltitle{Sparse Bayesian Deep Functional Learning with Structured Region Selection}



  \icmlsetsymbol{equal}{*}

  \begin{icmlauthorlist}
    \icmlauthor{Xiaoxian Zhu}{yyy}
    \icmlauthor{Yingmeng Li}{yyy}
    \icmlauthor{Shuangge Ma}{comp}
    \icmlauthor{Mengyun Wu}{yyy}
  \end{icmlauthorlist}

  \icmlaffiliation{yyy}{School of Statistics and Data Science, Shanghai University of Finance and Economics}
  \icmlaffiliation{comp}{Department of Biostatistics, Yale School of Public Health}
  %\icmlaffiliation{sch}{School of ZZZ, Institute of WWW, Location, Country}

  \icmlcorrespondingauthor{Mengyun Wu}{wu.mengyun@mail.shufe.edu.cn}
  %\icmlcorrespondingauthor{Firstname2 Lastname2}{first2.last2@www.uk}

  \icmlkeywords{Functional Data Analysis, Sparse Bayesian Learning, Region Selection, Deep Neural Networks, Model Interpretability}

  \vskip 0.3in
]



\printAffiliationsAndNotice{}  % no special notice (required even if empty)

\begin{abstract}
% With the rapid advancement of technology, increasingly complex structured and continuous data are being generated across many application domains. Such data are naturally suited for modeling within the functional data analysis framework. However, conventional functional regression methods often rely on linear assumptions, which struggle to capture complex nonlinear relationships between functional predictors and scalar responses. Moreover, while local sparsity is common in practical problems, most existing deep neural network (DNN) based methods do not support interpretable region selection. To address these gaps, this paper proposes a sparse Bayesian functional deep neural network framework for functional predictors. The framework integrates B-spline representations into a Bayesian DNN, enabling both nonlinear modeling and region selection. We introduce a group-normal spike-and-slab prior on the first hidden layer to identify active subregions of functional effects, thereby enhancing interpretability. {\color{red}The Bayesian posterior inference naturally quantifies estimation uncertainty.} Theoretically, we establish approximation error bounds, posterior consistency, and region selection consistency. Empirically, comprehensive simulations and real-world case studies demonstrate the effectiveness and superiority of the proposed method, particularly in settings where traditional functional models are misspecified or lack flexibility.
In modern applications such as ECG monitoring, neuroimaging, wearable sensing, and industrial equipment diagnostics, complex and continuously structured data are ubiquitous, presenting both challenges and opportunities for functional data analysis. However, existing methods face a critical trade-off: conventional functional models are limited by linearity, whereas deep learning approaches lack interpretable region selection for sparse effects. To bridge these gaps, we propose a sparse Bayesian functional deep neural network (sBayFDNN). It learns adaptive functional embeddings through a deep Bayesian architecture to capture complex nonlinear relationships, while a structured prior enables interpretable, region-wise selection of influential domains with quantified uncertainty. Theoretically, we establish rigorous approximation error bounds, posterior consistency, and region selection consistency. These results provide the first theoretical guarantees for a Bayesian deep functional model, ensuring its reliability and statistical rigor. Empirically, comprehensive simulations and real‑world studies confirm the effectiveness and superiority of sBayFDNN. Crucially, sBayFDNN excels in recognizing intricate dependencies for accurate predictions and more precisely identifies functionally meaningful regions, capabilities fundamentally beyond existing approaches.
\end{abstract}

\section{Introduction}

% With rapid technological advances, many applications are generating increasingly complex data, much of which exhibit inherent structure and continuity over time or space. Such data, naturally viewed as realizations of underlying functions, are well suited for analysis within the framework of functional data analysis (FDA), which captures their smoothness and dynamic structure to model complex temporal or spatial patterns. Its effectiveness has been demonstrated across a broad spectrum of real-world applications.
% {\color{blue}For example, ECG recordings are analyzed as functional curves to estimate the underlying morphological structure of heartbeats \cite{matuk2022bayesian}. Similarly, neuroimaging data are modeled as continuous fields to capture spatiotemporal dynamics constrained by the brain's geometry \cite{pang2023geometric}.}
% Similarly, brain signals from neuroimaging are often modeled as functions to investigate patterns of neural connectivity \cite{tsai2024latent}. }
% Similarly, traffic flow data in transportation systems can be treated as functional observations, facilitating more accurate and continuous traffic state prediction \cite{ma2024network}.

With rapid technological advances, diverse fields—from healthcare to neuroscience—are generating increasingly complex data that exhibit inherent structure and continuity over time or space. For instance, electrocardiogram (ECG) recordings are analyzed as functional curves to elucidate cardiac morphology and identify pathological patterns, offering insights into cardiovascular health \cite{pang2023geometric}. Similarly, neuroimaging data are modeled as spatiotemporal fields to capture brain dynamics constrained by anatomical geometry, advancing the study of neural functions and disorders  \cite{tsai2024latent}. Such data are naturally viewed as realizations of underlying functions and are suited for analysis within functional data analysis (FDA), which uses their smoothness and structure to model temporal or spatial patterns. The analysis of these functional observations is important, particularly in human health studies, where interpreting such signals can inform diagnosis, monitoring, and treatment.

In FDA, supervised learning with functional predictors is fundamentally important yet faces two major challenges. First, the relationship between functional predictors and responses is often intrinsically complex and nonlinear. Ignoring such nonlinearity risks significant model misspecification, potentially leading to poor predictive performance in real-world applications. 
% Second, many applications exhibit local sparsity, where the functional coefficient is nonzero only over specific subregions of the domain. 
Second, and more critically, many applications exhibit local sparsity or region-specific effects. This means the predictive relationship is not globally active across the entire function domain; instead, it is concentrated within one or a few specific, contiguous subregions (e.g., time intervals or wavelength bands) where the functional coefficient is nonzero, while being negligible or zero elsewhere. Failure to account for this structure may introduce substantial noise and degrade both interpretability and estimation efficiency. 

These two challenges are commonly intertwined in practice. For example, in bedside monitoring, physiological waveforms such as ECG—driven by intricate and dynamic pathophysiological states—not only manifest highly complex and nonlinear patterns but also frequently exhibit local sparsity \cite{moor2023foundation}.
% , where clinically meaningful information is concentrated within specific subintervals of the waveform while remaining negligible elsewhere \cite{moor2023foundation}. 
Specifically, clinically actionable information is not distributed uniformly but is concentrated within physiologically meaningful regions of the waveform, such as the QRS complex for arrhythmia analysis or the ST segment for ischemia detection. Signals from other intervals are often non-informative. This makes interpretable region selection a core enabling step, as it precisely targets the localized subdomains where complex nonlinear relationships are most active and meaningful, thereby forming the foundation for efficient and robust functional models.

%These dual aspects motivate the urgent need for methods that can simultaneously capture complex nonlinear relationships and perform interpretable region selection.

% A substantial body of work has focused on supervised learning with functional data. Classical functional linear regression (FLR) methods primarily estimate smooth coefficient functions but do not inherently support region selection—a gap later addressed by locally sparse FLR estimators. However, such methods remain constrained by linearity. While nonlinear functional models offer greater flexibility, they often rely on specific, prespecified structures and may struggle with highly complex relationships. Deep neural networks (DNNs) provide a powerful alternative for function approximation and have been adapted to functional data, yet existing DNN-based approaches largely prioritize prediction accuracy and lack mechanisms for interpretable region selection. Although feature selection in DNNs is an active area, most techniques are designed for scalar inputs and do not leverage the continuous, structured domain of functional data. These limitations collectively highlight the need for a DNN-based framework that simultaneously handles nonlinear complexity and performs structured, interpretable selection on the functional domain.

Supervised learning with functional data has been extensively studied. Classical functional linear regression (FLR) estimates smooth coefficient functions but lacks inherent region selection—a limitation partially mitigated by later sparse FLR variants, though linearity constraints remain. Nonlinear functional models offer flexibility yet often depend on prespecified structures, limiting their capacity to capture highly complex relationships. While deep neural networks (DNNs) provide strong approximation power and have been adapted to functional settings, existing DNN-based methods focus predominantly on prediction and do not incorporate interpretable, structured region selection on the functional domain. Moreover, common DNN feature selection techniques are designed for scalar inputs and fail to leverage the continuous nature of functional data. These gaps underscore the need for a DNN framework that jointly handles nonlinear complexity and performs interpretable region selection in functional settings.

\begin{figure*}[t]
  \centering
  \includegraphics[width=\textwidth]{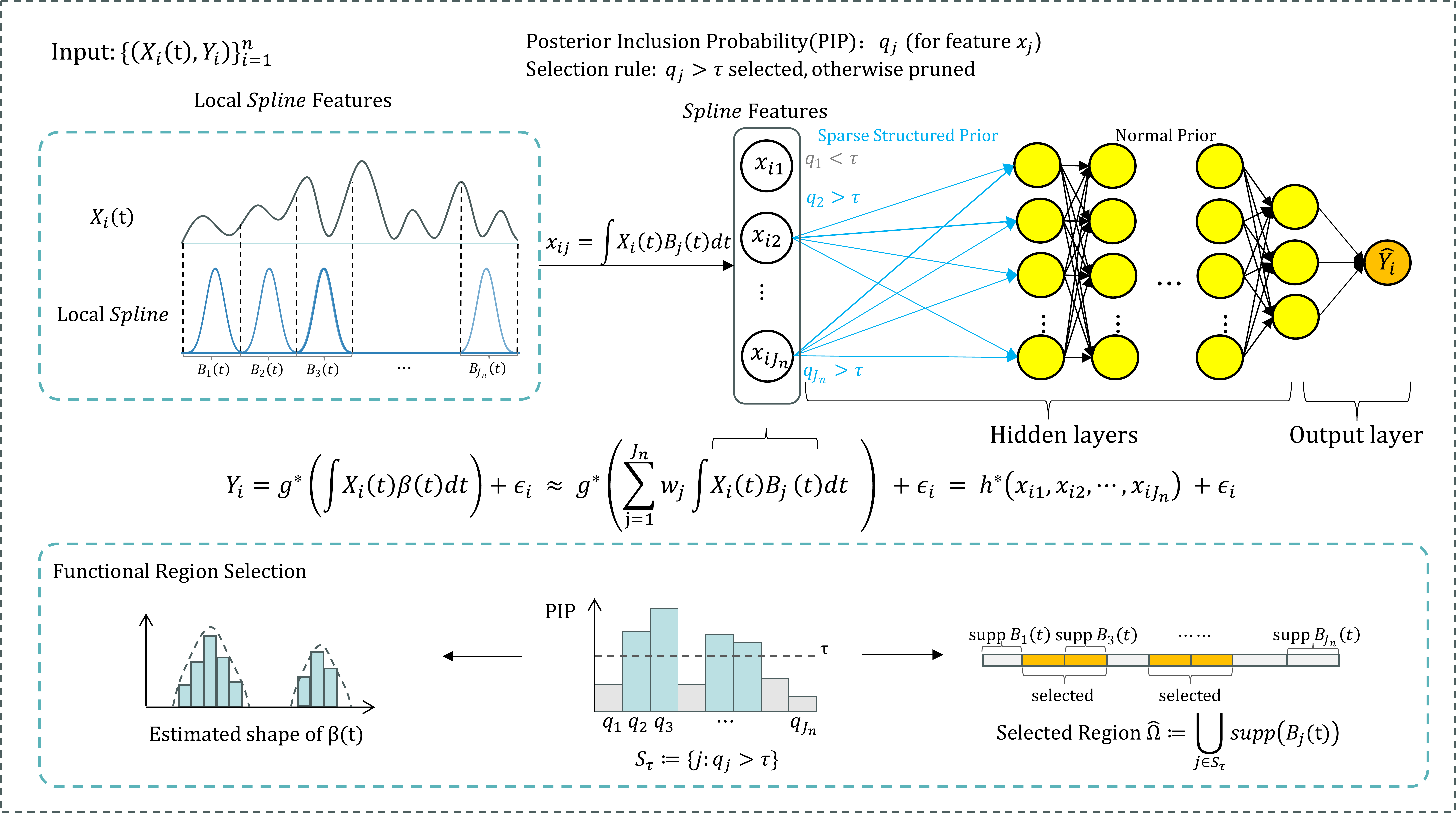}
  \caption{Workflow of sBayFDNN. The functional predictor $X_i(t)$ is projected onto locally supported B-spline bases to form spline features $\boldsymbol{x}_i=(x_{i1},\ldots,x_{i,J_n})^{\top}$. A DNN with a spike-and-slab prior on the first-layer weight columns yields feature-wise posterior inclusion probabilities (PIPs) $q_j$, which are thresholded to select spline features and mapped back to the function domain to produce an estimated active region $\widehat{\Omega}$.}
  \label{workflow}
\end{figure*}

We propose a sparse Bayesian deep neural network (sBayFDNN) that integrates functional embedding learning with a Bayesian DNN architecture to perform interpretable region‑wise selection for functional predictors. The model captures flexible nonlinear function‑to‑scalar relationships while providing posterior inference for uncertainty in region selection. Our main contributions are:
\begin{itemize}
\item A Bayesian Functional DNN Framework with Uncertainty Quantification: We introduce a model that embeds functional predictors through a deep Bayesian network to capture complex nonlinear associations with scalar responses. The framework delivers not only point estimates but also posterior uncertainty for the selected regions, overcoming the rigidity of traditional functional regression. 
\item Interpretable Region Selection via Structured Sparsity: A sparse prior is imposed on the first hidden layer to perform region-wise functional selection. This significantly enhances model interpretability while mitigating the black‑box nature of conventional DNNs, yielding actionable insights into influential domains.
\item Theoretical Guarantees: We establish rigorous theoretical results, including approximation error bounds, posterior consistency, and region selection consistency. These provide the first theoretical foundation for a Bayesian deep functional model, thereby opening a new venue for statistically rigorous and interpretable functional deep learning.
\item Empirical Validation: Comprehensive simulations and diverse real-world studies confirm that sBayFDNN outperforms existing methods in both prediction accuracy and region identification, especially in challenging scenarios where conventional models are misspecified or inflexible, demonstrating its practical superiority.
\end{itemize}

\section{Related Works}

\subsection{Classical Functional Regression Methods}
Classical functional regression methods, represented by the functional linear model and its extensions, establish relationships between functional predictors and scalar responses using techniques such as B-splines \citep{cardot2003spline}, reproducing kernel Hilbert spaces \cite{crambes2009smoothing}, and smoothing regularization \citep{cai2012minimax}. 
To enhance interpretability, sparse penalization methods have been introduced for local region selection, including sparse functional linear regression \cite{lee2012sparse}, sparse functional data embeddings with localized functional principal component analysis \cite{chen2015localized}, as well as smooth locally sparse estimators \citep{lin2017locally,belli2022smoothly} and their Bayesian counterparts \citep{zhu2025bayesian}. To move beyond linearity, nonlinear methods such as functional generalized additive models \cite{mclean2014functional} and functional single‑index models have been proposed \citep{jiang2020functional}. These nonlinear models are further combined with the $L_0$ norm \cite{chamon2020functional} and the functional SCAD penalty \citep{nie2023estimating} to achieve local sparsity.

However, these approaches face key limitations: linear models cannot adequately capture complex nonlinear relationships, while existing nonlinear methods rely on pre-defined kernels or bases, limiting their flexibility and capacity to model intricate dependencies directly from data.

\subsection{Deep Learning for Functional Data}
% The application of deep learning to functional data has developed over recent years, with methods primarily aimed at enhancing representation learning and predictive performance. Early approaches introduced adaptive basis layers within DNNs to learn data-driven functional representations \citep{yao2021deep}. Subsequent work developed DNNs specifically designed for functional inputs \citep{thind2023deep} and classifiers built on functional principal components \citep{wang2023deep}. Autoencoder architectures have also been adapted for functional latent representation learning \citep{wu2024functional}.   

Deep learning for functional data has advanced in recent years, with methods primarily aimed at enhancing representation learning and predictive performance. Early approaches introduced adaptive basis layers to learn data-driven functional representations within DNNs \citep{yao2021deep}, and were followed by DNNs specifically designed for functional inputs \citep{thind2023deep} and classifiers built on functional principal components \citep{wang2023deep}. Autoencoder architectures have also been adapted for functional latent representation learning \citep{wu2024functional}.

Despite their predictive power, these methods often operate as ``black boxes'', lacking explicit mechanisms for region selection on the functional domain. This limits their utility in scientific applications where identifying which specific regions of the data drive predictions is essential.

\subsection{Sparsity and Selection in Neural Networks}
A parallel line of research incorporates sparsity into neural networks for vector-valued inputs.
This includes frameworks performing variable selection via group sparsity penalties on the first hidden layer \citep{dinh2020consistent,chu2023estimating,luo2025sparse},
architectures employing sparsity-inducing mechanisms on linear input layers \citep{chen2021nonlinear,atashgahi2023supervised},
and residual layers used to achieve feature-wise sparsity \citep{lemhadri2021lassonet,fan2025multi}. 
Furthermore, \citet{Sun2022Consistent} proposed a Bayesian framework to learn sparse DNNs through connection pruning, establishing posterior consistency but not traditional input-level variable selection. Beyond variable selection, similar sparse DNN techniques have also been developed for applications in other fields, such as biological or social network reconstruction \cite{fan2026bilevel,yang2026heterogeneous}. 

While these techniques enhance interpretability for vector data, they are not designed to exploit the continuous, structured nature of functional inputs. Consequently, a significant gap remains for a method that integrates nonlinear deep learning with structured, interpretable region selection and uncertainty quantification specifically for functional data.

\section{Methodology}\label{sec:methodology}

Let $X_1(t),\dots,X_n(t)$ be $n$ independent functional covariates defined on the closed interval $\mathcal{T}\subset\mathbb{R}$. For theoretical and computational purposes, $\mathcal{T}$ is standardized to $[0,1]$. Let $Y_i$ denote a scalar response. We consider a functional single-index model (see Figure~\ref{workflow} for a schematic overview):
\begin{equation}
    Y_i = g^*\!\left(\int_{\mathcal{T}} X_i(t)\,\beta(t)\,dt\right) + \varepsilon_i,
    \label{equation functional_model}
\end{equation}
where $\varepsilon_i \sim \mathcal{N}(0,\sigma_\varepsilon^2)$, $\beta(\cdot)$ is an unknown coefficient function, and $g^*(\cdot)$ is an unknown nonlinear function. Without loss of generality, we assume the domain of $g^*(\cdot)$ is $[0,1]$; this can be achieved by suitably rescaling $\int_{\mathcal{T}} X_i(t)\beta(t)dt$ and absorbing the scaling into the definition of $g^*$. For notational simplicity, we retain the original notation in the subsequent exposition.

The function $g^*(\cdot)$ allows for a flexible, potentially nonlinear relationship between the scalar response and the functional covariate through the single-index projection.
Motivated by applications where the association is driven by only a few subregions of $\mathcal T$ (e.g., a small number of time intervals), we assume that $\beta(\cdot)$ is effectively localized and aim to identify the corresponding active regions.

To circumvent the infinite-dimensionality of $\beta(\cdot)$, we approximate it using a truncated B-spline basis expansion:
\begin{equation}
    \beta_{J_n}(t)=\sum_{j=1}^{J_n} w_{J_n,j}\,B_j(t),
    \label{equation beta_expansion}
\end{equation}
which yields the following finite-dimensional representation:
\begin{equation}\label{eq:finite_rep}
\begin{split}
    Y_i
\;\approx\;
g^*\!\bigl(\boldsymbol{\eta}(X_i)^\top \boldsymbol{w}_{J_n}\bigr)
+\varepsilon_i
\;=\;
h^*\!\bigl(\boldsymbol{\eta}(X_i)\bigr)
+\varepsilon_i,
\end{split}
\end{equation}
by replacing $\beta(\cdot)$ with $\beta_{J_n}(\cdot)$, where $\int_{\mathcal T} X_i(t)\beta_{J_n}(t)\,dt = \boldsymbol\eta(X_i)^\top\boldsymbol w_{J_n}$. 
Here, $\{B_j(\cdot)\}_{j=1}^{J_n}$ denotes a collection of $J_n$ B-spline basis functions, and $\boldsymbol w_{J_n}:=(w_{J_n,1},\ldots,w_{J_n,J_n})^\top\in\mathbb R^{J_n}$ is the corresponding vector of coefficients. The spline features $\boldsymbol\eta(X_i)=(x_{i1},\ldots,x_{iJ_n})^\top\in\mathbb R^{J_n}$, where $x_{ij}:=\int_{\mathcal T} X_i(t)\,B_j(t)\,dt$. We further introduce $h^*(\cdot)$, which absorbs the B-spline coefficient vector $\boldsymbol w_{J_n}$, thus reducing the problem to learning a finite-dimensional function. 

Next, we aim to estimate $h^*(\cdot)$ using a DNN $F_{\boldsymbol\theta}$ that takes $\boldsymbol\eta(X_i)\in\mathbb R^{J_n}$ as input.
The network can be written as
\begin{equation}\label{DNN}
F_{\boldsymbol\theta} = A_{H_n}\circ \sigma \circ A_{H_n-1}\circ \cdots \circ \sigma \circ A_1.
\end{equation}
% {\color{red}
% The DNN $F_{\boldsymbol{\theta}}$ directly estimates the composite function $h^*(x_1,\dots,x_{J_n}) = g^*\bigl(\sum_{j=1}^{J_n} b_j x_j\bigr)$, which is the exact transformation of the single‑index model after spline expansion. We use the DNN $F_{\boldsymbol{\theta}}$ as a flexible nonparametric estimator of this projected regression function. This connects the functional single-index model with the fitted DNN: the spline scores provide the finite-dimensional inputs, while the DNN estimates the corresponding conditional mean function.}

In (\ref{DNN}), we consider a feedforward DNN with $H_n-1$ hidden layers and width $L_h$ at layer $h$, where $L_0=J_n$ and $L_{H_n}=1$, with a common activation $\sigma(\cdot)$ in the hidden layers and a linear output layer. We consider the ReLU activation in our study. Let $\boldsymbol\theta=\{(\boldsymbol W_h,\boldsymbol b_h)\}_{h=1}^{H_n}$, where $\boldsymbol W_h\in\mathbb R^{L_h\times L_{h-1}}$ and $\boldsymbol b_h\in\mathbb R^{L_h}$ denote the weights and biases. The affine map is defined as $A_h(\boldsymbol u):=\boldsymbol W_h\boldsymbol u+\boldsymbol b_h$.
%Here, $f_\theta(\eta(X_i))$ is used to approximate $h^*(\eta(X_i))$ and estimate $\theta$ from the data.

Here, the DNN $F_{\boldsymbol{\theta}}$ directly estimates the composite function $h^*(x_{i1},\dots,x_{i,J_n}) = g^*\bigl(\sum_{j=1}^{J_n} w_{J_n,j} x_{i,j}\bigr)$, which is the transformation of the single‑index model after spline expansion. This perspective makes clear that the single‑index structure is a special case of the DNN model class. Importantly, our goal is not to explicitly recover the exact single‑index architecture; rather, we aim to find a DNN $F_{\boldsymbol{\theta}}$ within this class that accurately estimates the conditional expectation $E(Y_i\mid X_i(t))$ under the true single‑index model, which is a standard approach in nonparametric statistical learning.

In (\ref{eq:finite_rep}), since the B-spline basis is locally supported, the localization of $\beta(\cdot)$ over $\mathcal T$ is naturally reflected in the sparsity pattern of $\boldsymbol w_{J_n}$ in \eqref{equation beta_expansion}.
Therefore, we consider the setting where $w_{J_n}$ is sparse, so that 
$
\boldsymbol \eta(X_i)^\top \boldsymbol w_{J_n}=\boldsymbol\eta_{S_{J_n}}(X_i)^\top\boldsymbol  w_{S_{J_n}},
$
where $S_{J_n}\subset\{1,\ldots,J_n\}$ denotes an unknown support set, $\boldsymbol\eta_{S_{J_n}}(X_i)$ and $\boldsymbol w_{S_{J_n}}$ denote the corresponding subvectors. With (\ref{DNN}), since $\boldsymbol w_{J_n}$ is absorbed into $h^*(\cdot)$, the sparsity of $\boldsymbol w_{J_n}$ naturally translates into sparsity of the first‑layer weight matrix $\boldsymbol W_1\in\mathbb R^{L_1\times J_n}$. Specifically, denoting by $\boldsymbol W_{1,*j}\in\mathbb R^{L_1}$ the $j$-th column of $\boldsymbol W_1$, we have that if $||\boldsymbol W_{1,*j}||_2\neq 0$, then $j\in S_{J_n}$.

% In this parameterization, the first-layer weight matrix $W_1$ induces a data-adaptive linear projection of the spline features $\eta(\cdot)$. By imposing structured sparsity on the columns of $W_1$ (e.g., column-wise sparsity), the network effectively identifies a subset of influential spline features, yielding an estimated subset of active spline features.

Based on (\ref{eq:finite_rep}) and (\ref{DNN}), we propose a sparse Bayesian functional DNN (sBayFDNN) framework, employing the following prior distributions for the model parameters: 
% \begin{enumerate}[leftmargin=*,labelsep=0pt,align=left]
% \item
% $\pi(\boldsymbol W_{1,*g} \mid \gamma_g)$\\
% $= \prod_{h=1}^{L_1}
% \Bigl[\phi(W_{1,hg};0,\sigma_{1,n}^2)\Bigr]^{\gamma_g}
% \Bigl[\phi(W_{1,hg};0,\sigma_{0,n}^2)\Bigr]^{1-\gamma_g}$,\\
% $\forall g\in\{1,\dots,J_n\}$.
% \item
% $\pi(\boldsymbol\gamma)
% = \prod_{g=1}^{J_n}\lambda_n^{\gamma_g}(1-\lambda_n)^{1-\gamma_g}$.
% \item 
% $\pi(W_{h,ab})=\phi(W_{h,ab};0,\sigma^2), \qquad \forall\, h\in\{2,\dots,H_n\}$.
% \item
% $\pi(b_{h,a})=\phi(b_{h,a};0,\sigma^2),\qquad
% \forall\, h\in\{1,\dots,H_n\}$.
% \end{enumerate}
\begin{eqnarray}
&&\nonumber\text{for }j \in \{1,\dots,J_n\},\\
&&\boldsymbol W_{1,*j} \mid \gamma_j \sim
\Bigl[\mathcal{N}(0,\sigma_{1,n}^2 \boldsymbol{I})\Bigr]^{\gamma_j}\Bigl[\mathcal{N}(0,\sigma_{0,n}^2\boldsymbol{I})\Bigr]^{1-\gamma_j},\\
&&\pi(\gamma_j)\sim \operatorname{Bern}(\lambda_n),\\
&&W_{h,ab}\sim \mathcal{N}(0,\sigma^2),~\forall\, h\in\{2,\dots,H_n\},\\
&&b_{h,a}\sim \mathcal{N}(0,\sigma^2),~
\forall\, h\in\{1,\dots,H_n\}.
\end{eqnarray}
Here, $\mathcal{N}(c,d)$ denotes the normal distribution with mean $c$ and variance (covariance) $d$, and $\operatorname{Bern}(\lambda_n)$ denotes the Bernoulli distribution with success probability $\lambda_n  \in (0,1)$. The symbol $\boldsymbol{I}$ represents an identity matrix of appropriate dimensions. The entry $W_{h,ab}$ corresponds to the $(a,b)$-th element of the matrix $\boldsymbol{W}_h$, where $a\in{1,\ldots,L_h}$ and $b\in{1,\ldots,L_{h-1}}$. Each $\gamma_j$ (for $j \in {1,\dots,J_n}$) is a binary indicator. We set $\sigma_{1,n}^2 > \sigma_{0,n}^2 > 0$, with $\sigma_{0,n}^2$ taken to be a small positive value, and let $\sigma^2 > 0$.
%$\boldsymbol\gamma=(\gamma_1,\ldots,\gamma_{J_n})^\top\in\{0,1\}^{J_n}$

% \paragraph{Prior specification.}
% To facilitate the description of the prior distributions, we
% first introduce some notation.
% \begin{enumerate}[leftmargin=*,labelsep=0pt,align=left]
% \item $W_{1,*g}\in\mathbb R^{L_1}$ is the $g$th column of the first-layer weight matrix $W_1\in\mathbb R^{L_1\times J_n}$. The entry $W_{h,ab}$ denotes the $(a,b)$th element of $W_h$, where $a\in\{1,\ldots,L_h\}$ and $b\in\{1,\ldots,L_{h-1}\}$.
% \item $\gamma=(\gamma_1,\ldots,\gamma_{J_n})^\top\in\{0,1\}^{J_n}$ denotes a binary indicator vector.
% \end{enumerate}

As introduced above, distinct priors are specified for the first layer and the subsequent layers. To enable functional region selection in the spline‑feature domain, we impose a group‑wise continuous spike‑and‑slab prior on each $\boldsymbol{W}_{1,*j}$ for $j = 1,\dots,J_n$. When $\gamma_j = 0$, the spike variance $\sigma_{0,n}^2$ strongly shrinks all entries of $\boldsymbol{W}_{1,*j}$ toward zero with high probability; when $\gamma_j = 1$, the slab variance $\sigma_{1,n}^2$ permits the entries to take appreciable nonzero values, thereby allowing the corresponding spline‑basis feature to contribute to the functional representation.
% Marginally, integrating out $\gamma_g$ yields the two-component Gaussian mixture prior
% \[
% \boldsymbol W_{1,*g}\sim
% \lambda_n\,\mathcal N\!\left(0,\sigma_{1,n}^2  I_{L_1}\right)
% +(1-\lambda_n)\,\mathcal N\!\left(0,\sigma_{0,n}^2  I_{L_1}\right).
% \]
% Here $I_d$ denotes the $d\times d$ identity matrix.
% For the remaining layers, weights and biases are assigned entrywise independent Gaussian priors $\mathcal N(0,\sigma^2)$.
% This choice regularizes the network while retaining sufficient flexibility in the hidden layers, so that the nonlinear component can effectively learn the unknown structure once the relevant spline features are identified in the input layer.
For the remaining layers, weights and biases are assigned entrywise independent Gaussian priors $\mathcal{N}(0,\sigma^2)$, which preserves the flexibility needed in the hidden layers to learn complex, nonlinear patterns based on the features identified by the first hidden layer.

\section{Optimization-based Bayesian Inference}

% Given the model and prior specification in Section~\ref{sec:methodology}, we describe an optimization-based procedure for parameter estimation and functional region identification.

The proposed Bayesian inference proceeds via three steps: (i) compute a (local) posterior mode (MAP) of the network parameters under the marginal prior; (ii) derive feature-level posterior inclusion probabilities from the first‑layer; and (iii) map the selected spline features to an estimated active region. The detailed procedure is described below, while hyperparameter settings, sensitivity analyses, and runtime results are provided in Appendix A.

\subsection{MAP Fitting}
% Fix a truncation level $J_n$ and construct the spline features $\boldsymbol\eta(X_i)\in\mathbb R^{J_n}$ as in Section~\ref{sec:methodology}.
% Let $\boldsymbol\theta=\{(\boldsymbol W_h,\boldsymbol b_h)\}_{h=1}^{H_n}$ collect all network weights and biases.
For $\boldsymbol\theta$, we optimize the marginal posterior in which the inclusion indicators $\boldsymbol{\gamma}$ are integrated out, and denote the resulting marginal prior by $\pi(\boldsymbol\theta)$. Then, we compute a MAP by minimizing the negative log-posterior objective
\begin{equation}\label{eq:map_obj}
\hat{\boldsymbol\theta} \in \arg\min_{\boldsymbol \theta}\;
\frac{1}{2\sigma_\varepsilon^2}\sum_{i=1}^n\Bigl(Y_i-F_{\boldsymbol\theta}(\boldsymbol\eta(X_i))\Bigr)^2
\;-\;\log \pi(\boldsymbol\theta),
\end{equation}
where $F_{\boldsymbol\theta}$ is the $H_n$-layer feedforward network defined in (\ref{DNN}).
Optimization is carried out using stochastic gradient methods (see Appendix A for the detailed algorithm).

%an iterate sequence $\{\boldsymbol\theta^{(t)}\}_{t=1}^T$ and
% {\color{red}The optimization path $\{\boldsymbol\theta^{(t)}\}$ provides a convenient stability summary for feature relevance.
% In particular, we evaluate closed-form spike-and-slab inclusion scores at $\boldsymbol\theta^{(t)}$ and aggregate them across iterations to obtain frequency-based measures, which can be used for region extraction and optional FDR control.}

\subsection{MAP plug-in Posterior Inclusion Probabilities}
Denote $A_{1,n}=\lambda_n(\sigma_{1,n}^2)^{-L_1/2}$ and
$A_{0,n}=(1-\lambda_n)(\sigma_{0,n}^2)^{-L_1/2}$. Then,
for each feature index $j$, applying Bayes’ rule to (5) and (6) gives the conditional posterior inclusion probability (PIP):
$\Pr(\gamma_j=1\mid \boldsymbol W_{1,*j})=\frac{A_{1,n}\exp\!\left(-\frac{\|\boldsymbol W_{1,*j}\|_2^2}{2\sigma_{1,n}^2}\right)}
{A_{1,n}\exp\!\left(-\frac{\|\boldsymbol W_{1,*j}\|_2^2}{2\sigma_{1,n}^2}\right)
+
A_{0,n}\exp\!\left(-\frac{\|\boldsymbol W_{1,*j}\|_2^2}{2\sigma_{0,n}^2}\right)}.$

We compute plug‑in PIP $\hat{q}_j :=\Pr(\gamma_j=1\mid \boldsymbol{\hat{W}}_{1,*j}) $ based on the MAP estimate $\hat{\boldsymbol\theta}$ and select spline features by thresholding:
$\widehat S_{\tau}:=\{j:\hat{q}_j>\tau\}$ with default $\tau=1/2$. Since $\hat{q}_j$ is monotone in $\|\boldsymbol{\hat{W}}_{1,*j}\|_2^2$, this rule is equivalent to thresholding the first-layer column norms.

\subsection{Mapping Selected Features to an Active Region on $\mathcal T$}\label{subsec:region_map}
Given $\widehat S_{\tau}\subset\{1,\ldots,J_n\}$, we map the selected spline features back to the functional domain using the local support of B-splines.
Let $\operatorname{supp}(B_j(t))\subset\mathcal T$ denote the support of the $j$th basis function.
We define the estimated active region as
$
\widehat\Omega := \bigcup_{j\in\widehat S_{\tau}}\operatorname{supp}(B_j),
$
and represent $\widehat\Omega$ as a union of disjoint subintervals by merging adjacent components.
% Optionally, we replace the fixed thresholding rule by an FDR-controlled selection based on the aggregated inclusion frequencies along the optimization path.

\subsection{Practical Justification of MAP-based PIPs}
While full Bayesian sampling is more common for uncertainty quantification, MAP-based plug-in PIPs are also well established in sparse Bayesian modeling and sparse deep learning \cite{gan2019bayesian,rovckova2018bayesian,rovckova2018spike,yang2021gembag,Sun2022Consistent}, where they offer both theoretical guarantees and computational efficiency. In our setting, we therefore treat the resulting PIPs as approximate posterior selection scores rather than exact posterior marginals. We further examine their empirical stability across random initializations; the corresponding results are reported in Appendix~\ref{app:cross-seed-stability}.

\section{Theoretical Properties}
%We begin by introducing some notation. 
Let $\beta^*$ and $\beta_{J_n}^*$ be the true coefficient function and its truncated counterpart, respectively. Let $\boldsymbol{\omega}_{J_n}^* \in \mathbb{R}^{J_n}$ be the vector of true basis coefficients for $\beta_{J_n}^*$, $S^*_{J_n} = \operatorname{supp}(\boldsymbol{\omega}^*_{J_n})$ its support, and $s_n=|S^*_{J_n}|$ the sparsity level. Define $e_j^*= \mathbb{I}(j \in S^*_{J_n})$ as the true binary selection indicator and $\Omega^*\subset \mathcal{T}$ as the true non-zero region of $\beta^*$. 
Denote $\Omega^*(\kappa)=\{t\in \mathcal{T}:|\beta^*(t)|>\kappa\}$ as the strong-signal region and $\Omega^*(\kappa)^c$ as its complement. Denote \(a_n \lesssim b_n\) if \(a_n \leq C b_n\) for some constant \(C>0\), and \(a_n \asymp b_n\) if \(c \, b_n \leq a_n \leq C b_n\) for constants \(0<c<C\).

Denote $\mu^*(X) = g^*\left(\int_{\mathcal{T}}X(t)\beta^*(t)dt\right)$ as the true mean function and $\mu_{\boldsymbol{\theta}}(X):= F_{\boldsymbol{\theta}}(\boldsymbol{\eta}(X))$ as the network output defined in (\ref{DNN}).
Consider the class of fully-connected, $J_n$-input ReLU networks, denoted by $\mathcal{NN}_{J_n}({H_n},\bar L, E_n)$, with depth $H_n$, constant maximum hidden width $\bar L= \max_{2 \le k \le H_n} L_k$, and parameter bound $\|\boldsymbol{\theta}\|_\infty \le E_n$. We define the column-wise support of the network parameters as $\operatorname{supp}_{\mathrm{col}}(\boldsymbol{\theta})= \{j \in \{1,\cdots,J_n\}: \|\boldsymbol{W}_{1,*j}\|_2 > 0\}$. For any subset $T \subset \{1,\cdots,J_n\}$, the associated column-sparse network class is defined as
  $\mathcal{NN}_{J_n}^{\mathrm{col}}(T; H_n, \bar L, E_n) 
  := \bigl\{ F_{\boldsymbol{\theta}} \in \mathcal{NN}_{J_n}(H_n, \bar L, E_n) : \operatorname{supp}_{\mathrm{col}}(\boldsymbol{\theta}) = T \bigr\}$.

The following assumptions are imposed.

\begin{assumption}
    $\sup\limits_{X_i\in\mathcal X}\int_{\mathcal{T}} |X_i(t)|dt \le C_X$, where
$\mathcal{X}$ denotes the function space to which $X_i(t)$ belongs, and $C_X$ is a positive constant .
    \label{assumption_design_boundedness}
\end{assumption}

\begin{assumption}\label{assumption_smooth_coefficient}
$\beta^*(t)\in\mathcal H^{\alpha_\beta}([0,1])$ for some $\alpha_\beta>0$, where $\mathcal{H}^{\alpha}(I)$ denotes the H\"older space of functions on an interval $I$ with smoothness order $\alpha$.
$\Omega^*$ is a finite union of intervals and
$\bigl|\Omega^*(\kappa_{J_n})^c\bigr|
\rightarrow 0$, as $n\to\infty$,
with $\kappa_{J_n}=c_{\kappa}C_{\beta}J_n^{-{\alpha_\beta}}$, where $c_{\kappa}>4$ and $C_{\beta}$ are constants defined in Lemma \ref{prop:block1-main} (Appendix).
Moreover, the strong-signal region has proportional size in the sense that$\frac{|\Omega^*(\kappa_{J_n})|}{|\Omega^*|}\ge c_\Omega$
for some constant $c_\Omega\in(0,1]$ and all sufficiently large $n$.
\end{assumption}

\begin{assumption}
    $g^*(\cdot)\in\mathcal H^{\alpha_g}([0,1])$ for some constant $\alpha_g>0$.
    \label{assumption_Holder_link}
\end{assumption}
% These are standard regularity conditions in functional data analysis \cite{cai2012minimax,nie2023estimating}. Assumption \ref{assumption_design_boundedness}, on the uniform boundedness of the predictor trajectories, controls the scale of the data and ensures the problem is well-posed. Assumptions \ref{assumption_smooth_coefficient} and \ref{assumption_Holder_link} impose smoothness on the functional coefficient and the nonparametric link, respectively, enabling regularization and governing convergence rates.
These are standard regularity conditions in functional data analysis \cite{cai2012minimax, nie2023estimating}. Assumption \ref{assumption_design_boundedness} ensures the well-posedness of the problem by uniformly bounding the predictor trajectories.
%Assumption \ref{assumption_smooth_coefficient} imposes smoothness on the functional coefficient $\beta^*$, which facilitates regularization and controls the convergence rate; it also introduces a minimum-signal strength condition over the active region of $\beta^*$, ensuring that the nonzero parts of the coefficient are detectable. 
% {Assumption~\ref{assumption_smooth_coefficient} imposes H\"older smoothness on the functional coefficient $\beta^*$, which ensures that $\beta^*$ admits accurate spline approximation and controls the convergence rate.  
% We also assume that the active region $\Omega^*$ is not overly complex, which excludes highly fragmented supports and aligns with typical domain-selection interpretations. 
% Finally, we require that the strong-signal region occupies essentially all of $\Omega^*$ up to a negligible set; this is analogous to a minimum-signal assumption in high-dimensional selection and guarantees that the nonzero part of $\beta^*$ is sufficiently detectable to translate basis-level selection into meaningful support recovery on the original domain. We further focus on the practically common proportional sparsity regime, where the strong-signal region retains a non-vanishing fraction of the active domain. With locally supported spline bases, this implies that the number of active groups scales proportionally with $J_n$, i.e., $s_n \asymp J_n$.
% }
% Assumption \ref{assumption_Holder_link} similarly enforces smoothness on the nonparametric link function, further supporting stable estimation.
Assumption~\ref{assumption_smooth_coefficient} imposes H\"older smoothness on $\beta^*$, guaranteeing its accurate spline approximation. We further assume the active region $\Omega^*$ is not overly complex, which excludes highly fragmented supports and aligns with typical domain-selection interpretations. Crucially, we require that the strong-signal region essentially covers $\Omega^*$ up to a negligible set; this acts as a minimum-signal condition in the functional setting, ensuring that the nonzero part of $\beta^*$ is detectable enough to translate basis selection into faithful support recovery on the domain. This minimum signal strength condition is standard in high dimensional variable and region selection literature. We further focus on the practically common proportional sparsity regime, where the strong-signal region retains a non-vanishing fraction of the active domain, which, under locally supported spline bases, implies that the number of active groups scales as $s_n \asymp J_n$. Assumption \ref{assumption_Holder_link} similarly enforces smoothness on the nonparametric link function, thereby supporting stable estimation.

 % {\color{red}As in the broader high-dimensional variable-selection literature, this type of minimum-signal assumption is mainly a theoretical regularity condition: it is typically stated at the level of an unspecified positive constant and is rarely directly verifiable in real data, but is needed to formalize exact recovery guarantees.}

All theoretical results are established for the DNN‑based model $Y \sim p_{\mu_{\boldsymbol{\theta}}}$ with $\mu_{\boldsymbol{\theta}}(X)=F_{\boldsymbol{\theta}}(\boldsymbol{\eta}(X))$. The true data are generated from the single‑index model $Y \sim p_{\mu^*}$ with $\mu^*(X)=g^*(\int X(t)\beta^*(t)dt)$. Theorems~\ref{theorem_1},~\ref{theorem:Posterior_consistency}, and~\ref{thm:selection_consistency} therefore analyze the approximation, posterior contraction, and selection consistency of the DNN estimator when the truth follows a single‑index structure, demonstrating that the fitted DNN procedure is statistically valid even though the model class is larger.
\begin{theorem}
\label{theorem_1}
  Suppose that Assumptions~\ref{assumption_design_boundedness}--\ref{assumption_Holder_link} hold. Then:
  \begin{enumerate}[label=\textup{(\roman*)},leftmargin=2em]
    \item 
    There exists a network parameter vector $\boldsymbol{\theta}$ such that:
    \[
      \begin{aligned}
         &F_{\boldsymbol{\theta}} = F \in \mathcal{NN}_{{J_n}}^{\mathrm{col}}(S^*_{J_n};{H_n},\bar L, E_n), 
         %\\         &\supp_{\mathrm{col}}(\boldsymbol{\theta}) = S^*_{J_n}.
      \end{aligned}
    \]

    \item 
    With $\alpha_1=\min(\alpha_g,1)$,
    \[
      \sup_{X\in\mathcal X}|\mu_{\boldsymbol{\theta}}(X)-\mu^*(X)|
      \lesssim
      H_n^{-2\alpha_1} + J_n^{-\alpha_\beta\alpha_1}.
    \]
  \end{enumerate}
\end{theorem}
Theorem \ref{theorem_1} establishes that the true sparse model structure is exactly representable within the proposed architecture (i) and provides a non-asymptotic error bound (ii). This bound guarantees that the approximation error decays to zero as the network depth $H_n$ and the number of spline bases $J_n$ grow, with the rates governed by the smoothness of both the functional coefficient $\beta(t)$ and the link function $g^*(\cdot)$. % characterizes the approximation capabilities of the designed sparse neural network (\ref{DNN}). It

% {\color{red}Define the approximation error level as
% \[
%   \tilde\xi_n:= \inf_{\substack{\boldsymbol{\theta}: \operatorname{supp}_{\mathrm{col}}(\boldsymbol{\theta}) = S^*_{J_n} \\ \|\boldsymbol{\theta}\|_\infty \le E_n}} 
%   \sup_{X \in \mathcal{X}} \bigl| \mu_{\boldsymbol{\theta}}(X) - \mu^*(X) \bigr|.
% \]
% }

\begin{assumption} For some constants $\tau'>0$ and ${\alpha_{\sigma}}>0$: 
$\lambda_n \lesssim \frac{1}{J_n\bigl[(n\bar L)^{H_n}(J_n+1)L_1\bigr]^{\tau'}}$,    
$\frac{E_n^2}{H_n(\log n+\log\bar L)} \lesssim \sigma_{1,n}^2 \lesssim n^{{\alpha_{\sigma}}}$, and $\frac{E_n^2}{H_n(\log n+\log\bar L)}\lesssim\sigma^{2}$.
\label{assumption_parameter}
\end{assumption}

\begin{assumption}\label{assumption_DNN}
The DNN $F_{\boldsymbol{\theta}}$ in (\ref{DNN}) satisfies: $\bar L\asymp L_1\asymp 1$, $H_n \;\asymp\; \min\bigl\{\,J_n^{\alpha_\beta/2}, s_n\bigr\}$ with $s_n\asymp J_n$, and $\|\boldsymbol{\theta}\|_\infty \le E_n$, where $E_n=n^{c_1}$ for some positive constant $c_1$.
\end{assumption}

Assumption \ref{assumption_parameter} specifies the rates for key hyperparameters in the Bayesian framework, ensuring the prior is sufficiently diffuse to permit effective posterior contraction toward the true parameter, a standard requirement in Bayesian theory  \citep{ghosal2000convergence}. Assumption \ref{assumption_DNN} links the network depth $H_n$ to the sparsity level $s_n$ so that architectural growth respects the intrinsic sparse structure. Similar conditions are commonly adopted in related theoretical analyses. %restricts the DNN architecture with bounded width and polynomial weights, while 

Let \(\Pi(A \mid D_n)\) denote the posterior probability of an event \(A\) given the observed data $D_n = \{(X_i, Y_i)\}_{i=1}^n$. Let $p_{\mu^*}(\cdot\mid X)$ be the true conditional density of $Y$ given $X$, $p_{\mu_{\boldsymbol{\theta}}}(\cdot\mid X)$ be the approximate density induced by the finite-dimensional representation (\ref{eq:finite_rep}) and the sparse DNN defined in (\ref{DNN}), and $d(\cdot,\cdot)$ denote the Hellinger distance.

\begin{theorem}
  \label{theorem:Posterior_consistency}
   Suppose that Assumptions~\ref{assumption_design_boundedness}--\ref{assumption_DNN} hold and there exists an error sequence $\varepsilon_n^2$ satisfying $ \varepsilon_n^2
  \;\lesssim\;
  \frac{s_n\log(J_n/s_n)}{n}
  +\frac{s_n\bigl(H_n\log n+\log J_n\bigr)}{n}
  +\Bigl(H_n^{-2\alpha_1}+J_n^{-\alpha_\beta\alpha_1}\Bigr)^2$ such that $\sigma_{0,n}^2 \le \tilde{M}_{n,1}(\varepsilon_n)$ and $\max\{\sigma^2,\sigma_{0,n}^2,\sigma_{1,n}^2\}\le \tilde{M}_{n,2}(\varepsilon_n)$, where $\tilde{M}_{n,1}(\varepsilon_n)$ and $\tilde{M}_{n,2}(\varepsilon_n)$ are defined in (\ref{M_1}) and (\ref{M_2}) of Appendix. Then, for some constant $c>0$, we have
  \[
  \begin{aligned}
      P\Bigl[
    &\Pi\bigl\{d(p_{\mu_{\boldsymbol{\theta}}},p_{\mu^*})>4\varepsilon_n \mid D_n\bigr\}
    \ge 2\exp(-c\,n\,\varepsilon_n^2)
    \Bigr]\\
   & \le 2\exp(-c\,n\,\varepsilon_n^2),
  \end{aligned}
  \]
  and
  \[
    \mathbb{E}\Bigl[\Pi\bigl\{d(p_{\mu_{\boldsymbol{\theta}}},p_{\mu^*})>4\varepsilon_n \mid D_n\bigr\}\Bigr]
    \le 4\exp(-2c\,n\,\varepsilon_n^2).
  \]
\end{theorem}
Theorem \ref{theorem:Posterior_consistency} establishes the posterior contraction rate $\varepsilon_n$ of sBayFDNN. This result implies that, with high probability, the posterior distribution concentrates around the true data‑generating process at rate $\varepsilon_n$. Specifically, if $\alpha_\beta\le 2$, then $H_n\asymp J_n^{\alpha_\beta/2}$, setting $J_n \asymp \big(\frac{n}{\log n}\big)^{\frac{1}{1+\alpha_\beta/2+2\alpha_\beta\alpha_1}}$ yields $\varepsilon_n^2 \asymp \big(\frac{\log n}{n}\big)^{\frac{2\alpha_\beta\alpha_1}{1+\alpha_\beta/2+2\alpha_\beta\alpha_1}}$.
If $\alpha_\beta>2$, then $H_n\asymp s_n$, taking $J_n \asymp \big(\frac{n}{\log n}\big)^{\frac{1}{2+4\alpha_1}}$ leads to $\varepsilon_n^2 \asymp \big(\frac{\log n}{n}\big)^{\frac{2\alpha_1}{1+2\alpha_1}}$.

% We define $q_j := \Pi(\gamma_j = 1 \mid D_n) = \mathbb{E}_{\Pi} [\gamma_j \mid D_n]$, and let $e_j^* := \mathbb{I}(j \in S^*_{J_n})$ be the true selection indicator. For any threshold $\widehat q \in (0,1)$, the estimated support is $\widehat S_{\widehat q} := \{j : q_j > \widehat q\}$. 
% Let $\Pi(\cdot \mid D_n)$ denote the joint posterior distribution given the data $D_n = \{(X_i, Y_i)\}_{i=1}^n$.
%For any two sets $A, B$, we denote their symmetric difference by $A \Delta B := (A \setminus B) \cup (B \setminus A)$. 
Further, define the structural difference as
\[
  \rho_n(\varepsilon_n) := \max_{1 \le j \le J_n} \mathbb{E}\Bigl[ |\gamma_j - e_j^*| \cdot \mathbb{I} \{ \boldsymbol{\theta} \notin A_n(\varepsilon_n) \} \Bigm| D_n \Bigr],
\]
where $\mathbb{E}(\cdot|D_n)$ is the conditional expectation, and $A_n(\varepsilon_n)=\Bigl\{\boldsymbol{\theta}:\ d\bigl(p_{\mu_{\boldsymbol{\theta}}},p_{\mu^*}\bigr)\ge \varepsilon_n\Bigr\}$.

\begin{assumption}
$\rho_n(4\varepsilon_n)\to0$ as $n\to \infty$ and $\varepsilon_n\to 0$.
\label{assumption_identifiability}
\end{assumption}

Assumption \ref{assumption_identifiability} serves as an identifiability condition. It implies that as $n\to\infty$ and $\varepsilon_n\to 0$, any candidate model that is close to the true data-generating process in terms of Hellinger distance must asymptotically share the same underlying structure, thereby guaranteeing the consistent selection \cite{Sun2022Consistent}. This assumption should be understood as excluding correlation-equivalent sparse solutions that are nearly indistinguishable in predictive terms. Its role is not to assert that such ambiguity never arises in functional data, but rather to isolate a regime in which exact structural recovery is theoretically meaningful. When neighboring components of the functional predictor are highly correlated, the assumption becomes harder to verify, and the support recovered by the method may be better interpreted as a predictive sparse representative rather than a uniquely identifiable ground-truth structure.

\begin{theorem}
  \label{thm:selection_consistency} 
  Suppose Assumptions~\ref{assumption_design_boundedness}--\ref{assumption_identifiability} hold.
  Then:
  \begin{enumerate}[label=\textup{(\roman*)},leftmargin=2.2em]
    \item $\max_{1\le j\le J_n} \bigl| \hat{q}_j - e_j^* \bigr| \xrightarrow{P} 0$. 
    \item $P \bigl( S^*_{J_n} \subset \widehat S_{\tau} \bigr) \to 1$ for any prespecified $\tau \in (0,1)$.
    \item $P\bigl( \widehat S_{1/2} = S^*_{J_n} \bigr) \to 1$.
    \item $|\widehat{\Omega}\Delta\Omega^*|\xrightarrow{P} 0$.
  \end{enumerate}
\end{theorem}

Theorem \ref{thm:selection_consistency} establishes the asymptotic consistency of spline feature selection based on MAP plug-in posterior probabilities, showing that the estimated selection probabilities converge to the true binary indicators. Furthermore, it guarantees the asymptotically exact recovery of the nonzero region in the continuous function domain from its discrete coefficient support, thereby achieving exact structural selection for the functional effect.

\section{Experiments}

We evaluate sBayFDNN on synthetic and real-world functional data against five competitors: (1) FNN, a spline-feature-based feedforward network using truncated basis expansion \citep{thind2023deep}; (2) AdaFNN, which learns adaptive basis functions via auxiliary networks \citep{yao2021deep}; (3) cFuSIM, a functional single-index method with localized regularization \citep{nie2023estimating}; (4) BFRS, a Bayesian functional region selector with neighborhood structure \citep{zhu2025bayesian}; and (5) SLoS, a functional linear estimator with local sparsity via an fSCAD penalty \citep{lin2017locally}. Predictive RMSE is reported for all methods; region-recovery metrics (Recall, Precision, F1) are provided only for sBayFDNN, cFuSIM, BFRS, and SLoS, as FNN and AdaFNN do not perform region selection.

\subsection{Simulation Studies}

% We design our simulation studies along three key dimensions: (1) the true coefficient function $\beta^*(t)$, which varies from a single interior bump (Simple) to a boundary bump (Medium) and then to a pair of narrow oscillating regions (Complex); (2) the link function $g^*$, which ranges from linear to logistic, sinusoidal, and a composite nonlinear form; and (3) the signal-to-noise ratio (SNR) of the response, set to 5 or 10. Functional covariates are generated from a truncated cosine basis, observed on a discrete grid with measurement noise. Detailed data-generating mechanisms are provided in the Appendix.

Our simulation studies vary three key aspects: (1) the true coefficient function $\beta^*(t)$, which spans a single interior bump (Simple), a boundary bump (Medium), and a pair of narrow oscillating peaks (Complex); (2) the link function $g^*$, taken as linear, logistic, sinusoidal, or a composite nonlinear form; and (3) the response signal‑to‑noise ratio (SNR), set to 5 or 10. Functional covariates are generated from a truncated cosine basis, observed on a discrete grid with added measurement noise (see Appendix for full details).

\begin{figure*}[t]
  \centering
  \includegraphics[width=\textwidth]{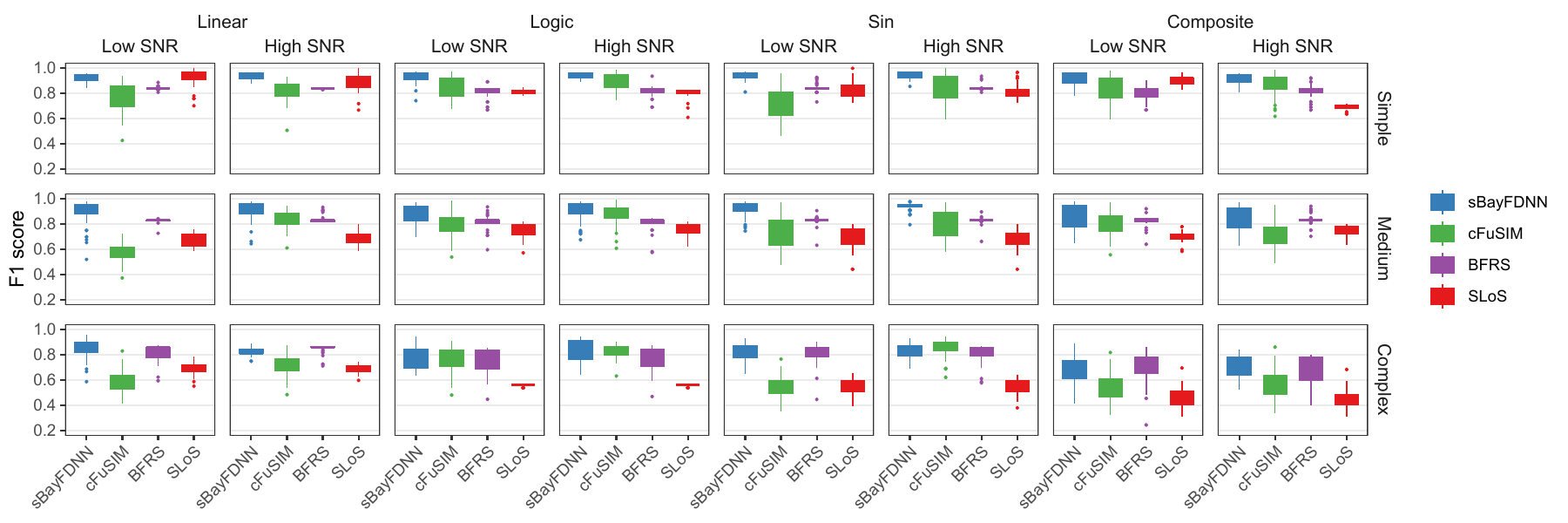}
  \caption{F1 scores across $g$ functions, SNR settings, and $\beta(t)$ scenarios.}
  \label{fig:f1-total}
\end{figure*}

Figure~\ref{fig:f1-total} summarizes region-recovery performance (F1 score) across all simulation scenarios (see Appendix for Recall and Precision). sBayFDNN consistently achieves high mean F1 with tight interquartile ranges, demonstrating robust and stable region identification across both linear and nonlinear regimes. Its advantage is more pronounced in more challenging scenarios with lower signal-to-noise ratios, as well as under harder localization regimes and complex nonlinear link functions $g^*$. By contrast, several baselines exhibit increased dispersion and more frequent low-F1 outcomes as moves away from central support or $g^*$ departs from linearity.

% Although performance modestly declines as the localization task becomes more challenging—moving from a single interior bump to a boundary-adjacent or two separated oscillating bumps, sBayFDNN remains competitive with or superior to other region-selection methods (cFuSIM, BFRS, SLoS) in most settings.

The distributions also suggest different selection behaviors across methods. For example, cFuSIM often returns broader active regions (capturing truly active intervals but at the cost of more false positives), whereas linear region-selection baselines deteriorate under strong nonlinearity due to model misspecification. In contrast, sBayFDNN maintains a better balance between recall and precision. Beyond aggregate metrics, we further evaluate uncertainty quantification. Taking the high-SNR, Medium-$\beta$, logistic-link scenario as an example, Figure~\ref{fig:selection_uncertainty} presents the estimated PIPs from sBayFDNN together with the normalized true coefficient function. The estimated PIPs vary with the signal strength of the true $\beta(t)$: regions with larger signal magnitude are assigned higher inclusion probabilities, while weaker or boundary regions receive more moderate probabilities. This suggests that the PIPs capture not only the location of the active regions but also the relative strength and detectability of the underlying signal. %, yielding robust F1 scores without scenario-specific tuning

\begin{figure}[htbp]
    \centering
     \includegraphics[width=0.8\linewidth]{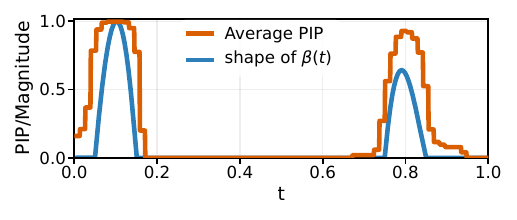} 
    \caption{PIPs from sBayFDNN in the high-SNR, Medium-$\beta$, logistic-link scenario. 
    %The orange curve shows the average plug-in PIP function $p(t)$ across replications, and the blue curve shows the normalized magnitude of the true coefficient function, $|\beta(t)|/\max_{u\in[0,1]}|\beta(u)|$, over the same domain.
    }
    \label{fig:selection_uncertainty}
\end{figure}

% Overall, the results are consistent with the intended design of sBayFDNN: structured sparsity in the first layer supports localized region identification while retaining downstream flexibility for nonlinear response mechanisms.
% In contrast, BayFDNN maintains a better balance between recall and precision across regimes, yielding robust F1 without requiring scenario-specific tuning.

\begin{figure*}[htbp]
  \centering
  \includegraphics[width=\textwidth]{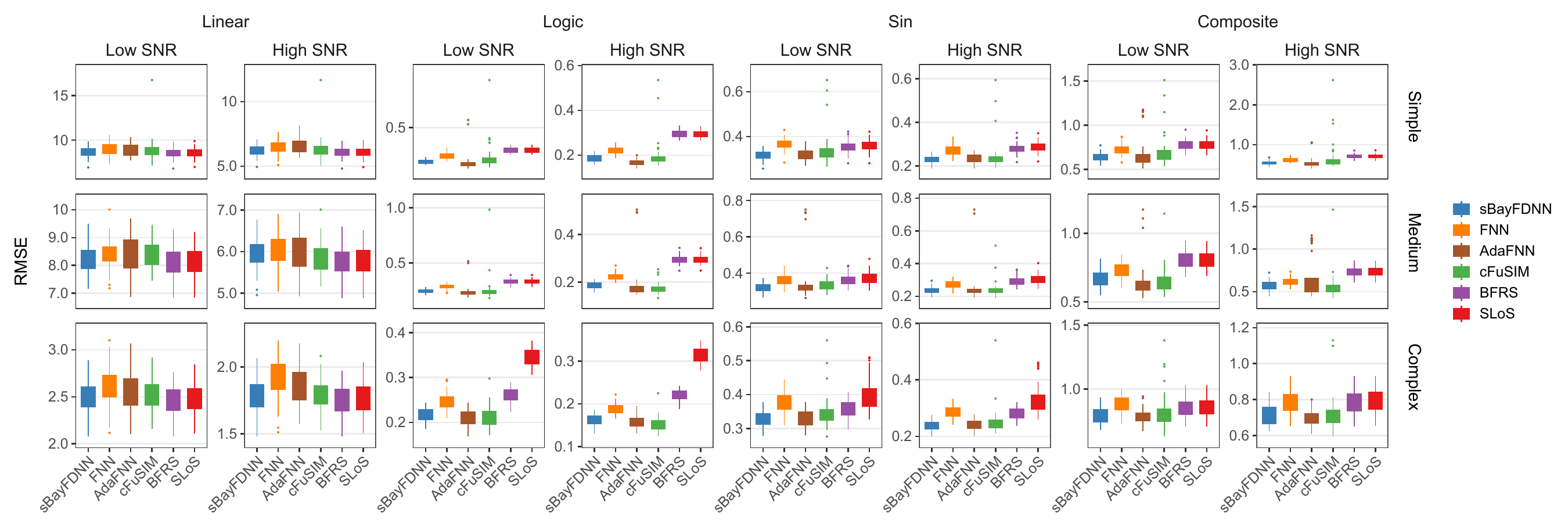}
  \caption{RMSE across $g$ functions, SNR settings, and $\beta(t)$ scenarios.}
  \label{fig:rmse-total}
\end{figure*}

% Figure~\ref{fig:rmse-total} summarizes predictive accuracy. sBayFDNN achieves the lowest or near-lowest mean RMSE in most scenarios, and its RMSE distributions are typically concentrated, suggesting that region interpretability is not obtained at the expense of predictive performance. The advantage is most pronounced under nonlinear links, where accurate approximation of $g^*(\cdot)$ requires flexible function classes. sBayFDNN achieves substantially higher predictive accuracy than the spline-based FNN while matching the performance of the adaptive-basis AdaFNN, yet offers greater stability than the latter. Against the functional single-index method cFuSIM, sBayFDNN attains comparable or better accuracy, particularly under pronounced nonlinearity where its deep architecture proves advantageous. Linear benchmarks (BFRS, SLoS) remain competitive only when the link function is nearly linear, while their performance deteriorates markedly with increasing nonlinearity. % by avoiding the challenging joint optimization of basis functions and regression mapping
Figure~\ref{fig:rmse-total} shows that sBayFDNN achieves the lowest or near-lowest mean RMSE in most scenarios, with concentrated distributions, indicating that region interpretability does not come at the cost of predictive performance. Its advantage is most evident under nonlinear links, where flexible function approximation is crucial. Here, sBayFDNN substantially outperforms the spline-based FNN, matches the accuracy of AdaFNN with greater stability, and surpasses cFuSIM when nonlinearity is pronounced. Linear methods (BFRS, SLoS) remain competitive only under nearly linear settings, with performance degrading sharply as nonlinearity increases.

To further assess robustness beyond the main settings, we additionally examine low-SNR regimes, FGAM mean misspecification, and non-Gaussian heavy-tailed errors; these results are reported in Appendix~\ref{app:low-snr} and Appendix~\ref{app:stress-tests}. Together, these results affirm the sBayFDNN framework: structured sparsity in the first layer supports interpretable region identification with inherent uncertainty quantification, while subsequent deep layers furnish the flexibility needed for accurate nonlinear prediction.

% Compared with neural baselines, BayFDNN matches or improves upon FNN and AdaFNN in both central tendency and stability. FNN provides a strong predictive benchmark by applying a feedforward network to spline features; BayFDNN attains comparable predictive accuracy while additionally producing an interpretable active-region estimate. AdaFNN can perform well when a small learned basis captures the dominant functional variation, but its variability increases in some regimes due to the additional difficulty of jointly learning basis functions and the regression map. Linear baselines (BFRS and SLoS) remain competitive when the response mechanism is close to linear, but are more frequently outperformed under stronger nonlinearities, consistent with their more restrictive modeling assumptions.

% Relative to cFuSIM, BayFDNN often attains comparable or better RMSE, especially in more nonlinear regimes where the deep network approximation is advantageous.
% Linear baselines (BFRS and SLoS) are competitive when the response mechanism is close to linear, but can be outperformed under strongly nonlinear links, consistent with their more restrictive modeling assumptions.
% Taken together, the results support the core motivation of BayFDNN: a structured sparse first layer enables region interpretability, while the remaining network layers preserve sufficient representational capacity for accurate nonlinear prediction.

\subsection{Real Data Analysis}

We evaluate our method on four benchmark datasets: ECG, Tecator, Bike rental, and IHPC, using their official train/validation/test splits throughout. In the ECG task, Lead‑II signals are used as functional inputs to predict the QRS duration—a clinically informative measure of ventricular depolarization relevant for detecting conduction abnormalities \citep{HUMMEL2009553,kashani2005significance}. For the Tecator dataset, near‑infrared absorbance spectra of meat samples serve as functional inputs to predict water content \citep{Tecator}. In the Bike rental forecasting task \citep{bike_sharing_275}, the daily rental‑demand profile (a 24‑point curve) is used to predict total rentals over the following 7 days. For the IHPC dataset \citep{individual_household_electric_power_consumption_235}, daily minute‑averaged active‑power trajectories are taken as functional inputs to predict the next‑day total energy usage. We report standard predictive metrics (RMSE/MAE) for all datasets. For ECG and Tecator, silver‑standard region annotations are available, allowing us to also evaluate region‑identification metrics. Detailed data information are provided in Appendix.

%Together, these tasks demonstrate the applicability of our method across distinct domains: physiological signal analysis and spectral regression. 
%the Individual Household Electric Power Consumption (
% To assess region identification, we define silver-standard intervals on the original domains and then map them to the normalized domain $\mathcal T=[0,1]$ induced by our fixed-window (ECG) and wavelength normalization (Tecator).
% For ECG, we use a $120$\,ms window \citep{yu2003high,HUMMEL2009553} centered at the R peak, i.e., $[-0.06,0.06]$ seconds relative to the R peak, as a silver-standard proxy for the QRS complex extent. 
% For Tecator, we use the water-related absorption band around 970–980 nm\citep{VANKOLLENBURG2021121865} and define the silver-standard wavelength interval as $[965,985]$\,nm. 
% Table~\ref{tab:emp_overview_two_datasets_rmse_f1_recall} summarizes prediction and region identification results.
\begin{table*}[htbp]
\centering
\caption{Performance on ECG and Tecator datasets.
F1, Recall, and Precision are reported only for methods that output an estimated active region; otherwise shown as ``--''. Best results within each dataset/metric are in bold.}
\label{tab:emp_overview_two_datasets_rmse_f1_recall}
\setlength{\tabcolsep}{7pt}
\renewcommand{\arraystretch}{0.8}
\begin{tabular}{lcccccccccc}
\toprule
& \multicolumn{5}{c}{ECG} & \multicolumn{5}{c}{Tecator} \\
\cmidrule(lr){2-6}\cmidrule(lr){7-10}
Method
& RMSE & MAE & F1 & Recall & Precision 
& RMSE & MAE & F1 & Recall & Precision \\
\midrule
sBayFDNN      & \textbf{12.069} & \textbf{8.711} & \textbf{0.634} & \textbf{1.000} & \textbf{0.464} &\textbf{2.138} & \textbf{1.594} & \textbf{0.339}   & \textbf{1.000}  &\textbf{0.204}  \\
FNN       & 12.991 & 9.239 & --    & -- &--   & 2.217  & 1.613 & --    & -- &--   \\
AdaFNN    & 14.083 & 10.198 & --    & --  &--   &3.027  & 2.299  & --    & -- &--    \\
cFuSIM    & 17.677 & 12.861 & 0.501 & 1.000 &0.334 & 3.932 & 3.283 & 0.137 & 0.977 & 0.074  \\
BFRS      & 16.297 & 11.805 & 0.396 & 0.784 &0.265& 2.691 & 2.250 & 0.211 & 0.714 & 0.124\\
SLoS      & 16.258 & 11.782 & 0.412 & 0.815 &0.276 & 2.567 & 2.068 & 0.228 & 0.854 & 0.132\\
\bottomrule
\end{tabular}
\end{table*}

\begin{table}[!htbp]
\centering
\caption{Performance on Bike rental and IHPC datasets.}
\label{tab:emp_overview_bike_ihpc_rmse_mae}
\setlength{\tabcolsep}{7pt}
\renewcommand{\arraystretch}{0.8}
\begin{tabular}{lcccc}
\toprule
& \multicolumn{2}{c}{Bike} & \multicolumn{2}{c}{IHPC} \\
\cmidrule(lr){2-3}\cmidrule(lr){4-5}
Method
& RMSE & MAE
& RMSE & MAE \\
\midrule
sBayFDNN  & \textbf{0.618} & \textbf{0.497} & \textbf{0.536} & \textbf{0.409} \\
FNN       & 0.699 & 0.535 & 0.549 & 0.409 \\
AdaFNN    & 0.720 & 0.577 & 0.552 & 0.420 \\
cFuSIM    & 0.693 & 0.539 & 0.548 & 0.411 \\
BFRS      & 0.749 & 0.550 & 0.552 & 0.416 \\
SLoS      & 0.684 & 0.506 & 0.549 & 0.411 \\
\bottomrule
\end{tabular}
\end{table}
As shown in Table~\ref{tab:emp_overview_two_datasets_rmse_f1_recall} and \ref{tab:emp_overview_bike_ihpc_rmse_mae}, sBayFDNN delivers the strongest overall performance across all datasets. While all methods attain relatively low precision in region selection, sBayFDNN achieves substantially higher recall and F1. Figure~\ref{fig:realdata_selection_uncertainty} plots the estimated PIPs on the original domains with sBayFDNN for ECG and Tecator. For ECG, sBayFDNN assigns higher inclusion strength within and near the clinically motivated QRS interval (shaded), with PIP values approaching 1 close to the boundaries. This is consistent with the fact that QRS duration is determined by onset/offset timing, making the endpoint morphology the most informative. Although the selected region is somewhat wider than the predefined silver interval, we do not interpret this as a localization failure. The QRS silver interval is a simplified fixed proxy, whereas the predictive morphology for QRS duration may extend into adjacent transition regions—particularly around onset/offset boundaries—which also carry clinically relevant information. For Tecator, the PIP increases within the predefined water band (965--985\,nm; shaded), while additional elevated PIP regions appear at earlier wavelengths (including $\sim$930\,nm, often regarded as lipid/fat-associated in short-wave NIR), plausibly reflecting strong collinearity/compositional coupling between water and fat and other broad predictive structure. These results demonstrate that sBayFDNN can recover physically interpretable regions without sacrificing predictive accuracy.

% For ECG, sBayFDNN assigns higher inclusion strength within and near the clinically motivated QRS interval (shaded), with peaks close to the boundaries, consistent with the fact that duration is determined by onset/offset timing and thus endpoint morphology is most informative. For Tecator, the PIP increases within the predefined water band (965--985\,nm; shaded), while additional peaks appear at earlier wavelengths (including $\sim$930\,nm, often regarded as lipid/fat-associated in short-wave NIR), plausibly reflecting strong collinearity/compositional coupling between water and fat and other broad predictive structure. These results demonstrate that sBayFDNN can recover physically interpretable regions without sacrificing predictive accuracy.

% Figure~\ref{fig:realdata_selection_uncertainty} plots the \emph{Average PIP} (average plug-in posterior inclusion probability over five repeats) on the original domains. 
% For ECG, the model assigns higher inclusion strength within and near the clinically motivated QRS interval (shaded), with peaks close to the boundaries, consistent with the fact that duration is determined by onset/offset timing and thus endpoint morphology is most informative. 
% For Tecator, the Average PIP increases within the predefined water band (965--985\,nm; shaded), while additional peaks appear at earlier wavelengths (including $\sim$930\,nm, often regarded as lipid/fat-associated in short-wave NIR), plausibly reflecting strong collinearity/compositional coupling between water and fat and other broad predictive structure.

\section{Conclusion}
We have presented a sparse Bayesian functional DNN framework for nonlinear scalar-on-function regression with automatic region selection. By integrating B-spline expansions with a Bayesian neural network and imposing a structured spike-and-slab prior, the proposed model captures complex nonlinear dependence between functional predictors and scalar responses through data-driven functional representations and a flexible deep architecture. The framework yields interpretable region-wise selection together with uncertainty quantification, supported by theoretical guarantees. Simulations and experiments on multiple real-world datasets confirm its selection accuracy and competitive predictive performance, demonstrating practical utility in identifying region-specific functional effects.
% {\color{red}With fixed-order B-splines, each basis function has local support of order $1/J_n$ and overlaps with its neighbors, so the selected objects remain unions of local intervals rather than degenerating into pointwise variables.} 

\begin{figure}[htbp]
    \centering
     \includegraphics[width=\linewidth]{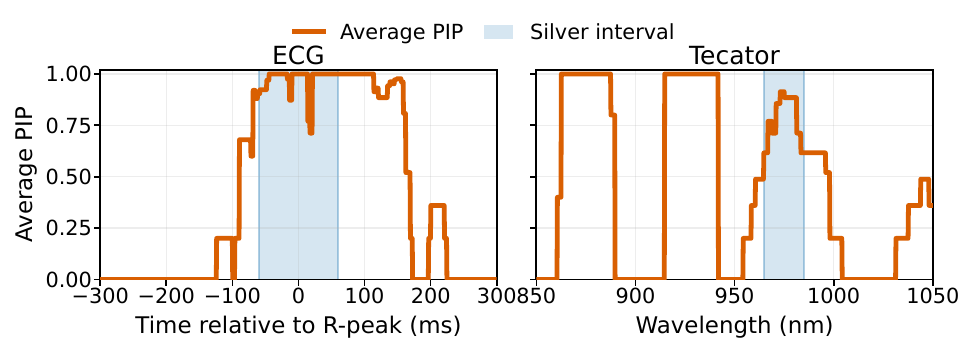} 
    \caption{PIPs from sBayFDNN for ECG and Tecator datasets. 
    }
    \label{fig:realdata_selection_uncertainty}
\end{figure}

At the same time, several limitations should be noted. First, the method depends on the B-spline representation, including the choice of $J_n$, which controls a bias--resolution trade-off; although our additional analysis shows that the method does not collapse to pointwise screening at larger resolutions, fully adaptive basis selection remains an important direction for future work. Second, the current framework focuses on a single functional predictor and a single projection direction within each selected region. Extending it to multiple functional inputs, or to richer within-region embeddings with multiple projection directions, is natural but raises additional challenges in identifiability, since the relevant object may become a region-wise subspace rather than a single direction. Addressing this setting will likely require extra structural constraints such as orthogonality, group sparsity, or subspace-level regularization. Third, region-level interpretation in real applications depends on the covariance structure of the functional predictor; in settings with strong local collinearity, the selected regions are more appropriately viewed as data-supported predictive regions rather than oracle-exact support recovery, and should therefore be interpreted together with domain knowledge.

Several extensions are of practical interest. One is to jointly model functional and discrete covariates rather than pre-adjusting for potential confounders. Another is to replace the current $l_2$ loss with a more robust alternative for noisy or heavy-tailed settings. Finally, embedding the framework in generalized linear or survival-type models would broaden its applicability to binary, count, and time-to-event outcomes, further extending its utility in functional data analysis.

%Future work includes extending the framework to jointly model functional and discrete covariates, rather than pre‑adjusting for confounders. The approach can also be generalized to handle multiple functional predictors with simultaneous variable and region selection. Replacing the $l_2$ loss with a more robust alternative may improve stability in noisy settings, while embedding the model within a generalized linear model would allow applications to binary, count, or survival outcomes, further broadening its utility in functional data analysis.


% Authors are kindly asked to make their submissions as accessible as possible
% for everyone including people with disabilities and sensory or neurological
% differences. Tips of how to achieve this and what to pay attention to will be
% provided on the conference website \url{http://icml.cc/}.

\section*{Software and Data}

The implementation code for sBayFDNN is available at https://github.com/mengyunwu2020/sBayFDNN, while the download link for the accompanying real dataset is provided in Appendix D.

% If a paper is accepted, we strongly encourage the publication of software and
% data with the camera-ready version of the paper whenever appropriate. This can
% be done by including a URL in the camera-ready copy. However, \textbf{do not}
% include URLs that reveal your institution or identity in your submission for
% review. Instead, provide an anonymous URL or upload the material as
% ``Supplementary Material'' into the OpenReview reviewing system. Note that
% reviewers are not required to look at this material when writing their review.

\newpage
\section*{Acknowledgments}
We thank the Area Chair and the anonymous reviewers for their insightful feedback, which was
instrumental in strengthening this paper. This research was supported by the MOE Project of Humanities and Social Sciences (25YJCZH291); Shanghai Science and Technology Development Funds (23JC1402100); Shanghai Research Center for Data Science and Decision Technology; National Institutes of Health (CA204120); and National Science Foundation (2209685).

% If a paper is accepted, the final camera-ready version can (and usually should)
% include acknowledgements.  Such acknowledgements should be placed at the end of
% the section, in an unnumbered section that does not count towards the paper
% page limit. Typically, this will include thanks to reviewers who gave useful
% comments, to colleagues who contributed to the ideas, and to funding agencies
% and corporate sponsors that provided financial support.

\section*{Impact Statement}

This work introduces a Bayesian deep learning framework for interpretable region selection in functional data. It advances nonlinear function‑to‑scalar regression by integrating structured sparsity with principled uncertainty quantification. The primary aim is to enhance machine learning methodology for functional and structured data. The proposed framework benefits domains where interpretability and reliability are essential, such as spectral chemometrics, neural signal analysis, and clinical monitoring. By enabling precise identification of informative functional subdomains, the method can support more accurate material composition analysis, refined physiological signal interpretation, and more targeted diagnostic interventions.

% Authors are \textbf{required} to include a statement of the potential broader
% impact of their work, including its ethical aspects and future societal
% consequences. This statement should be in an unnumbered section at the end of
% the paper (co-located with Acknowledgements -- the two may appear in either
% order, but both must be before References), and does not count toward the paper
% page limit. In many cases, where the ethical impacts and expected societal
% implications are those that are well established when advancing the field of
% Machine Learning, substantial discussion is not required, and a simple
% statement such as the following will suffice:

% ``This paper presents work whose goal is to advance the field of Machine
% Learning. There are many potential societal consequences of our work, none
% which we feel must be specifically highlighted here.''

% The above statement can be used verbatim in such cases, but we encourage
% authors to think about whether there is content which does warrant further
% discussion, as this statement will be apparent if the paper is later flagged
% for ethics review.


\bibliography{example_paper}
\bibliographystyle{icml2026}

%%%%%%%%%%%%%%%%%%%%%%%%%%%%%%%%%%%%%%%%%%%%%%%%%%%%%%%%%%%%%%%%%%%%%%%%%%%%%%%
%%%%%%%%%%%%%%%%%%%%%%%%%%%%%%%%%%%%%%%%%%%%%%%%%%%%%%%%%%%%%%%%%%%%%%%%%%%%%%%
% APPENDIX
%%%%%%%%%%%%%%%%%%%%%%%%%%%%%%%%%%%%%%%%%%%%%%%%%%%%%%%%%%%%%%%%%%%%%%%%%%%%%%%
%%%%%%%%%%%%%%%%%%%%%%%%%%%%%%%%%%%%%%%%%%%%%%%%%%%%%%%%%%%%%%%%%%%%%%%%%%%%%%%
\newpage
\appendix
\onecolumn

% You can have as much text here as you want. The main body must be at most $8$
% pages long. For the final version, one more page can be added. If you want, you
% can use an appendix like this one.

% The $\mathtt{\backslash onecolumn}$ command above can be kept in place if you
% prefer a one-column appendix, or can be removed if you prefer a two-column
% appendix.  Apart from this possible change, the style (font size, spacing,
% margins, page numbering, etc.) should be kept the same as the main body.

\section{Details for the Posterior Inference}

\subsection{Hyperparameter Settings}
\label{app:hyperparams}

All responses $Y$ are standardized using training-set statistics; unless otherwise stated, all computations are carried out on this standardized scale and we set the noise variance to $\sigma_\varepsilon^2=1$.
For non-sparsified network parameters, we use independent Gaussian priors with variance $\sigma^2=1$ (numerically equal to $\sigma_\varepsilon^2$ under standardization, but conceptually distinct).
We use $R=5$ random restarts for each $J_n$.

\smallskip
\noindent\textbf{Simulation defaults.}
\begin{itemize}[leftmargin=*]
  \item Basis: B-splines with degree 4 (unless otherwise stated).
  \item Network: fully-connected ReLU, $64$--$64$--$64$--$1$; first-layer column-wise spike-and-slab selection.
  \item Optimizer: mini-batch SGD, learning rate $10^{-3}$, batch size 64, max 80{,}001 iterations, early stopping patience 3{,}000.
  \item Spike-and-slab: $(\lambda_n,\sigma_{0,n}^2,\sigma_{1,n}^2)=(10^{-5},10^{-5},2\times10^{-3})$.% $\sigma_\varepsilon^2=1$ on the standardized scale.
\end{itemize}

We treat the projection truncation level $J_n$ as a resolution hyperparameter and select it via restart-aggregated evidence in simulations or restart-aggregated validation loss when evidence computation is prohibitive. We use the candidate set $\mathcal J=\{55,60,70,80\}$ for $J_n$ in simulations; for real datasets, $\mathcal J$ is chosen as a small neighborhood around a dataset-specific baseline resolution. This choice is further supported by an additional exploration of $J_n$ sensitivity across replications in Appendix~\ref{app:jn-exploration-case3}.

We train with a small learning rate and a relatively large early-stopping patience to obtain stable solutions across restarts. 
The sparsification prior is intentionally strong in simulations; empirically, the same default sparsification hyperparameters remain effective on ECG ($n\approx 2$--$3\times 10^4$), consistent with increasingly data-dominated inference as $n$ grows.

For competing methods, we aim to ensure fair comparisons by using comparable model capacity whenever applicable (e.g., similar depth/width for neural-network baselines).
Method-specific hyperparameters are taken from authors' recommended defaults when available, and otherwise selected by cross-validation or validation loss on the same training/validation split.

\subsection{Stochastic Gradient Algorithm (SGD) for the MAP Fitting}

We first provide the derivations on the marginal prior on $\boldsymbol{\theta}$. Denote $K_{1_n}=(J_n+1)L_1+\sum_{k=2}^{H_n}(L_{k-1}+1)L_k$ as the dimension of $\boldsymbol{\theta}$ and $\mathcal W\subset\{1,\dots,K_{1_n}\}$ as the set of indices in $\boldsymbol{\theta}$
corresponding to the weights of the first-layer $\{W_{1,hg}\}_{h\le L_1,\,g\le J_n}$,
and let $\mathcal G:=\{1,\dots,K_{1_n}\}\setminus\mathcal W$. In addition, let $\phi(\cdot;0,s^2)$ denote the $\mathcal N(0,s^2)$ density.

Based on the priors introduced in Section \ref{sec:methodology}, we have
\begin{eqnarray*}
&&\pi(\boldsymbol W_{1,*j} \mid \gamma_j)= \prod_{h=1}^{L_1}
\Bigl[\phi(W_{1,hj};0,\sigma_{1,n}^2)\Bigr]^{\gamma_j}
\Bigl[\phi(W_{1,hj};0,\sigma_{0,n}^2)\Bigr]^{1-\gamma_j},
\forall j\in\{1,\dots,J_n\};\\
&&\pi(\gamma_j)
= \lambda_n^{\gamma_j}(1-\lambda_n)^{1-\gamma_j};\\
&&
\pi(W_{h,ab})=\phi(W_{h,ab};0,\sigma^2), \qquad \forall\, h\in\{2,\dots,H_n\};\\
&&\text{and}\\
&&\pi(b_{h,a})=\phi(b_{h,a};0,\sigma^2),\qquad
\forall\, h\in\{1,\dots,H_n\}.
\end{eqnarray*}
Then, the marginal prior on $\boldsymbol{\theta}$ is $\pi(\boldsymbol{\theta})=\sum_{\boldsymbol{\gamma}}\pi(\boldsymbol{\theta},\boldsymbol{\gamma})$, where
\[
  \pi(\boldsymbol{\theta},\boldsymbol{\gamma})=\prod_{j=1}^{J_n}\left[\pi(\gamma_j)\pi(\boldsymbol W_{1,*j} \mid \gamma_j)\right]\prod_{j\in\mathcal G}\phi(\theta_j;0,\sigma^2).
\]

Based on the marginal prior on $\boldsymbol{\theta}$, we develop the following SGD algorithm (Algorithm \ref{alg:evidence_sparse_dnn}) for optimizing (\ref{eq:map_obj}).

\begin{algorithm}[htbp]
\caption{Sparse DNN elicitation with projection-size selection}
\label{alg:evidence_sparse_dnn}
\KwIn{
$\mathcal D_{\mathrm{tr}}=\{(X_i,Y_i)\}_{i=1}^{n_{\mathrm{tr}}}$, $\mathcal D_{\mathrm{va}}=\{(X_i,Y_i)\}_{i=1}^{n_{\mathrm{va}}}$;
candidate projection dimensions $\mathcal J=\{J_1,\ldots,J_M\}$; basis specification for $\{B_j(t)\}_{j\ge1}$;
random restarts $R$; hyperparameters $\sigma_{0,n}^2,\sigma_{1,n}^2,\sigma^2,\lambda_n$; noise variance $\sigma_\varepsilon^2$;
$\mathsf{Crit}\in\{\text{evidence},\text{val}\}$.
}
\KwOut{Selected projection dimension $J^\star$, aggregated predictor $\widehat f(\cdot)$, averaged inclusion scores $\widehat{\mathbf q}$, and selected feature mask $\widehat{\bm\gamma}$.}

\BlankLine
\ForEach{$J\in\mathcal J$}{
\textbf{Step 0 (projection).}
Construct $\{B_1(t),\ldots,B_J(t)\}$ and compute $\bm\eta_J(X)$ for all $(X,Y)\in\mathcal D_{\mathrm{tr}}\cup\mathcal D_{\mathrm{va}}$.\;

\For{$r=1,\ldots,R$}{
\textbf{Step 1 (initialize).} Randomly initialize $\bm\theta$.\;

\textbf{Step 2 (MAP training).} Obtain
\[
\widehat{\bm\theta}_{J,r}\in \arg\min_{\bm\theta}\;
\mathcal L_{n,J}(\bm\theta) =
\frac{1}{2\sigma_\varepsilon^2}\sum_{(X_i,Y_i)\in\mathcal D_{\mathrm{tr}}}\!\Bigl(Y_i-F_{\bm\theta}(\bm\eta_J(X_i))\Bigr)^2
-\log \pi(\bm\theta).
\]

\textbf{Step 3 (validation score).}
Set $v_{J,r}:=\mathrm{MSE}\!\left(\widehat{\bm\theta}_{J,r};\,\mathcal D_{\mathrm{va}}\right)$.\;

\textbf{Step 4 (feature mask and evidence surrogate).}
Let $\widehat{\mathbf W}^{(1)}_{J,r}\in\mathbb R^{w\times J}$ be the first-layer weight matrix in $\widehat{\bm\theta}_{J,r}$.
For $j=1,\ldots,J$, set
\[
\gamma_{J,r,j}
=\mathbbm 1\!\left(\bigl\|\widehat{\mathbf W}^{(1)}_{J,r,:,j}\bigr\|_2^2>\tau_n\right),
\quad
\tau_n
:=\frac{\log\!\bigl((1-\lambda_n)/\lambda_n\bigr)+\frac{w}{2}\log(\sigma_{1,n}^2/\sigma_{0,n}^2)}
{\frac{1}{2\sigma_{0,n}^2}-\frac{1}{2\sigma_{1,n}^2}}.
\]
Denote $\widehat{\bm\gamma}_{J,r}:=(\gamma_{J,r,1},\ldots,\gamma_{J,r,J})^\top$ and construct the evidence surrogate
$\widehat{\bm\theta}^{\,\mathrm{s}}_{J,r}$ by replacing
$\widehat{\mathbf W}^{(1)}_{J,r}$ with
$\widehat{\mathbf W}^{(1)}_{J,r}\mathrm{diag}(\widehat{\bm\gamma}_{J,r})$
(used only for evidence computation).\;

\textbf{Step 5 (evidence score; post-sparsification).}
Compute $\ell_{J,r}:=\log \mathrm{Ev}\!\left(\widehat{\bm\theta}^{\,\mathrm{s}}_{J,r};\,\mathcal D_{\mathrm{tr}}\right)$.\;
}
\BlankLine
\textbf{Step 6 (aggregate over restarts).}
Set $\bar\ell_J:=R^{-1}\sum_{r=1}^R \ell_{J,r}$ and $\bar v_J:=R^{-1}\sum_{r=1}^R v_{J,r}$.\;
}

\BlankLine
\textbf{Step 7 (select $J$).}
\eIf{$\mathsf{Crit}=\text{evidence}$}{
$J^\star \in \arg\max_{J\in\mathcal J}\ \bar\ell_J$.\;
}{
$J^\star \in \arg\min_{J\in\mathcal J}\ \bar v_J$.\;
}

\BlankLine
\textbf{Step 8 (aggregate predictor and PIPs at $J^\star$).}
Define the aggregated predictor
$\widehat f(x):=\frac{1}{R}\sum_{r=1}^R F_{\widehat{\bm\theta}_{J^\star,r}}(\bm\eta_{J^\star}(x)).$

For each $r=1,\ldots,R$ and $j=1,\ldots,J^\star$, compute the restart-specific plug-in inclusion score
$q_{J^\star,r,j}$ from $\widehat{\mathbf W}^{(1)}_{J^\star,r,:,j}$.

Compute the restart-specific plug-in inclusion score $q_{J^\star,r,j}$ then set
$\widehat q_j:=\frac{1}{R}\sum_{r=1}^R q_{J^\star,r,j},
   j=1,\ldots,J^\star.$

Let $\widehat{\mathbf q}:=(\widehat q_1,\ldots,\widehat q_{J^\star})^\top$ and define the final feature mask
$\widehat\gamma_j=\mathbbm 1(\widehat q_j>\tau), j=1,\ldots,J^\star,$ for a prespecified thresholding rule.\;

\Return $(J^\star,\widehat f,\widehat{\mathbf q},\widehat{\bm\gamma})$.
\end{algorithm}

\clearpage
After selecting $J^\star$, we form the final predictor by averaging the outputs of the $R$ independently trained models at $J^\star$, i.e.,
$\hat y(x)=\frac{1}{R}\sum_{r=1}^R F_{\widehat{\bm\theta}_{J^\star,r}}(\bm\eta_{J^\star}(x)).$
For interpretability summaries, we first compute restart-specific plug-in PIPs at the selected resolution $J^\star$, average them across the same $R$ fitted models, and then map the averaged PIPs to the functional domain for region-level interpretation.

In Step~5 of Algorithm~\ref{alg:evidence_sparse_dnn}, we score each run using a 
Laplace-type evidence surrogate \cite{mackay1992evidence,liang2013bayesian}. 
This score is used only as a practical model-comparison criterion across candidate 
projection dimensions and random restarts, rather than as an exact marginal likelihood 
for the full nonconvex posterior.

Let $\mathcal L_{n,J}(\bm\theta)$ denote the negative log-posterior objective used in MAP 
training under projection dimension $J$, and define
$h_{n,J}(\bm\theta):=-\mathcal L_{n,J}(\bm\theta)/n_{\mathrm{tr}}$.
Let $\bm H_{n,J}(\bm\theta):=\nabla^2_{\bm\theta} h_{n,J}(\bm\theta)$.
For this score, we evaluate $h_{n,J}$ and $\bm H_{n,J}$ at the sparsified surrogate 
parameter $\widehat{\bm\theta}^{\,\mathrm{s}}_{J,r}$ constructed in Step~4, while 
predictive evaluations are based on the original MAP estimate 
$\widehat{\bm\theta}_{J,r}$.

Let $\mathcal I_{J,r}$ index the parameters retained in the sparsified surrogate, and let
$d_{J,r}:=|\mathcal I_{J,r}|$. Denote by 
$\bm H_{n,J}(\bm\theta)_{\mathcal I_{J,r},\mathcal I_{J,r}}$ the corresponding principal 
submatrix. The Laplace-type log-score is defined as
\begin{equation}
\label{eq:laplace_logevd_appendix}
\ell_{J,r}
:= n_{\mathrm{tr}}\, h_{n,J}\!\left(\widehat{\bm\theta}^{\,\mathrm{s}}_{J,r}\right)
+\frac{d_{J,r}}{2}\log(2\pi)
-\frac{d_{J,r}}{2}\log(n_{\mathrm{tr}})
-\frac{1}{2}\log \det\!\left(
-\bm H_{n,J}\!\left(\widehat{\bm\theta}^{\,\mathrm{s}}_{J,r}\right)_{\mathcal I_{J,r},\mathcal I_{J,r}}
\right).
\end{equation}
Here the Hessian is restricted to the retained parameters of the sparsified surrogate. 
Because the underlying DNN posterior is nonconvex and potentially multimodal, 
$\ell_{J,r}$ should be interpreted as a computationally tractable Laplace-type 
selection score, not as exact Bayesian evidence.

We compute the log-determinant term in~\eqref{eq:laplace_logevd_appendix} using the 
eigenvalues of the restricted negative Hessian for numerical stability, and apply 
standard stabilization when needed, such as adding a small diagonal jitter.

\subsection{Examination of Evidence-Based $J_n$ Selection}
\label{app:evidence-jn}

To assess the practical effectiveness of the evidence-based criterion for selecting $J_n$, we examine how highly the selected candidate ranks relative to the oracle choice defined by F1 over the candidate set.
Based on 50 replicates, Table~\ref{tab:evidence-jn-seedlevel-topk} summarizes the resulting seed-level rankings across three representative scenarios with composite $g$ and SNR$=10$ (see Sections 6.1 and C.2 for details). Since the candidate set contains eight possible $J_n$ values in each run, this analysis is intended to evaluate whether the evidence-based criterion tends to identify competitive choices, rather than to claim exact recovery of the oracle $J_n$ in every case.

\begin{table*}[htbp]
\centering
\small
\setlength{\tabcolsep}{4pt}
\caption{Top-$k$ effectiveness of evidence-based $J_n$ selection (seed-level; F1 criterion). Candidate set size per run is 8.}
\label{tab:evidence-jn-seedlevel-topk}
\begin{tabular}{ccccccc}
\toprule
Scenario & Mean Rank & Top-1 & Top-2 & Top-3 & Within 15\% of Oracle F1 & Within 30\% of Oracle F1 \\
\midrule
Complex $\beta(t)$, Composite $g$ & 3.85 & 11.0\% & 21.4\% & 28.6\% & 50.0\% & 92.9\% \\
Simple $\beta(t)$, Composite $g$ & 2.5 & 28.6\% & 42.9\% & 64.3\% & 100.0\% & 100.0\% \\
Moderate $\beta(t)$, Composite $g$ & 3.33 & 33.3\% & 33.3\% & 33.3\% & 66.7\% & 100.0\% \\
Overall & 3.17 & 17.6\% & 23.5\% & 44.1\% & 73.5\% & 97.1\% \\
\bottomrule
\end{tabular}
\end{table*}

The results suggest that the evidence-based criterion is reasonably effective in practice.
Although it does not always identify the oracle-ranked $J_n$, the selected candidate is often competitive: overall, 44.1\% of the selections fall within the top three candidates, 73.5\% are within 15\% of the oracle F1, and 97.1\% are within 30\% of the oracle F1.
This indicates that the criterion frequently yields near-oracle choices, especially in the simpler scenarios, while leaving room for future exploration of potentially stronger model-selection rules.

\subsection{Sensitivity Analysis of $J_n$}
\label{app:jn-exploration-case3}

To examine how the number of spline bases $J_n$ affects support recovery and prediction performance, we perform a sensitivity analysis under two simulation scenarios: simple $\beta(t)$ with composite $g$ at SNR $=10$ and SNR $=5$.
For each scenario, we generated 50 independent data replications and evaluated
\(
J_n \in \{20,40,\ldots,240\}
\).
For each replication and each $J_n$, we fit the model once and computed region-selection and prediction metrics.
We then summarized results as mean(sd) over the 50 replications.

In addition to the standard recovery metrics, we also consider two measures to characterize the continuity and granularity of the resulting importance profile and selected regions.
Specifically, \textit{Curve Roughness} quantifies the roughness of the importance curve $p(t)$ through
$\frac{1}{T-1}\sum_{k=1}^{T-1}\big|p(t_{k+1})-p(t_k)\big|,$
where larger values indicate a less smooth and more fragmented importance curve.
We also report \textit{MinLen/Mesh}, defined as the minimum selected interval length divided by the mesh size.
This measures how many grid cells are covered by the shortest selected interval, with smaller values indicating more fragmented selections. Together, these two metrics help assess whether increasing $J_n$ leads to more irregular importance profiles or excessively short selected intervals. A third metric, \textit{MeanLen} (the mean interval length of the identified region), is also reported.
\begin{table*}[htbp]
\centering
\small
\setlength{\tabcolsep}{2.5pt}
\caption{Simulation results under different $J_n$ settings for the scenario with simple $\beta(t)$, composite link function $g$, and two SNR levels (reported as mean(sd) over 50 reps).}
\label{tab:jn-case3-ghard-snr10-rep50}
\begin{tabular}{ccccccccc}
\toprule
SNR &$J_n$ & MeanLen & Curve Roughness & RMSE & F1 & Recall & Precision & MinLen/Mesh \\
\midrule
\multirow{12}{*}{$10$} &20 & 0.5500(0.1843) & 0.0056(0.0014) & 0.5587(0.0640) & 0.5573(0.1017) & 1.0000(0.0000) & 0.3926(0.0924) & 11.0000(3.6857) \\
&40 & 0.2800(0.0307) & 0.0077(0.0005) & 0.5515(0.0518) & 0.8366(0.0523) & 1.0000(0.0000) & 0.7225(0.0780) & 11.2000(1.2262) \\
&60 & 0.2286(0.0300) & 0.0077(0.0005) & 0.5522(0.0493) & 0.9175(0.0437) & 0.9807(0.0456) & 0.8684(0.0823) & 13.7143(1.7981) \\
&80 & 0.2087(0.0337) & 0.0077(0.0012) & 0.5677(0.0496) & 0.9323(0.0480) & 0.9500(0.0568) & 0.9245(0.0921) & 16.6947(2.6985) \\
&100 & 0.1782(0.0352) & 0.0084(0.0015) & 0.5708(0.0513) & 0.9095(0.0477) & 0.8867(0.0875) & 0.9462(0.0783) & 17.1875(4.9006) \\
&120 & 0.1489(0.0402) & 0.0083(0.0014) & 0.5698(0.0535) & 0.8803(0.0576) & 0.8229(0.0997) & 0.9601(0.0652) & 16.4276(6.9745) \\
&140 & 0.1360(0.0366) & 0.0097(0.0025) & 0.5647(0.0486) & 0.8720(0.0611) & 0.7928(0.0965) & 0.9797(0.0459) & 17.0265(7.9199) \\
&160 & 0.1084(0.0405) & 0.0102(0.0026) & 0.5733(0.0540) & 0.8477(0.0719) & 0.7505(0.1032) & 0.9854(0.0384) & 14.4821(8.9812) \\
&180 & 0.0921(0.0399) & 0.0106(0.0026) & 0.5764(0.0552) & 0.8180(0.0685) & 0.7014(0.0963) & 0.9927(0.0318) & 13.4591(9.4446) \\
&200 & 0.0707(0.0298) & 0.0112(0.0025) & 0.5751(0.0558) & 0.8021(0.0714) & 0.6783(0.0992) & 0.9946(0.0227) & 9.7347(7.5132) \\
&220 & 0.0605(0.0278) & 0.0114(0.0021) & 0.5689(0.0532) & 0.7702(0.0618) & 0.6325(0.0806) & 0.9949(0.0192) & 9.2074(7.3905) \\
&240 & 0.0569(0.0289) & 0.0127(0.0029) & 0.5724(0.0533) & 0.7583(0.0687) & 0.6210(0.0889) & 0.9862(0.0373) & 9.2339(7.9701) \\
\midrule
\multirow{12}{*}{$5$} &20 & 0.5475(0.1659) & 0.0058(0.0014) & 0.6542(0.0661) & 0.5549(0.0928) & 1.0000(0.0000) & 0.3892(0.0841) & 10.9500(3.3174) \\
&40 & 0.2878(0.0336) & 0.0077(0.0004) & 0.6474(0.0515) & 0.8237(0.0543) & 1.0000(0.0000) & 0.7037(0.0777) & 11.5111(1.3424) \\
&60 & 0.2354(0.0376) & 0.0075(0.0006) & 0.6457(0.0491) & 0.9090(0.0601) & 0.9843(0.0343) & 0.8528(0.1059) & 14.1214(2.2567) \\
&80 & 0.2129(0.0364) & 0.0077(0.0008) & 0.6587(0.0510) & 0.9240(0.0495) & 0.9571(0.0591) & 0.9032(0.0976) & 16.9263(3.2311) \\
&100 & 0.1732(0.0351) & 0.0084(0.0016) & 0.6611(0.0513) & 0.9037(0.0551) & 0.8802(0.0927) & 0.9414(0.0819) & 16.5417(5.0492) \\
&120 & 0.1468(0.0437) & 0.0088(0.0020) & 0.6646(0.0542) & 0.8808(0.0546) & 0.8262(0.0962) & 0.9549(0.0621) & 16.0138(7.4076) \\
&140 & 0.1299(0.0415) & 0.0098(0.0025) & 0.6602(0.0537) & 0.8693(0.0601) & 0.7904(0.0919) & 0.9750(0.0473) & 16.1000(8.2870) \\
&160 & 0.0983(0.0381) & 0.0100(0.0026) & 0.6613(0.0500) & 0.8355(0.0667) & 0.7374(0.1006) & 0.9767(0.0522) & 12.7179(8.2988) \\
&180 & 0.0889(0.0396) & 0.0105(0.0024) & 0.6663(0.0552) & 0.8169(0.0748) & 0.7009(0.1046) & 0.9916(0.0346) & 12.7432(9.2985) \\
&200 & 0.0721(0.0350) & 0.0115(0.0021) & 0.6685(0.0554) & 0.7980(0.0755) & 0.6756(0.1006) & 0.9877(0.0362) & 10.4286(8.6346) \\
&220 & 0.0663(0.0325) & 0.0115(0.0026) & 0.6614(0.0512) & 0.7699(0.0685) & 0.6356(0.0920) & 0.9901(0.0311) & 10.5926(8.8314) \\
&240 & 0.0524(0.0280) & 0.0128(0.0028) & 0.6633(0.0550) & 0.7497(0.0677) & 0.6095(0.0863) & 0.9870(0.0358) & 9.0305(7.6008) \\
\bottomrule
\end{tabular}
\end{table*}

Table \ref{tab:jn-case3-ghard-snr10-rep50} shows a consistent bias--resolution trade-off.
When $J_n$ is small, selected regions are relatively broad (larger mean interval length), which tends to increase recall but reduce precision.
As $J_n$ increases to a moderate range, localization improves and F1 reaches its best levels.
For very large $J_n$, selected regions become increasingly fragmented (smaller interval length and lower MinLen/Mesh), and recall gradually decreases, leading to lower F1. Importantly, even at larger $J_n$, selected regions are still induced by overlapping spline supports in the function domain, rather than pointwise discrete screening. Overall, these results support the practical criterion used in the main text: moderate $J_n$ values provide a better balance between interpretability (region continuity) and selection accuracy.

\subsection{Sensitivity Analysis of Hyperparameters}
We conduct one-at-a-time (OAT) sensitivity analyses by varying one hyperparameter at a time while fixing the others at their default values. Specifically, simulated data under various settings (see Sections~6.1 and~C.2 for details) are examined with 50 replications per scenario. Summarized results for the spike-and-slab hyperparameters ($\lambda_n$, $\sigma_0^2$, and $\sigma_1^2$) and the PIP threshold ($\tau$) are provided in Tables~\ref{tab:sensitivity-hyperparam} and~\ref{tab:sensitivity-tau}, respectively, and the results in both tables are macro-averaged over all simulation scenarios.

For the spike-and-slab hyperparameters, the prediction and recovery results remain reasonably stable across the examined ranges, suggesting that the proposed method is not overly sensitive to moderate changes in these prior hyperparameters. For $\lambda_n$, no clear monotone trend is observed, and the macro-averaged RMSE and recovery metrics remain broadly comparable across values, indicating relative robustness to the prior inclusion probability. In contrast, $\sigma_0^2$ exhibits a clearer effect on the recall--precision trade-off: an extremely small spike variance leads to overly inclusive selection, with near-perfect recall but much lower precision, whereas larger values make the selection more conservative. For $\sigma_1^2$, the effect is also non-monotone, but different values induce different recall--precision trade-offs; the default choice $\sigma_1^2=2\times 10^{-3}$ provides the most balanced overall performance in terms of macro-averaged F1.

As expected, $\tau$ controls the trade-off between recall and precision (Table \ref{tab:sensitivity-tau}). When $\tau$ is small (e.g., $\tau=0.1$), the procedure is more inclusive, leading to very high recall but lower precision and a larger number of selected components. As $\tau$ increases to a moderate range ($\tau=0.2$--$0.5$), the method achieves a more balanced recovery performance, with the best F1 scores attained at $\tau=0.2$ and $\tau=0.3$. When $\tau$ becomes too large ($\tau \ge 0.6$), the selection becomes increasingly conservative, resulting in substantially reduced recall and fewer selected components, which in turn leads to a marked drop in F1. Overall, these results suggest that the method is reasonably stable over a moderate range of threshold values, while overly small or overly large thresholds may induce over-selection or under-selection, respectively.

% Overall, the region recovery results remain stable over a reasonable range of hyperparameter values, indicating that the proposed method is not overly sensitive to moderate perturbations of the tuning parameters. In particular, the main recovery metrics exhibit only limited variation across the OAT settings, suggesting that the default configuration provides a robust practical choice.

\begin{table}[h]
\centering
\small
\setlength{\tabcolsep}{4pt}
\caption{Sensitivity analysis of the spike-and-slab hyperparameters (macro-averaged over scenarios). Results are reported as mean(sd).}
\label{tab:sensitivity-hyperparam}
\begin{tabular}{cccccc}
\toprule
Hyperparameter & Value & RMSE & F1 & Recall & Precision \\
\midrule
\multirow{5}{*}{$\lambda_n$}
& $1\times 10^{-6}$ & 0.6747(0.1180) & 0.7868(0.1472) & 0.7574(0.1582) & 0.8349(0.1261) \\
& $3\times 10^{-6}$ & 0.6813(0.1273) & 0.8182(0.2748) & 0.7859(0.2667) & 0.8532(0.1764) \\
& $1\times 10^{-5}$ & 0.6693(0.1201) & 0.7982(0.1071) & 0.8884(0.1045) & 0.7346(0.1007) \\
& $3\times 10^{-5}$ & 0.6931(0.1237) & 0.7165(0.1995) & 0.6845(0.2687) & 0.8470(0.0874) \\
& $1\times 10^{-4}$ & 0.6743(0.1183) & 0.7947(0.1382) & 0.7746(0.1438) & 0.8315(0.1228) \\
\midrule
\multirow{5}{*}{$\sigma_0^2$}
& $1\times 10^{-6}$  & 0.6542(0.1271) & 0.5463(0.0754) & 0.9990(0.0010) & 0.3826(0.0732) \\
& $3\times 10^{-6}$ & 0.6889(0.1341) & 0.7142(0.2248) & 0.6796(0.3045) & 0.8437(0.0551) \\
& $1\times 10^{-5}$ & 0.6693(0.1201) & 0.7982(0.1071) & 0.8884(0.1045) & 0.7346(0.1007) \\
& $3\times 10^{-5}$ & 0.6937(0.1229) & 0.6684(0.2178) & 0.5682(0.2576) & 0.8952(0.0752) \\
& $1\times 10^{-4}$ & 0.6873(0.1232) & 0.6991(0.1973) & 0.6642(0.2592) & 0.7380(0.0322) \\
\midrule
\multirow{5}{*}{$\sigma_1^2$}
& $1\times 10^{-3}$ & 0.6920(0.1195) & 0.7035(0.1536) & 0.6346(0.2215) & 0.7892(0.0238) \\
& $1.5\times 10^{-3}$ & 0.6832(0.1256) & 0.7150(0.2056) & 0.6785(0.3121) & 0.8606(0.0016) \\
& $2\times 10^{-3}$ & 0.6693(0.1201) & 0.7982(0.1071) & 0.8884(0.1045) & 0.7346(0.1007) \\
& $3\times 10^{-3}$ & 0.6928(0.1303) & 0.7095(0.2294) & 0.6742(0.3097) & 0.8470(0.0529) \\
& $5\times 10^{-3}$ & 0.6534(0.1278) & 0.7759(0.1444) & 0.9854(0.0028) & 0.6696(0.2009) \\
\bottomrule
\end{tabular}
\end{table}

\begin{table}[h]
\centering
\small
\setlength{\tabcolsep}{5pt}
\caption{Sensitivity analysis of the PIP threshold $\tau$ (macro-averaged over scenarios). Results are reported as mean(sd).}
\label{tab:sensitivity-tau}
\begin{tabular}{ccccc}
\toprule
$\tau$ & F1 & Recall & Precision & $n_{\mathrm{selected}}$ \\
\midrule
0.1 & 0.8256(0.0217) & 0.9930(0.0031) & 0.7251(0.0316) & 11.51(0.48) \\
0.2 & 0.8969(0.0137) & 0.9573(0.0083) & 0.8579(0.0244) & 8.30(0.21) \\
0.3 & 0.8969(0.0137) & 0.9573(0.0083) & 0.8579(0.0244) & 8.30(0.21) \\
0.4 & 0.8781(0.0126) & 0.8610(0.0164) & 0.9146(0.0201) & 5.83(0.19) \\
0.5 & 0.8781(0.0126) & 0.8610(0.0164) & 0.9146(0.0201) & 5.83(0.19) \\
0.6 & 0.7428(0.0281) & 0.6498(0.0340) & 0.9201(0.0159) & 3.32(0.22) \\
0.7 & 0.7428(0.0281) & 0.6498(0.0340) & 0.9201(0.0159) & 3.32(0.22) \\
0.8 & 0.4557(0.0407) & 0.3523(0.0353) & 0.6864(0.0497) & 1.20(0.13) \\
0.9 & 0.4557(0.0407) & 0.3523(0.0353) & 0.6864(0.0497) & 1.20(0.13) \\
\bottomrule
\end{tabular}
\end{table}

\subsection{Computational Time Analysis of sBayFDNN on Simulated Data}
\label{app:time-cost}

Tables~\ref{tab:cost-tobs-sbay-repeat5} and~\ref{tab:cost-n-sbay-repeat5} report the computational cost of the proposed method for varying numbers of observed points $T_{obs}$ and sample size $n$, with the running time of sBayFDNN decomposed into projection, training, and evidence evaluation. It is observed that with a fixed sample size, increasing $T_{obs}$ from 200 to 2000 has little effect on the overall runtime. In particular, the projection step is negligible throughout, and the total runtime is dominated by training and evidence computation.
This suggests that, once the functional input has been projected onto a fixed basis representation, the computational burden depends much more on the downstream optimization procedure than on the raw observation grid size itself.

By contrast, the sample size $n$ has a more visible impact on runtime.
At fixed $T_{obs}=200$, increasing $n$ from 1000 to 3000 leads to a clear increase in both training and evidence computation time, and hence in the overall runtime per repeat.
This is consistent with the fact that the optimization and model-evaluation steps scale primarily with the number of samples rather than with the original number of observed time points.

% The goal is to assess whether the proposed framework remains computationally feasible when the functional curves are observed on dense grids.

\begin{table}[h]
\centering
\small
\setlength{\tabcolsep}{8pt}
\caption{Runtime vs. observed points $T_{obs}$ at fixed $n=1000$. Entries are mean(sd) over 5 repeats (seconds).}
\label{tab:cost-tobs-sbay-repeat5}
\begin{tabular}{ccccc}
\toprule
$T_{obs}$ & Projection ($\times 10^{-3}$ s) & Training & Evidence & Runtime/Repeat \\
\midrule
200 & 2.0 (5.0) & 83.8 (22.1) & 41.3 (3.6) & 125.1 (22.8) \\
400 & 2.0 (4.0) & 80.3 (24.8) & 41.5 (7.6) & 121.7 (20.1) \\
600 & 2.0 (4.0) & 105.9 (66.0) & 38.8 (7.0) & 144.7 (62.8) \\
1000 & 3.0 (6.0) & 68.4 (18.0) & 36.3 (3.7) & 104.7 (15.1) \\
1500 & 2.0 (4.0) & 68.0 (18.5) & 36.2 (3.0) & 104.2 (15.8) \\
2000 & 2.0 (5.0) & 67.6 (19.3) & 36.1 (3.1) & 103.7 (16.7) \\
\bottomrule
\end{tabular}
\end{table}

\begin{table}[h]
\centering
\small
\setlength{\tabcolsep}{4pt}
\caption{Runtime vs. sample size $n$ at fixed $T_{obs}=200$. Entries are mean(sd) over 5 repeats (seconds).}
\label{tab:cost-n-sbay-repeat5}
\begin{tabular}{ccccc}
\toprule
$n$ & Projection ($\times 10^{-3}$ s) & Train & Evidence & Runtime/Repeat \\
\midrule
1000 & 2.0 (5.0) & 83.8 (22.1) & 41.3 (3.6) & 125.1 (22.8) \\
2000 & 2.0 (5.0) & 122.2 (83.3) & 80.3 (56.9) & 202.5 (58.3) \\
3000 & 3.0 (6.0) & 151.2 (97.5) & 49.4 (11.5) & 200.6 (96.3) \\
\bottomrule
\end{tabular}
\end{table}

{\subsection{Cross-seed Stability of Plug-in PIPs and Region Selection}}
\label{app:cross-seed-stability}

To assess whether plug-in PIPs and the resulting selected regions are sensitive to random initialization and optimization path, we perform a cross-seed stability analysis under the simple $\beta(t)$ and composite $g$ setting at SNR $=10$ and SNR $=5$.
For each scenario, we generate 50 independent data replications, and within each replication we refit the model using 10 random seeds while keeping the dataset and all other settings fixed.

For each replication, we summarize cross-seed variability under two reference schemes: comparison to the seed-0 fit and comparison to the cross-seed mean. We report both PIP-level and selection-level stability metrics.
Specifically, \emph{PIP Corr} and \emph{PIP MeanDiff} compare the estimated PIP curves across seeds, measuring similarity in shape and average numerical discrepancy, respectively.
After thresholding the PIP curves to obtain selected regions, we further compute \emph{Interval SymDiff}, namely the symmetric-difference length between two selected supports, and $\Delta$F1, which measures the seed-induced change in support recovery performance.
These replication-level summaries are then aggregated over the 50 replications, and the resulting distributions are shown in Figure~\ref{fig:cross_seed_stability}.

\begin{figure}[t]
\centering
\includegraphics[width=0.85\textwidth]{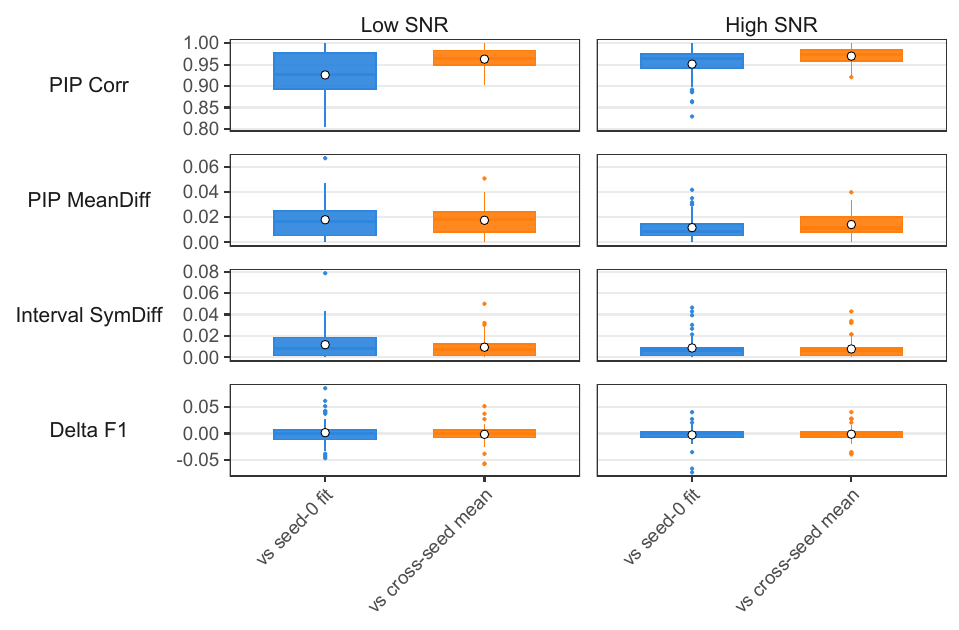}
\caption{
Cross-seed stability of the plug-in PIPs and the resulting selected regions under the simple $\beta(t)$ and complex $g$ setting, shown for both low and high SNR.
Higher values indicate better stability for \textit{PIP Corr},
whereas lower values indicate better stability for \textit{PIP MeanDiff} and \textit{Interval SymDiff};
for $\Delta$F1, values closer to zero indicate smaller seed-induced variation.
}
\label{fig:cross_seed_stability}
\end{figure}

Figure~\ref{fig:cross_seed_stability} shows that cross-seed variability is low under both low and high SNR scenarios.
In particular, the PIP correlations are consistently high, while the mean PIP differences and interval symmetric differences remain small.
Moreover, $\Delta$F1 stays close to zero across seeds, indicating that the final region-selection performance is only mildly affected by initialization and optimization path.
Overall, the plug-in PIP profiles and the resulting selected regions remain highly stable across seeds, with only slightly larger variability under low SNR scenario.
These results provide empirical support that the MAP-based plug-in PIPs and the resulting selected regions are stable in practice.

\section{Proof of Statistical Properties}
We divide the theoretical proof into four parts. First, we establish the theoretical foundation for finite-dimensional approximations of functional data, which lays the groundwork for the subsequent theorems. Then, in the second, third, and fourth parts, we provide proofs for the bound of the approximation error (Theorem \ref{theorem_1}), the posterior consistency (Theorem \ref{theorem:Posterior_consistency}), and the selection consistency (Theorem \ref{thm:selection_consistency}), respectively. Prior to proving each theorem, we present necessary lemmas. Throughout, $C, C', C_1, \ldots$ denote generic positive constants whose
values may change from line to line.

\subsection{Finite-dimensional Approximations}

First, recall and introduce some notations. Let $\beta^*$ and $\beta_{J_n}^*=\sum_{j=1}^{J_n}\omega^*_{J_n,j} B_j(t)$ denote the true coefficient function and its truncated counterpart, respectively, where $\boldsymbol{\omega}_{J_n}^*=(\omega^*_{J_n,1},\cdots,\omega^*_{J_n,J_n})^T$ is the vector of true basis coefficients for $\beta_{J_n}^*$. Denote $S^*_{J_n} = \operatorname{supp}(\boldsymbol{\omega}^*_{J_n})$ and $\Omega_{J_n}^*=\cup_{j\in S^*_{J_n}}I_j$ with $I_j:=\operatorname{supp}(B_j)$. In addition, let
$u^*(X):=\langle X,\beta^*\rangle, \mu^*(X):=g^*(u^*(X)), u^*_{J_n}(X):=\langle X,\beta^*_{J_n}\rangle$, and $\mu^*_{J_n}(X):=g^*(u^*_{J_n}(X))$.
Moreover, for $\kappa>0$, define the strong-signal region $\Omega^*(\kappa):=\{t\in[0,1]:|{\beta^*}(t)|>\kappa\}$.
Define $S^*_{J_n}(\kappa):=\{j\in[{J_n}]: I_j\cap \Omega^*(\kappa)\neq\varnothing\}$ and $\Omega^*_{J_n}(\kappa):=\bigcup_{j\in S^*_{J_n}(\kappa)} I_j$.
Let $\Delta_{J_n}:=\max_{1\le j\le J_n}\mathrm{diam}(I_j)\asymp J_n^{-1}$ for fixed spline order. For a measurable set $A\subset[0,1]$ and $\Delta>0$, define its $\Delta$-enlargement by $A^{+\Delta}:= \{t\in[0,1]: \operatorname{dist}(t,A)\le \Delta\}.$
For the strong-signal cover $\Omega^*_{J_n}(\kappa)$, we will also use the enlarged cover $\Omega^*_{J_n}(\kappa)^{+c_{\rm loc}\Delta_{J_n}}$,
where $c_{\rm loc}\ge 1$ is an absolute constant depending only on the locality radius of the spline quasi-interpolant/projection used below.

\begin{lemma}
\label{prop:block1-main}
  Suppose Assumption~\ref{assumption_design_boundedness}--\ref{assumption_Holder_link} holds. 
  Define $\kappa_{J_n}:=c_\kappa\,C_\beta J_n^{-\alpha_\beta}$ with $c_\kappa\ge 4$ and $C_{\beta}$ being some constant. Then for all sufficiently large $n$, we have
\begin{enumerate}[label=\textup{(\roman*)},leftmargin=2.2em]
\item Uniform approximation at resolution $J_n$:
  \begin{equation}\label{eq:block1-linf}
    \|{\beta^*}-\beta^*_{J_n}\|_\infty \ \le\ C\,J_n^{-\alpha_\beta}\ \lesssim\ \kappa_{J_n}.
  \end{equation}
  Moreover, on the strong-signal set $\Omega^*(2\kappa_{J_n})$:
  \begin{equation}\label{eq:block1-linf-strong}
    \|{\beta^*}-\beta^*_{J_n}\|_{L^\infty(\Omega^*(2\kappa_{J_n}))} \ \le\ C_\beta\,J_n^{-\alpha_\beta}
    \ \le\ \kappa_{J_n}/c_\kappa,
  \end{equation}
where $\|v(t)\|_{L^\infty(\Omega^*(2\kappa_{J_n}))}=\sup_{t\in \Omega^*(2\kappa_{J_n})} |v(t)|$.
\item Localization to the strong-signal cover:
There exists an absolute constant $c_{\rm loc}\ge 1$ and $\Delta_{J_n}\asymp J_n^{-1}$ such that, for all sufficiently large $n$,
\begin{equation}\label{eq:block1-localize}
  \Omega_{J_n}^*\ \subseteq\ \Omega^*_{J_n}(\kappa_{J_n})^{+c_{\rm loc}\Delta_{J_n}}.
\end{equation}

\item Index and link truncation: letting $\alpha_1=\min(\alpha_g,1)$,
  \begin{equation}\label{eq:block1-index-trunc}
    \sup_{X\in\mathcal X}|u^*_{J_n}(X)-u^*(X)|
    \le C_X\|\beta^*_{J_n}-{\beta^*}\|_\infty
    \lesssim J_n^{-\alpha_\beta},
  \end{equation}
  and
  \begin{equation}\label{eq:block1-link-trunc}
    \sup_{X\in\mathcal X}|\mu^*_{J_n}(X)-\mu^*(X)|
    =\sup_{X\in\mathcal X}|g^*(u^*_{J_n}(X))-g^*(u^*(X))|
    \le L_g\Bigl(\sup_{X\in\mathcal X}|u^*_{J_n}(X)-u^*(X)|\Bigr)^{\alpha_1}
    \lesssim J_n^{-\alpha_\beta\alpha_1}.
  \end{equation}
\item Strong-signal sandwich:
\begin{equation}\label{eq:block1-sandwich}
  \Omega^*(2\kappa_{J_n})\subseteq \Omega_{J_n}^*
  \subseteq \Omega^*_{J_n}(\kappa_{J_n})^{+c_{\rm loc}\Delta_{J_n}}.
\end{equation}
\item Gap decomposition (no extra structure assumed):
\begin{equation}\label{eq:block1-gap-decomp}
  \bigl|\Omega_{J_n}^* \,\Delta\, \Omega^*(2\kappa_{J_n})\bigr|
  \ \le\
  \underbrace{\bigl|\Omega^*(\kappa_{J_n})\setminus \Omega^*(2\kappa_{J_n})\bigr|}_{\text{threshold band}}
  \ +\
  \underbrace{\bigl|\Omega^*_{J_n}(\kappa_{J_n})^{+c_{\rm loc}\Delta_{J_n}}\setminus \Omega^*(\kappa_{J_n})\bigr|}_{\text{resolution boundary layer}}.
\end{equation}
  \end{enumerate}
\end{lemma}

\begin{proof}
\emph{Step 1 (spline approximation and the choice of $Q_{J_n}$).}
Fix a local B-spline basis $\{B_j\}_{j=1}^{J_n}$ of fixed order on $[0,1]$ with supports $I_j$ and
$\Delta_{J_n}:=\max_{1\le j\le J_n}\mathrm{diam}(I_j)\asymp J_n^{-1}$.
Let $Q_{J_n}$ denote a \emph{local} spline quasi-interpolant/projection onto the corresponding spline space
(e.g., a standard quasi-interpolant associated with the B-spline partition).
By standard $L_\infty$ spline approximation theory, for any $\beta\in\mathcal H^{\alpha_\beta}([0,1])$,
\[
  \|\beta - Q_{J_n}\beta\|_\infty \le C_\beta J_n^{-\alpha_\beta}.
\]
In the sequel, we work with this fixed operator $Q_{J_n}$.

\emph{Step 2 (localization with an enlarged cover).}
Recall $\kappa_{J_n}=c_\kappa C_\beta J_n^{-\alpha_\beta}$.
Construct a smooth cut-off function $\chi_{J_n}:[0,1]\to[0,1]$ such that
\[
  \chi_{J_n}(t)=1 \ \text{for } t\in\Omega^*(2\kappa_{J_n}),\qquad
  \chi_{J_n}(t)=0 \ \text{for } t\in\Omega^*(\kappa_{J_n})^c,
\]
and on the transition band $\Omega^*(\kappa_{J_n})\cap \Omega^*(2\kappa_{J_n})^c$, one has
$0<\chi_{J_n}(t)<1$ (with $\chi_{J_n}$ chosen smooth across the boundaries of these sets).
Define the truncated coefficient
\[
  \beta^{\rm loc}(t):=\beta^*(t)\chi_{J_n}(t).
\]
Then $\supp(\beta^{\rm loc})\subseteq \Omega^*(\kappa_{J_n})$ and, since $1-\chi_{J_n}$ vanishes on
$\Omega^*(2\kappa_{J_n})$ and is supported inside $\Omega^*(\kappa_{J_n})^c\cup
\bigl(\Omega^*(\kappa_{J_n})\cap \Omega^*(2\kappa_{J_n})^c\bigr)$,
\[
  \|\beta^*-\beta^{\rm loc}\|_\infty
  =\sup_{t\in[0,1]}|\beta^*(t)\{1-\chi_{J_n}(t)\}|
  \le \sup_{t\in \Omega^*(\kappa_{J_n})^c}|\beta^*(t)|
  \le 2\kappa_{J_n}.
\]
Moreover, $\beta^{\rm loc}\in\mathcal H^{\alpha_\beta}([0,1])$ with the same smoothness order, since
$\beta^*\in\mathcal H^{\alpha_\beta}$ and $\chi_{J_n}$ is smooth and bounded.

Now define
\[
  \beta^*_{J_n}:=Q_{J_n}\beta^{\rm loc}.
\]
By the locality of $Q_{J_n}$, there exists an absolute constant $c_{\rm loc}\ge 1$ such that
\[
  \supp(Q_{J_n}f)\subseteq \supp(f)^{+c_{\rm loc}\Delta_{J_n}}
  \qquad\text{for all bounded } f.
\]
Therefore,
\[
  \supp(\beta^*_{J_n})
  \subseteq \supp(\beta^{\rm loc})^{+c_{\rm loc}\Delta_{J_n}}
  \subseteq \Omega^*(\kappa_{J_n})^{+c_{\rm loc}\Delta_{J_n}}
  \subseteq \Omega^*_{J_n}(\kappa_{J_n})^{+c_{\rm loc}\Delta_{J_n}},
\]
where the last inclusion uses $\Omega^*(\kappa_{J_n})\subseteq \Omega^*_{J_n}(\kappa_{J_n})$.
Since $\Omega_{J_n}^*=\supp(\beta^*_{J_n})$ by local linear independence of the B-spline basis, we obtain
$\Omega_{J_n}^*\subseteq \Omega^*_{J_n}(\kappa_{J_n})^{+c_{\rm loc}\Delta_{J_n}}$,
which proves \eqref{eq:block1-localize}.

  For approximation, note that on $\Omega^*(2\kappa_{J_n})$, we have ${\beta^*}^{\rm loc}={\beta^*}$, hence
  \[
    \|{\beta^*}-\beta^*_{J_n}\|_{L^\infty(\Omega^*(2\kappa_{J_n}))}
    =\|{\beta^*}^{\rm loc}-Q_{J_n}{\beta^*}^{\rm loc}\|_{L^\infty(\Omega^*(2\kappa_{J_n}))}
    \le \|{\beta^*}^{\rm loc}-Q_{J_n}{\beta^*}^{\rm loc}\|_\infty
    \le C_\beta J_n^{-\alpha_\beta},
  \]
  which gives \eqref{eq:block1-linf-strong}. Globally,
  \[
    \|{\beta^*}-\beta^*_{J_n}\|_\infty
    \le \|{\beta^*}-{\beta^*}^{\rm loc}\|_\infty + \|{\beta^*}^{\rm loc}-Q_{J_n}{\beta^*}^{\rm loc}\|_\infty
    \le 2\kappa_{J_n} + C_\beta J_n^{-\alpha_\beta}
    \lesssim J_n^{-\alpha_\beta}.
  \]
  % since $\|{\beta^*}-{\beta^*}^{\rm loc}\|_\infty\le \kappa_{J_n}$ on $\Omega^*(\kappa_{J_n})^c$
  % by definition of $\Omega^*(\kappa_{J_n})$, and $\kappa_{J_n}\asymp J_n^{-\alpha_\beta}$.
  This proves \eqref{eq:block1-linf}.

  \emph{Step 3 (index and link truncation).}
  By Assumption~\ref{assumption_design_boundedness},
  \[
    |u^*_{J_n}(X)-u^*(X)|
    =\Bigl|\int_0^1 X(t)\bigl(\beta^*_{J_n}(t)-{\beta^*}(t)\bigr)\,dt\Bigr|
    \le \|X\|_{L^1}\|\beta^*_{J_n}-{\beta^*}\|_\infty
    \le C_X\|\beta^*_{J_n}-{\beta^*}\|_\infty,
  \]
  and \eqref{eq:block1-index-trunc} follows from \eqref{eq:block1-linf}. The link bound
  \eqref{eq:block1-link-trunc} follows from H\"older continuity of $g^*$.

  \emph{Step 4 (sandwich).}
  Take $t\in\Omega^*(2\kappa_{J_n})$. Then $|{\beta^*}(t)|>2\kappa_{J_n}$ and by \eqref{eq:block1-linf-strong},
  \[
    |\beta^*_{J_n}(t)| \ge |{\beta^*}(t)|-\|{\beta^*}-\beta^*_{J_n}\|_{L^\infty(\Omega^*(2\kappa_{J_n}))}
    > 2\kappa_{J_n}-\kappa_{J_n}/c_\kappa >0
  \]
  since $c_\kappa\ge 4$. Hence $t\in\operatorname{supp}(\beta^*_{J_n})$.
  For standard local B-spline bases (local linear independence), $\operatorname{supp}(\beta^*_{J_n})=\cup_{j\in S^*_{J_n}}I_j=\Omega_{J_n}^*$,
  so $\Omega^*(2\kappa_{J_n})\subseteq\Omega_{J_n}^*$.
  The right inclusion $\Omega_{J_n}^*\subseteq\Omega^*_{J_n}(\kappa_{J_n})^{+c_{\rm loc}\Delta_{J_n}}$
follows from \eqref{eq:block1-localize}.
  This proves \eqref{eq:block1-sandwich}.

  \emph{Step 5 (gap decomposition).}
  By \eqref{eq:block1-sandwich},
  \[
    \Omega_{J_n}^*\Delta\Omega^*(2\kappa_{J_n})
    \subseteq \Omega^*_{J_n}(\kappa_{J_n})^{+c_{\rm loc}\Delta_{J_n}}\setminus \Omega^*(2\kappa_{J_n})
    = \bigl(\Omega^*(\kappa_{J_n})\setminus\Omega^*(2\kappa_{J_n})\bigr)\ \cup\
    \bigl(\Omega^*_{J_n}(\kappa_{J_n})^{+c_{\rm loc}\Delta_{J_n}}\setminus\Omega^*(\kappa_{J_n})\bigr),
  \]
  and \eqref{eq:block1-gap-decomp} follows by subadditivity of Lebesgue measure.
\end{proof}

\subsection{Approximation Error Bounds (Theorem \ref{theorem_1})}
\begin{lemma}
\label{lem:yarotsky-1d}
  Suppose Assumption \ref{assumption_DNN} holds such that there exists $L_g$ satisfy
  $|g^*(u)-g^*(v)|\le L_g|u-v|^{\alpha_1}$ for all $u,v\in [0,1]$ with $\alpha_1=\min(\alpha_g,1)$.
  Then for any integer depth $H_n\ge 2$, there exists a univariate ReLU network $f_{H_n}$
  with constant width (at most a universal constant) and depth at most $H_n$ such that
  \[
    \sup_{u\in [0,1]}|f_{H_n}(u)-g^*(u)|\lesssim H_n^{-2\alpha_1}.
  \]
\end{lemma}

\begin{proof}
    This result is a direct specialization of Theorem~2 in
\citet{yarotsky2018optimal} to the one--dimensional setting.
Taking input dimension $d=1$ and approximation accuracy
$\varepsilon \asymp H_n^{-2\alpha_1}$ in that theorem yields
the stated rate with constant network width.
\end{proof}

\subsubsection{Proof of Theorem \ref{theorem_1}}

The proof is structured into three steps.
First, for the unknown link function $g^*(u)$, Lemma \ref{lem:yarotsky-1d} guarantees the existence of a one‑dimensional neural network $f_{H_n-1}$ with constant width and depth $H_n-1$ such that \[
    \sup_{u\in [0,1]}|f_{H_n-1}(u)-g^*(u)|\lesssim H_n^{-2\alpha_1}.
  \]
Furthermore, the network output $f_{H_n-1}(u_{J_n}^{*}(X))$, where $u_{J_n}^{*}(X)=(\boldsymbol{\omega}^*_{J_n})^\top \boldsymbol{\eta}(X)$, can be reinterpreted as an $H_n$-depth network $F_{\boldsymbol{\theta}}$ with $J_n$-dimensional input $\boldsymbol{\eta}(X)$. To construct $F_{\boldsymbol{\theta}}$, we explicitly design its first hidden layer to satisfy the support condition $\operatorname{supp}_{\mathrm{col}}(\boldsymbol{\theta}) = S^*_{J_n}$: we set the width of the first layer as $L_1\asymp 1$, the weights of the first neuron in the first hidden layer to be $\boldsymbol{W}_{1, 1*} = (\boldsymbol{\omega}_{J_n}^*)^\top$ and all other row weights $\boldsymbol{W}_{1, j*} = \mathbf{0}^\top$ for $j > 1$. The remaining layers then implement $f_{H_n-1}$ acting on the computed scalar $(\boldsymbol{\omega}^*_{J_n})^\top \boldsymbol{\eta}(X)$. This yields a DNN $F_{\boldsymbol{\theta}}$ with depth at most $H_n$, constant width up to universal constants, and $\operatorname{supp}_{\mathrm{col}}(\boldsymbol{\theta}) = S^*_{J_n}$, establishing (i).

Finally, to prove (ii), we decompose the approximation error into a network approximation error and a functional truncation error. For any $X \in \mathcal{X}$,
$$
\begin{aligned}
|\mu_{\boldsymbol{\theta}}(X) - \mu^*(X)| &= |F_{\boldsymbol{\theta}}(\boldsymbol{\eta}(X)) - g^*(u^*(X))| \\
&=|f_{H_n-1}(u_{J_n}^*(X)) - g^*(u^*(X))|\\
&\le \underbrace{|f_{H_n-1}(u_{J_n}^*(X)) - g^*(u_{J_n}^*(X))|}_{I} 
+ \underbrace{|g^*(u_{J_n}^*(X)) - g^*(u^*(X))|}_{II}.
\end{aligned}
$$
Term I is bounded directly by Lemma \ref{lem:yarotsky-1d}, while term II is controlled via inequality (\ref{eq:block1-link-trunc}) from Lemma \ref{prop:block1-main}. Taking the supremum over $\mathcal X$ then yields the final convergence rate stated in (ii).
%where $u_{J_n}^*(X):=\int_{\mathcal{T}}X(t)\beta_{J_n}^*(t)dt$ and $u^*(X):=\int_{\mathcal{T}}X(t)\beta^*(t)dt$.

\subsection{Posterior Consistency (Theorem \ref{theorem:Posterior_consistency})}

We first give a general posterior consistency result introduced in \citet{jiang2007bayesian}. Specifically, 
let $D_n=\{(\boldsymbol{x}_i,Y_i)\}_{i=1}^n$ denote the dataset, where $(\boldsymbol{x}_i,Y_i)$ are i.i.d.\ under the reference distribution $p^*$.
Let $\mathcal P$ denote the space of probability
densities under consideration. We consider a sequence of model classes (sieves)
$P_n \subset \mathcal P$, and write $P_n^c=\mathcal P\setminus P_n$ for
their complements. We construct $P_n$ through a parameter sieve $\Theta_n$ via
\[
  P_n:=\{p_\theta:\theta\in\Theta_n\}.
\]

Let $\Pi$ denote the prior measure on $\mathcal P$ (or on $\Theta$ via $p_\theta$),
and let $\Pi(\cdot\mid D_n)$ denote the corresponding posterior given the data $D_n$.
For each $\varepsilon>0$, define the posterior probability
\[
  \widehat\Pi(\varepsilon)
  \;:=\;
  \Pi\bigl(d(p,p^*)>\varepsilon \mid D_n\bigr),
\]
where the metric $d(\cdot,\cdot)$ denotes the Hellinger distance, defined by
$d(p,q)=\sqrt{\int \bigl(\sqrt{p}-\sqrt{q}\bigr)^2}$. Let
$N(\varepsilon, P_n, d)$ denote the $\varepsilon$-covering number of $P_n$ with respect to the metric $d$. 

%{\color{red}Let $P^*$ and $E^*$ denote the probability measure and expectation under the
%true data-generating process of $D_n$, respectively.} 

% Let $\mathcal{P}$ denote the space of probability densities under consideration, and let $P_n := \{p_\theta : \theta \in \Theta_n\} \subset \mathcal{P}$ be a sequence of model classes, where
% $\Theta_n := \{ \theta: \|\theta\|_\infty \le M_n, C(\theta) \le k_0 s_n \}$
% for a constant $k_0 \ge 2$.
% Here, $C(\theta):=|S(\theta)|$ with $S(\theta):=\Bigl\{j\le J_n:\max_{1\le h\le L_1}|W_{1,hj}|\ge \delta'_n\Bigr\}$. 
% % Let $\zeta:=((\operatorname{vec}(W_1))^\top,\dots,(\operatorname{vec}(W_{H_n}))^\top)^\top$ be the vector of all weights.
% Let $N(\varepsilon, P_n, d)$ denote the $\varepsilon$-covering number of $P_n$ with respect to the metric $d$.

\begin{lemma}\label{lem:general-ggv}
For a sequence $\varepsilon_n \to 0$, if there exist constants $2 > b > 2b' > 0$ and $t > 0$ such that the following conditions hold for all sufficiently large $n$:
\begin{enumerate}[label=(\alph*)]
    \item $\log N(\varepsilon_n, P_n, d) \le n\varepsilon_n^2$;
    \item $\pi(P_n^c) \le \exp(-b n \varepsilon_n^2)$;
    \item $\pi \{ p \in \mathcal{P} : d_t(p, p{^*}) \le b' \varepsilon_n^2 \} \ge \exp(-b' n \varepsilon_n^2)$,
\end{enumerate}
where $d_t(p,p^*)
  \;=\;
  \frac{1}{t}\Bigl(
  \int p^*(x)\Bigl(\tfrac{p^*(x)}{p(x)}\Bigr)^t\,dx
  \;-\; 1
  \Bigr)$, then for any $2b' < x < b$, the posterior probability $\widehat{\Pi}(4\varepsilon_n)$ satisfies:
\begin{enumerate}[label=(\roman*)]
    \item $P \left[ \widehat{\Pi}(4\varepsilon_n) \ge 2 \exp\left( -\frac{1}{2} n \varepsilon_n^2 m(x) \right) \right] \le 2 \exp\left( -\frac{1}{2} n \varepsilon_n^2 m(x) \right)$,
    \item $\mathbb{E} \left[ \widehat{\Pi}(4\varepsilon_n) \right] \le 4 \exp\left( -n \varepsilon_n^2 m(x) \right)$,
\end{enumerate}
where $m(x) := \min \{1, 2-x, b-x, t(x-2b')\}$.
\end{lemma}

\begin{proof}
    The proof follows from an argument analogous to that of Proposition 1 in \citet{jiang2007bayesian}.
\end{proof}

\begin{lemma}
  \label{lem:perturbation_sparse}
   Fix any subset $S\subset\{1,\dots,J_n\}$ with $m:=|S|$.  Consider the proposed DNN $F_{\boldsymbol{\theta}}$ defined in (\ref{DNN}) with input of dimension $J_n$. Let $\boldsymbol{\theta}$ be a network parameter vector with $\|\boldsymbol{\theta}\|_\infty\le E_n$. 
  Let $\tilde{\boldsymbol{\theta}}$ be another parameter vector such that
  \[
    \bigl|W_{1,hg}-\tilde{W}_{1,hg}\bigr|\le
    \begin{cases}
      \delta_1, & g\in S,    \\
      \delta_2, & g\notin S,
    \end{cases}
    \qquad 1\le h\le L_1,\ 1\le g\le J_n,
  \]
  and for all deeper-layer coordinates $\mathcal G$, including all biases and weights beyond the first layer,
  \[
    \max_{j\in\mathcal G}\bigl|\theta_j-\tilde\theta_j\bigr|\le \delta_1 .
  \]
  Then, for all $\boldsymbol{x}\in[-1,1]^{J_n}$,
  \[
    \begin{aligned}
      \bigl|F_{\boldsymbol{\theta}}(\boldsymbol{x})-F_{\tilde{\boldsymbol{\theta}}}(\boldsymbol{x})\bigr|
      \;\le\; &
      (E_n+\delta_1)^{H_n-1}
      \Bigl[
      H_n\,(m+1)L_1\prod_{k=2}^{H_n}(L_{k-1}+1)L_k\;\delta_1 \\
              & \qquad\quad
        +\Bigl\{(J_n-m)L_1\prod_{k=2}^{H_n}(L_{k-1}+1)L_k\Bigr\}\delta_2
        \Bigr].
    \end{aligned}
  \]
\end{lemma}

\begin{proof}
Consider the pre-activation at the first layer $z_{1,h} = \sum_{g=1}^{J_n} W_{1,hg} x_g + b_{1,h}$. For any $\boldsymbol{x} \in [-1,1]^{J_n}$, the difference satisfies:
\begin{equation*}
|z_{1,h} - \tilde{z}_{1,h}| \le \sum_{g \in S} \delta_1 |x_g| + \sum_{g \notin S} \delta_2 |x_g| + \delta_1 \le (m+1)\delta_1 + (J_n-m)\delta_2.
\end{equation*}
For $k \ge 2$, let $\boldsymbol{z}_{k}$ denote the pre-activation vector of layer $k$. Since ReLU is $1$-Lipschitz, the error propagates as:
\begin{equation*}
\|\boldsymbol{z}_{k} - \tilde{\boldsymbol{z}}_{k}\|_\infty \le \|\boldsymbol{W}_k\|_\infty \|\boldsymbol{z}_{k-1} - \tilde{\boldsymbol{z}}_{k-1}\|_\infty + \|\boldsymbol{W}_k - \tilde{\boldsymbol{W}}_k\|_\infty \|\tilde{\boldsymbol{z}}_{k-1}\|_\infty + \|\boldsymbol{b}_k - \tilde{\boldsymbol{b}}_k\|_\infty.
\end{equation*}
Note that $\|\tilde{a}_{k-1}\|_\infty \le (E_n + \delta_1)^{k-1}$. By induction over $k=2, \dots, H_n$ and accounting for the total number of parameters in each layer (width products $\prod L_k$), the perturbations $\delta_1$ across $H_n$ layers accumulate linearly. The initial perturbation from inactive columns $(J_n-m)\delta_2$ is magnified by the depth-induced factor $(E_n + \delta_1)^{H_n-1}$. Summing these contributions yields the desired bound.
\end{proof}

\subsubsection{Definition of $\tilde{M}_{n,1}(\varepsilon_{n})$ and $\tilde{M}_{n,2}(\varepsilon_{n})$ in Theorem \ref{theorem:Posterior_consistency}}
\begin{itemize}
\item $\sigma_{0,n}^2 \le \tilde{M}_{n,1}(\varepsilon_n)$;
\item $\max\{\sigma^2,\sigma_{0,n}^2,\sigma_{1,n}^2\}\le \tilde{M}_{n,2}(\varepsilon_n)$.
\end{itemize}
Here, 
\begin{equation}\label{M_1}
\tilde{M}_{n,1}(\varepsilon_{n})=\min\left\{ \frac{(\delta'_n)^2}{2\tau A_n + 2\log(4J_nL_1^2)}, \frac{(\omega_n')^2}{2\log(4J_nL_1^2)} \right\},
\end{equation}
and 
\begin{equation}\label{M_2}
\tilde{M}_{n,2}(\varepsilon_{n})= \frac{M_n^{2}}{2\Bigl[b_1\,n\varepsilon_n^{2}+\log(2K_n)\Bigr]},
\end{equation}
where $A_n= H_n\log n+H_n\log\bar L+\log\{(J_n+1)L_1\}$, 
\begin{equation}\label{delta}
\delta'_n =\frac{c_{1}\varepsilon_{n}}{H_n\,J_n\,L_1\,(\bar L)^{2(H_n-1)}(c_0M_n)^{H_n-1}},
\end{equation} 
and
\begin{equation}\label{omega}
\omega_n' =\frac{c_{1}\varepsilon_{n}}{J_n\,L_1\,(\bar L)^{2(H_n-1)}(c_0E_n)^{H_n-1}},
\end{equation}
with  
  $c_0$ and $c_1$ being some positive constants. In addition, $K_n=(J_n+1)L_1+H_n\bar L^2$, $\log M_n = O(\log n)$, and that for sufficiently large $n$, $M_n \ge E_n$.

\subsubsection{Proof of Theorem \ref{theorem:Posterior_consistency}}

We consider the specific scenario developed in Section \ref{sec:methodology}. The observed data is $D_n=\{(\boldsymbol{x}_i(t),Y_i)\}_{i=1}^n$. 

%where $\boldsymbol{x}_i$ represents projected features:
% \[
%   \boldsymbol{x}_i=\boldsymbol{\eta}(X_i),\qquad
%   x_{ij}=\eta_j(X_i)=\frac{1}{B_z}\int_0^1 X_i(t)\,B_j(t)\,dt.
% \]

Throughout the posterior contraction analysis, we take the reference truth to be the original law
\[
  p^* \;:=\; p_{\mu^*},
\]
where the model family is indexed by $\mu_{\boldsymbol{\theta}}(X):=F_{\boldsymbol{\theta}}(\boldsymbol{\eta}(X))$ and the true mean is
$\mu^*(X)=g^*(\langle X,\beta^*\rangle)$. Here, $p_{\mu^*}(\cdot\mid X)$ is the true conditional density of $Y$ given $X$. Let $P_n := \{p_{\mu_{\boldsymbol{\theta}}}: \boldsymbol{\theta} \in \Theta_n\} \subset \mathcal{P}$ be a sequence of model classes, where $p_{\mu_{\boldsymbol{\theta}}}(\cdot\mid X)$ is the approximate density induced by the finite-dimensional representation (\ref{eq:finite_rep}) and the sparse DNN defined in (\ref{DNN}) and $\Theta_n := \{\boldsymbol{\theta}: \|\boldsymbol{\theta}\|_\infty \le M_n, C(\boldsymbol{\theta}) \le k_0 s_n \}$
for a constant $k_0 \ge 2$.
Here, $C(\boldsymbol{\theta}):=|S(\boldsymbol{\theta})|$ with $S(\boldsymbol{\theta}):=\Bigl\{j\le J_n:\max_{1\le h\le L_1}|W_{1,hj}|\ge \delta'_n\Bigr\}$.

To prove Theorem \ref{theorem:Posterior_consistency}, according to Lemma~\ref{lem:general-ggv}, it suffices to verify that for a sequence $ \varepsilon_n^2
  \;\lesssim\;
  \frac{s_n\log(J_n/s_n)}{n}
  +\frac{s_n\bigl(H_n\log n+\log J_n\bigr)}{n}
  +\Bigl(H_n^{-2\alpha_1}+J_n^{-\alpha_\beta\alpha_1}\Bigr)^2$, the following conditions hold for all sufficiently large $n$:
\begin{enumerate}[label=(\alph*)]
    \item $\log N(\varepsilon_n, P_n, d) \le n\varepsilon_n^2$;
    \item $\pi(P_n^c) \le \exp(-b n \varepsilon_n^2)$;
    \item $\pi \{ p_{\mu_{\boldsymbol{\theta}}} \in \mathcal{P} : d_t(p_{\mu_{\boldsymbol{\theta}}}, p_{\mu^*}) \le b' \varepsilon_n^2 \} \ge \exp(-b' n \varepsilon_n^2)$.
\end{enumerate}

\paragraph{Verification of condition (a).}
Set
% \begin{equation}\label{eq:deltas-two-stage}
%   \delta_{1,n}
%   := \frac{c_2\varepsilon_n}{H_n\,(k_0s_n+1)\,L_1\,(\bar L)^{2(H_n-1)}(c_0M_n)^{H_n-1}},
%   \qquad c_2\in(0,1).
% \end{equation}
\[\delta'_n =\frac{c_{1}\varepsilon_{n}}{H_n\,J_n\,L_1\,(\bar L)^{2(H_n-1)}(c_0M_n)^{H_n-1}}.\]
% {\color{red}This choice ensures that perturbations of size $\delta_{1,n}$ on the active first-layer weights and on all remaining (non--first-layer) coordinates induce an $O(\varepsilon_n)$ perturbation of the network output; for inactive first-layer columns, no extra discretisation is needed since $|W_{1,hg}|<\delta'_n$ holds by definition
% of $S(\theta)$.}

Fix $\boldsymbol{\theta}\in \Theta_n$ and let $S:=S(\boldsymbol{\theta})$ with $|S|\le k_0s_n$.
Define the truncated parameter $\boldsymbol{\theta}^{(S)}$
% $\theta^{(S)}:=((\operatorname{vec}(W_1^{(S)}))^\top,\dots,(\operatorname{vec}(W_{H_n}^{(S)}))^\top)^\top$ 
by zeroing out non-activated columns:
\[
  W^{(S)}_{1,hg}:=W_{1,hg}\mathbf 1\{g\in S\},\qquad
  1\le h\le L_1,\ 1\le g\le J_n,
\]
and keep all remaining coordinates unchanged: $\theta^{(S)}_{j}:=\theta_j$ for $j \in \mathcal{G}$ ($\mathcal G$ includes all biases and weights beyond the first layer).
Then for $g\notin S$,
\[
  \max_{h\le L_1}\bigl|W_{1,hg}-W^{(S)}_{1,hg}\bigr|
  = \max_{h\le L_1}|W_{1,hg}|
  <\delta'_n.
\]
Applying Lemma~\ref{lem:perturbation_sparse} with $(\delta_1,\delta_2)=(0,\delta'_n)$ yields
\begin{equation}\label{eq:bridge-trunc}
  \sup_{\boldsymbol{x}\in[-1,1]^{J_n}}|F_{\boldsymbol{\theta}}(\boldsymbol{x})-F_{\boldsymbol{\theta}^{(S)}}(\boldsymbol{x})|
  \le {C\,\varepsilon_n},
\end{equation}
with $C$ being some constant.

Now consider the truncated class
\[
  \Theta_n^{\rm trunc}(S)
  :=\Bigl\{\boldsymbol{\theta}:\ \|\boldsymbol{\theta}\|_\infty\le M_n,\boldsymbol{W}_{1,* g}\equiv 0\ \text{for }g\notin S\Bigr\}.
\]
Set
\begin{equation}\label{eq:deltas-two-stage}
  \delta_{1,n}
  := \frac{c_2\varepsilon_n}{H_n\,(k_0s_n+1)\,L_1\,(\bar L)^{2(H_n-1)}(c_0M_n)^{H_n-1}},
  \qquad c_2\in(0,1).
\end{equation}

Let $\tilde{\boldsymbol{\theta}}\in \Theta_n^{\rm trunc}(S)$
% $\tilde{\theta}:=((\operatorname{vec}(\tilde{W}_1))^\top,\dots,(\operatorname{vec}(\tilde{W}_{H_n}))^\top)^\top\in \Theta_n^{\rm trunc}(S)$ 
satisfy the coordinate-wise bounds
\[
  \max_{g\in S,\ h\le L_1}|W^{(S)}_{1,hg}-\tilde{W}_{1,hg}|\le \delta_{1,n},
  \qquad
  \max_{j\in\mathcal G}|\theta^{(S)}_{j}-\tilde{\theta}_{j}|\le \delta_{1,n}.
\]
For $g\notin S$, we have $W^{(S)}_{1,hg}=\tilde{W}_{1,hg}=0$, hence the same bound holds
with $\delta_2=0\le \delta'_n$. Therefore Lemma~\ref{lem:perturbation_sparse}
(with $(\delta_1,\delta_2)=(\delta_{1,n},\delta'_n)$) and \eqref{eq:deltas-two-stage} give
\begin{equation}\label{eq:grid-step}
\sup_{\boldsymbol{x}\in[-1,1]^{J_n}}|F_{\boldsymbol{\theta}^{(S)}}(\boldsymbol{x})-F_{\tilde{\boldsymbol{\theta}}}(\boldsymbol{x})|
  \le {C\,\varepsilon_n}.
\end{equation}
Combining \eqref{eq:bridge-trunc}--\eqref{eq:grid-step},
\[
  \sup_{\boldsymbol{x}\in[-1,1]^{J_n}}|F_{\boldsymbol{\theta}}(\boldsymbol{x})-F_{\tilde{\boldsymbol{\theta}}}(\boldsymbol{x})|
  \le C\,\varepsilon_n .
\]
As $\mu_{\boldsymbol{\theta}}(X)=F_{\boldsymbol{\theta}}(\boldsymbol{x})$ with $\boldsymbol{x}=\boldsymbol{\eta}(X)$, we further have
\[
  \sup_{X \in \mathcal{X} }|\mu_{\boldsymbol{\theta}}(X)-\mu_{\tilde{\boldsymbol{\theta}}}(X)|
  \le C\,\varepsilon_n .
\]
Next, since $d^2(p,q)\le d_0(p,q)$,
it suffices to control the Kullback--Leibler divergence. 
For the Gaussian regression case, the KL divergence is bounded by the squared difference of mean functions. Hence,
% For the normal likelihood,
% $d_0(p_\theta(\cdot|x),p_{\tilde\theta}(\cdot|x))  =(\mu(\theta,x)-\mu(\tilde\theta,x))^2/(2\sigma^2)$.
% ; for the Bernoulli--logistic likelihood,
% $d_0(p_\theta(\cdot|x),p_{\tilde\theta}(\cdot|x))\le (\mu(\theta,x)-\mu(\tilde\theta,x))^2/8$.
% Hence
 % in both cases,

\[
  d_0(p_{\mu_{\boldsymbol{\theta}}},p_{\mu_{\tilde{\boldsymbol{\theta}}}})
  \le C\,\mathbb{E}_X\bigl[\mu_{\boldsymbol{\theta}}(X)-\mu_{\tilde{\boldsymbol{\theta}}}(X)\bigr]^2
  \le C'\varepsilon_n^2,
\]

and therefore $d(p_\theta,p_{\tilde\theta})\le C''\varepsilon_n$ for all large $n$.
Consequently, any coordinate-wise $\ell_\infty$--net for $\Theta_n^{\rm trunc}(S)$
induces a $c\,\varepsilon_n$--net
(for some constant $c>0$) under $d$. Since $N(\varepsilon;P_n,d)$ is non-increasing in
$\varepsilon$, it suffices to bound $N(c\,\varepsilon_n;P_n,d)$, which is of the same order.

Fix $S\subset\{1,\dots,J_n\}$ with $m:=|S|\le k_0s_n$ and let $K_{\rm deep}:=|\mathcal G|\le H_n\bar L^2$.
Discretise (i) the $mL_1$ active first-layer weights and (ii) the $K_{\rm deep}$ deep-layer
coordinates on a uniform grid over $[-M_n,M_n]$ with mesh width $\delta_{1,n}$.
Let $R_n:=\Bigl\lceil\frac{2M_n}{\delta_{1,n}}\Bigr\rceil$.
By product covering, there exists an $c\varepsilon_n$--net $\mathcal N(S)$ of
$\{p_\theta:\theta\in \Theta_n^{\rm trunc}(S)\}$
under $d$ such that $|\mathcal N(S)|\le R_n^{mL_1+K_{\rm deep}}$.

Since $m\le k_0s_n$,
\[
  \sum_{m=0}^{k_0s_n}\binom{J_n}{m}
  \le\Bigl(\frac{eJ_n}{k_0s_n}\Bigr)^{k_0s_n}.
\]
Let $\mathcal N:=\bigcup_{|S|\le k_0s_n}\mathcal N(S)$.
For any $\theta\in \Theta_n$, letting $S=S(\theta)$, the truncation bridge and the net $\mathcal N(S)$
yield some $\tilde\theta\in\mathcal N(S)\subset\mathcal N$ such that
$d(p_\theta,p_{\tilde\theta})\le \varepsilon_n$. Therefore, for the constant $c>0$ above,
\begin{align*}
  \log N(c\varepsilon_n;P_n,d)
   & \le k_0s_n\log\frac{eJ_n}{k_0s_n}
  +\bigl(k_0s_nL_1+K_{\rm deep}\bigr)\log R_n \\
   & \le k_0s_n\log\frac{eJ_n}{k_0s_n}
  +\bigl(k_0s_nL_1+H_n\bar L^2\bigr)\log\frac{2M_n}{\delta_{1,n}} .
\end{align*}
Since $N(\varepsilon;P_n,d)$ is non-increasing in $\varepsilon$,
$\log N(\varepsilon_n;P_n,d)\le \log N(c\varepsilon_n;P_n,d)$.
By $\log M_n=O(\log n)$ and {$\log(1/\varepsilon_n)=O(\log n)$},
and noting that $k_0$ is fixed and $k_0s_n\le J_n$, we obtain
{
\[
  \log N(\varepsilon_n;P_n,d)
  \ \lesssim\
  s_n \log\frac{eJ_n}{s_n}
  +\bigl[s_nL_1+H_n\bar L^2\bigr]
  \Bigl(H_n\log n+H_n\log\bar L+\log(J_nL_1)\Bigr).
\]
}

By the rate condition of $\varepsilon_n$, the right-hand side is $O(n\varepsilon_n^2)$,
hence $\log N(\varepsilon_n;P_n,d)\le n\varepsilon_n^2$ for all large $n$,
verifying condition~(a).

\paragraph{Verification of condition (b).}

By the definition of $\Theta_n$, we have
\[
  \Pi(\Theta_n^c)\le T_{1,n}+T_{2,n},
  \qquad
  T_{1,n}:=\Pi(\|\boldsymbol{\theta}\|_\infty>M_n),\quad
  T_{2,n}:=\Pi(C(\boldsymbol{\theta})>k_0s_n).
\]
Note that, under the induced prior on densities,
\[
  \Pi(P_n^c)=\Pi\{\boldsymbol{\theta}\notin\Theta_n\}=\Pi(\Theta_n^c).
\]
Denote $\sigma_{\max }^{2}=\max\{\sigma^2,\sigma_{0,n}^2,\sigma_{1,n}^2\}$ and $K_n=(J_n+1)L_1+H_n\bar L^2$.
For every coordinate $\theta_j$,
$\Pr(|\theta_j|>M_n
  )\le 2\exp\!\bigl(-M_n
  ^{2}/2\sigma_{\max }^{2}\bigr)$.
With the condition (\ref{M_2}) on $\sigma_{\max }^{2}$ and by the union bound, we have
\[
  T_{1,n}\le 2K_n\exp\!\Bigl(-\frac{M_n
    ^{2}}{2\sigma_{\max }^{2}}\Bigr)
  \le \exp\!\bigl(-b_1\,n\varepsilon_n^{2}\bigr).
\]

For each column $g$, define
\[
  I_g:=\mathbf 1\Bigl\{\max_{1\le h\le L_1}|W_{1,hg}|\ge \delta'_n\Bigr\},
  \qquad C(\boldsymbol{\theta})=\sum_{g=1}^{J_n} I_g.
\]
Under the specific prior, the indicators $\{I_g\}_{g=1}^{J_n}$ are i.i.d.\ Bernoulli ($p_n$) with
\begin{equation}\label{p_n}
  p_n:=\Pr(I_g=1)=(1-\lambda_n)p_{0,n}+\lambda_n p_{1,n}\le p_{0,n}+\lambda_n,
\end{equation}
where
\[
  \begin{aligned}
    p_{0,n}
     & := \Pr\Bigl(\max_{1\le h\le L_1}|Z_h|\ge \delta'_n\Bigr),
    \qquad Z_h \overset{\mathrm{i.i.d.}}{\sim} N(0,\sigma_{0,n}^2), \\
    p_{1,n}
     & := \Pr\Bigl(\max_{1\le h\le L_1}|W_h|\ge \delta'_n\Bigr),
    \qquad W_h \overset{\mathrm{i.i.d.}}{\sim} N(0,\sigma_{1,n}^2).
  \end{aligned}
\]

By the union bound and Gaussian tails,
\[
  p_{0,n}\le 2L_1\exp\!\Bigl(-\frac{(\delta'_n)^2}{2\sigma_{0,n}^2}\Bigr),
  \qquad
  p_{1,n}\le 1.
\]

Under assumption (\ref{M_1}) for $\sigma_{0,n}^2$, we have
\[
  \frac{(\delta'_n)^2}{2\sigma_{0,n}^2}\ \ge\ \tau A_n+\log(4J_nL_1^2),
\]
with $A_n= H_n\log n+H_n\log\bar L+\log\{(J_n+1)L_1\}$,
and hence
\begin{equation}\label{q_0n}
  p_{0,n}\le 2L_1\exp\!\Bigl(-\frac{(\delta'_n)^2}{2\sigma_{0,n}^2}\Bigr)
  \le 2L_1\exp\!\bigl\{-\tau A_n-\log(4J_nL_1^2)\bigr\}
  = \frac{1}{2J_nL_1}\,e^{-\tau A_n}.
\end{equation}
Therefore $J_np_{0,n}\le \frac{1}{2L_1}e^{-\tau A_n}$, {which is $o(s_n)$} and in particular
implies $J_np_{0,n}\le \frac14 k_0s_n$ for all large $n$.
Moreover, under Assumption \ref{assumption_parameter}, we have $J_n\lambda_n\lesssim \bigl[(n\bar L)^{H_n}(J_n+1)L_1\bigr]^{-\tau'}$,
so for all large $n$,
{
\[
  J_n\lambda_n\le \frac{1}{4}k_0s_n .
\]
}
Combining the last two displays yields
\[
  J_n p_n \;\le\; J_n(p_{0,n}+\lambda_n)\;\le\;\frac{1}{2}k_0s_n .
\]

Let $q_n:=k_0s_n/J_n$ and assume $q_n\le 1/2$ for all large $n$.
Since $\{I_g\}_{g=1}^{J_n}$ are i.i.d.\ Bernoulli$(p_n)$, we have
\[
  C(\boldsymbol{\theta})=\sum_{g=1}^{J_n} I_g \ \sim\ \mathrm{Bin}(J_n,p_n).
\]
Under the above bounds, we have $p_n\le q_n/2$, hence $q_n>p_n$.
Applying \citet[Theorem~1]{zubkov2013complete} to
$X\sim\mathrm{Bin}(J_n,p_n)$ with $k=\lfloor J_n q_n\rfloor-1$ yields
\[
  \Pr\{C(\boldsymbol{\theta})\ge k_0s_n\}
  \;\le\;
  1-\Phi\!\Big(\sqrt{2J_n\,H\!\big(p_n,q_n\big)}\Big),
\]
where
\[
  H(p_n,q_n)
  = q_n\log\frac{q_n}{p_n}+(1-q_n)\log\frac{1-q_n}{1-p_n}
  = \operatorname{KL}(q_n\|p_n).
\]
Using the standard Gaussian tail bound $1-\Phi(t)\le e^{-t^2/2}$ for $t>0$,
we obtain
\[
  \Pr\{C(\boldsymbol{\theta})\ge k_0s_n\}
  \;\le\;
  \exp\!\big\{-J_n\,\operatorname{KL}(q_n\|p_n)\big\}.
\]
Moreover, since $p_n\le q_n/2$ and {$\log(q_n/p_n)\to\infty$} in our regime,
the negative term $(1-q_n)\log\{(1-q_n)/(1-p_n)\}$ is negligible compared to
$q_n\log(q_n/p_n)$, and in particular for all sufficiently large $n$,
\[
  \operatorname{KL}(q_n\|p_n)\ \ge\ \frac12\,q_n\log\frac{q_n}{p_n}.
\]
Therefore,
\[
  \Pr\{C(\boldsymbol{\theta})\ge k_0s_n\}
  \;\le\;
  \exp\!\Big\{-c\,k_0s_n\log\frac{q_n}{p_n}\Big\}
\]
for some absolute constant $c>0$.

Moreover, under Assumption \ref{assumption_parameter} and (\ref{p_n}) and (\ref{q_0n}), we can ensure
$p_n\lesssim e^{-\tau A_n}/J_n$ for some $\tau>0$, 
so that
\[
  \log\frac{q_n}{p_n}
  \ \gtrsim\
  \tau A_n+\log(k_0s_n)
  \ \gtrsim\ A_n
  \quad\text{for all large $n$}.
\]
Therefore,
\[
  -\log \Pr\{C(\boldsymbol{\theta})\ge k_0s_n\}
  \ \gtrsim\ k_0s_n\,A_n.
\]
% {\color{red}If the rate is chosen so that $n\varepsilon_n^2\lesssim s_nA_n$
% (e.g. in the regime $H_n\bar L^2\lesssim s_nL_1$),}
Then, under Assumption \ref{assumption_DNN}, we have $n\varepsilon_n^2\lesssim s_nA_n$, and thus
for some $b_2>0$,
\[
  T_{2,n}=\Pi\{C(\boldsymbol{\theta})>k_0s_n\}\le \exp(-b_2\,n\varepsilon_n^2),
\]
verifying condition~(b).

\paragraph{Verification of condition (c).}

We check condition (c) for $t=1$.
% Here, we take $\theta^*=\theta_n^\dagger$ to be an
% approximating parameter such that
% \[
%   \operatorname{supp}_{\mathrm{col}}(\theta_n^\dagger)=S^*_{J_n},\qquad \|\theta_n^\dagger\|_\infty\le E_n,
% \]
% and
% \[
%   \sup_{X\in\mathcal X}\bigl|\mu(\theta_n^\dagger,\eta(X))-\mu^*(X)\bigr|
%   \ \le\ \tilde\xi_n .
% \]
Consider the set
\[
  \mathcal A_n=\Bigl\{\boldsymbol{\theta}:
  \max_{g\in S^*_{J_n}}\max_{h\le L_1}|W_{1,hg}-W^*_{1,hg}|\le \omega_n,\
  \max_{g\notin S^*_{J_n}}\max_{h\le L_1}|W_{1,hg}|\le \omega_n',\
  \|\boldsymbol{\theta}_{\mathcal G}-\boldsymbol{\theta}^*_{\mathcal G}\|_\infty\le \omega_n
  \Bigr\},
\]
where $\omega_n=\frac{c_1\varepsilon_n}{[\,H_n(s_n+1)L_1(\bar{L})^{2(H_n-1)}(c_0E_n)^{H_n-1}\,]}$ and $\omega_n'$ are defined in (\ref{delta}) and (\ref{omega}), {and $\boldsymbol{\theta}^*$ is the network parameter vector of the DNN $F_{\boldsymbol{\theta}^*}$ obtained in Theorem \ref{theorem_1}.}
If $\boldsymbol{\theta}\in \mathcal A_n$, then by Lemma~\ref{lem:perturbation_sparse} we have
\[
\sup_{X\in\mathcal X}\bigl|\mu_{\boldsymbol{\theta}}(X)-\mu_{\boldsymbol{\theta}^*}(X)\bigr|=\sup_{\boldsymbol{x}\in[-1,1]^{J_n}}\bigl|F_{\boldsymbol{\theta}}(\boldsymbol{x})-F_{\boldsymbol{\theta}^*}(\boldsymbol{x})\bigr|
  \ \le\ 3c_1\varepsilon_n,
\]
where $\boldsymbol{x}=\boldsymbol{\eta}(X)$ and $\mu_{\boldsymbol{\theta}}(X)=F_{\boldsymbol{\theta}}(\boldsymbol{\eta}(X))$.
Define
\[
  \tilde\xi_n
  \;:=\;
  \inf_{\boldsymbol{\theta}:\ \supp_{\mathrm{col}}(\boldsymbol{\theta})=S^*_{J_n},\ \|\boldsymbol{\theta}\|_\infty\le E_n}
  \ \sup_{X\in\mathcal X}\bigl|\mu_{\boldsymbol{\theta}}(X)-\mu^*(X)\bigr|.
\]
Then, Theorem \ref{theorem_1} gives $\tilde\xi_n \lesssim
      H_n^{-2\alpha_1} + J_n^{-\alpha_\beta\alpha_1}$.

Since
\[
  \sup_{X\in\mathcal X}\bigl|\mu_{\boldsymbol{\theta}^*}(X)-\mu^*(X)\bigr|
  \ \le\ \tilde \xi_n,
\]
we have,
\[
 \sup_{X\in\mathcal X}\bigl|\mu_{\boldsymbol{\theta}}(X)-\mu^*(X)\bigr|
  \ \le\ 3c_1\varepsilon_n+\tilde \xi_n.
\]
For normal models, we obtain
\[
  d_1\bigl(p_{\mu_{\boldsymbol{\theta}}},p_{\mu^*}\bigr)
  \;\le\;
  C\,(1+o(1))\,\mathbb{E}_X\bigl[\mu_{\boldsymbol{\theta}}(X)-\mu^*(X)\bigr]^2 
  \;\le\;
  C\,(1+o(1))\,(3c_1\varepsilon_n +\tilde  \xi_n)^2,
  \quad\text{if }\boldsymbol{\theta}\in \mathcal A_n,
\]
for some constant $C$. Under Assumption \ref{assumption_DNN}, we have $n\varepsilon_n^2\ge M_0\,n\,\tilde{\xi}_n^2$ for large $M_0$. Thus for any small $b'>0$, condition (c) holds as long as $c_1$ is sufficiently small, {and the prior satisfies $-\log\Pi(\mathcal A_n)\le b'\,n\varepsilon_n^2$}.

Let $S^*_{J_n}\subset\{1,\dots,J_n\}$ be the true active column set with $|S^*_{J_n}|=s_n$,
and define the configuration $\boldsymbol{\gamma}^*$ by $\gamma_g^*=\mathbf 1\{g\in S^*_{J_n}\}$.
Write $K_{\rm deep}:=|\mathcal G|\le H_n\bar L^2$.

Since columns are independent under the hierarchical prior and $(\gamma_g)$ are i.i.d.,
\[
  \Pi(\mathcal A_n)\ \ge\ \Pi(\boldsymbol{\gamma}=\boldsymbol{\gamma}^*)\,
  \Pi(\mathcal A_n\mid \boldsymbol{\gamma}=\boldsymbol{\gamma}^*).
\]

To bound $\Pi(\mathcal{A}_n)$ from below, we consider the event where the selection indicators match the target indices exactly, i.e., $\gamma_g = 1$ if $g \in S^*_{J_n}$ and $\gamma_g = 0$ otherwise.

\medskip\noindent\textit{(i) Consider $\Pi(\boldsymbol{\gamma}=\boldsymbol{\gamma}^*)$. }
\[
  \Pi(\boldsymbol{\gamma}=\boldsymbol{\gamma}^*)=\lambda_n^{s_n}(1-\lambda_n)^{J_n-s_n},
  \quad\Rightarrow\quad
  -\log \Pi(\gamma=\gamma^*)
  \le s_n\log\frac1{\lambda_n}+ (J_n-s_n)\lambda_n .
\]
Under Assumption \ref{assumption_parameter}, $(J_n-s_n)\lambda_n=o(n\varepsilon_n^2)$ and
$s_n\log(1/\lambda_n)\lesssim s_n\{H_n\log n+H_n\log\bar L+\log(J_nL_1)\}$.

\medskip\noindent\textit{(ii) Consider $\Pi(\mathcal A_n\mid \boldsymbol{\gamma}=\boldsymbol{\gamma}^*)$.}

For $X\sim N(0,\sigma^2)$ and any $|a|\le E_n$,
\[
  \Pr(|X-a|\le \omega)\ \ge\ 2\omega\cdot \inf_{|u-a|\le \omega}\phi(u;0,\sigma^2)
  \ \ge\ c\,\frac{\omega}{\sigma}\exp\!\Bigl(-\frac{(E_n+\omega)^2}{2\sigma^2}\Bigr),
\]
hence
\[
  -\log \Pr(|X-a|\le \omega)
  \ \lesssim\ \log\frac{\sigma}{\omega}+\frac{(E_n+1)^2}{2\sigma^2}.
\]

Applying this bound to the active first-layer weights
($s_nL_1$ coordinates with slab variance $\sigma_{1,n}^2$) and to the deep parameters
($K_{\rm deep}$ coordinates with variance $\sigma^2$), we obtain
\begin{align*}
  -\log \Pi(\mathcal A_n\mid \gamma=\gamma^*)
   & \le
  C\Biggl[
    s_nL_1\Bigl\{\log\frac{\sigma_{1,n}}{\omega_n}+\frac{(E_n+1)^2}{2\sigma_{1,n}^2}\Bigr\}
    +
    K_{\rm deep}\Bigl\{\log\frac{\sigma}{\omega_n}+\frac{(E_n+1)^2}{2\sigma^2}\Bigr\}
    \Biggr]
  \\&\quad
  -\log \Pi\Bigl(\max_{g\notin S^*_{J_n}}\max_{h\le L_1}|W_{1,hg}|\le \omega_n'\,\Big|\,\gamma_g=0\ \forall g\notin S^*_{J_n}\Bigr).
\end{align*}

For the inactive part, by (\ref{M_2}) and (\ref{omega}), we have
\[
  \frac{(\omega_n')^2}{2\sigma_{0,n}^2}\ \ge\ \log(4J_nL_1^2),
\]
so for $Z\sim N(0,\sigma_{0,n}^2)$,
\[
  \Pr(|Z|>\omega_n')\le 2\exp\!\Bigl(-\frac{(\omega_n')^2}{2\sigma_{0,n}^2}\Bigr)
  \le \frac{1}{2J_nL_1^2}.
\]
Hence, for each $(g,h)$ with $g\notin S^*_{J_n}$,
\[
  \Pr(|W_{1,hg}|\le \omega_n'\mid \gamma_g=0)\ge 1-\frac{1}{2J_nL_1^2}.
\]
By independence over $(g,h)$,
\[
  \Pi\Bigl(\max_{g\notin S^*_{J_n}}\max_{h\le L_1}|W_{1,hg}|\le \omega_n'
  \,\Big|\,\gamma_g=0\ \forall g\notin S^*_{J_n}\Bigr)
  \ge \Bigl(1-\frac{1}{2J_nL_1^2}\Bigr)^{(J_n-s_n)L_1}
  \ge e^{-1}
\]
for all large $n$.

Combining the above bounds and using $K_{n}\le H_n\bar L^2$ and
$\log(1/\omega_n)=O\!\bigl(H_n\log n+H_n\log\bar L+\log(s_nL_1)\bigr)$,
we conclude that
\[
  -\log \Pi(\mathcal A_n)
  \ \le\
  C'\Bigl\{
  s_n\log\frac1{\lambda_n}
  +
  \bigl[s_nL_1+H_n\bar L^2\bigr]
  \bigl(H_n\log n+H_n\log\bar L+\log(s_nL_1)\bigr)
  \Bigr\}
  \ \le\ b'n\varepsilon_n^2,
\]
where the last inequality follows from the rate condition and the hyper-parameter bounds in Assumption \ref{assumption_parameter}.
Consequently, $\Pi(\mathcal A_n)\ge \exp(-b'n\varepsilon_n^2)$, verifying condition (c).

\subsection{Selection Consistency (Theorem \ref{thm:selection_consistency})}
Denote $q_j=\Pi(r_j=1|D_n)=\mathbb{E}(r_j|D_n)$ and $A_n(\varepsilon_n)=\Bigl\{\boldsymbol{\theta}:\ d\bigl(p_{\mu_{\boldsymbol{\theta}}},p_{\mu^*}\bigr)\ge \varepsilon_n\Bigr\}$.

\subsubsection{Proof of Theorem \ref{thm:selection_consistency}}

Fix $j\in\{1,\ldots,J_n\}$. By the definition of $q_j$, we have
  \[
    |q_j-e_j^*|
    = \Big|\mathbb{E}((\gamma_j-e_j^*)|D_n)\Big|
    \le \mathbb{E}\left(\left.|\gamma_j-e_j^*| \right|D_n\right).
  \]
  Split according to $A_n(4\varepsilon_n)$:
  \[
    \mathbb{E}\left(\left.|\gamma_j-e_j^*| \right|D_n\right)
    \le
    \mathbb{E}\Bigl[\left.|\gamma_j-e_j^*|\mathbf 1\{\theta\notin A_n(4\varepsilon_n)\}\right|D_n\Bigr]
    +
    \mathbb{E}\Bigl[\left.|\gamma_j-e_j^*|\mathbf 1\{\theta\in A_n(4\varepsilon_n)\right| D_n\}\Bigr].
  \]
  Since $|\gamma_j-e_j^*|\le 1$, the second term is bounded by
  $\Pi\!\bigl(A_n(4\varepsilon_n)\mid D_n\bigr)$, while the first term is controlled by
  $\rho_n(4\varepsilon_n)$ in Assumption \ref{assumption_identifiability}. Taking the maximum over $j$ gives
  \[
    \max_j |q_j-e_j^*|
    \le \rho_n(4\varepsilon_n) + \Pi\!\bigl(A_n(4\varepsilon_n)\mid D_n\bigr)
    \xrightarrow{P}0.
  \]
Then based on Theorem 2.3 stated in Sun et al. (2022), with an appropriate choice of prior hyperparameters, the estimated $\hat{q}_j=Pr(r_j=1|\hat{\boldsymbol{\theta}})$ based on the MAP estimate $\hat{\boldsymbol{\theta}}$ and $q_j$ are approximately the same as $n\rightarrow \infty$. Thus, $\hat{q}_j$ is also a consistent estimator of $e_j^*$, which proves (i). Parts (ii) and (iii) follow immediately from (i).

{To prove (iv),
recall that the estimated active region on the original domain is
$\widehat{\Omega}:=\bigcup_{j\in \widehat{S}_{1/2}} I_j$, and the population (truncated) active region at
resolution $J_n$ is $\Omega_{J_n}^*=\bigcup_{j\in S^*_{J_n}} I_j$.
By Parts (i)--(iii), $\Pr(\widehat{S}_{1/2}=S^*_{J_n})\to 1$, hence also
$\Pr(\widehat{\Omega}=\Omega_{J_n}^*)\to 1$.
Therefore, it suffices to show $|\Omega_{J_n}^*\Delta \Omega^*|\to 0$.

By the triangle inequality for symmetric differences,
\begin{equation}\label{eq:tri-decomp}
  |\Omega_{J_n}^*\Delta \Omega^*|
  \le
  |\Omega_{J_n}^*\Delta \Omega^*(2\kappa_{J_n})|
  +
  |\Omega^*(2\kappa_{J_n})\Delta \Omega^*|.
\end{equation}

For the first term, Lemma~\ref{prop:block1-main}(v) gives
\[
  |\Omega_{J_n}^*\Delta \Omega^*(2\kappa_{J_n})|
  \le
  |\Omega^*(\kappa_{J_n})\setminus \Omega^*(2\kappa_{J_n})|
  +
  \Bigl|\Omega^*_{J_n}(\kappa_{J_n})^{+c_{\rm loc}\Delta_{J_n}}\setminus \Omega^*(\kappa_{J_n})\Bigr|.
\]
The first summand is controlled by Assumption~\ref{assumption_smooth_coefficient}, since
$\Omega^*(\kappa_{J_n})\setminus \Omega^*(2\kappa_{J_n})
\subseteq \{t\in\Omega^*:|\beta^*(t)|\le 2\kappa_{J_n}\}$ and $2\kappa_{J_n}\downarrow 0$, hence
$|\Omega^*(\kappa_{J_n})\setminus \Omega^*(2\kappa_{J_n})|
\le |\{t\in\Omega^*:|\beta^*(t)|\le 2\kappa_{J_n}\}|\to 0$.

For the second summand, we relate $\Omega^*_{J_n}(\kappa)$ to $\Omega^*(\kappa)$ via a $\Delta_{J_n}$-enlargement.
Indeed, by definition $\Omega^*(\kappa)\subseteq \Omega^*_{J_n}(\kappa)$; moreover, since each spline support interval
$I_j$ has diameter at most $\Delta_{J_n}$, if $t\in I_j$ for some $j$ with $I_j\cap \Omega^*(\kappa)\neq\varnothing$,
then there exists $s\in I_j\cap \Omega^*(\kappa)$ such that
$|t-s|\le \mathrm{diam}(I_j)\le \Delta_{J_n}$, implying
$\Omega^*_{J_n}(\kappa)\subseteq \Omega^*(\kappa)^{+\Delta_{J_n}}$.
Consequently,
\[
  \Omega^*_{J_n}(\kappa)^{+c_{\rm loc}\Delta_{J_n}}
  \subseteq
  \bigl(\Omega^*(\kappa)^{+\Delta_{J_n}}\bigr)^{+c_{\rm loc}\Delta_{J_n}}
  =
  \Omega^*(\kappa)^{+(c_{\rm loc}+1)\Delta_{J_n}},
\]
and hence
\[
  \Bigl|\Omega^*_{J_n}(\kappa_{J_n})^{+c_{\rm loc}\Delta_{J_n}}\setminus \Omega^*(\kappa_{J_n})\Bigr|
  \le
  \bigl|\Omega^*(\kappa_{J_n})^{+(c_{\rm loc}+1)\Delta_{J_n}}\setminus \Omega^*(\kappa_{J_n})\bigr|.
\]
Since $\Omega^*(\kappa_{J_n})$ is a finite union of intervals, its boundary has finite cardinality; therefore there exists
$C_{\partial}>0$ such that for all $\delta>0$,
$\bigl|\Omega^*(\kappa_{J_n})^{+\delta}\setminus \Omega^*(\kappa_{J_n})\bigr|\le C_{\partial}\delta$.
Taking $\delta=(c_{\rm loc}+1)\Delta_{J_n}$ and using $\Delta_{J_n}\asymp J_n^{-1}\to 0$ yields
$\bigl|\Omega^*_{J_n}(\kappa_{J_n})^{+c_{\rm loc}\Delta_{J_n}}\setminus \Omega^*(\kappa_{J_n})\bigr|\to 0$.
This shows $|\Omega_{J_n}^*\Delta \Omega^*(2\kappa_{J_n})|\to 0$.

For the second term in \eqref{eq:tri-decomp}, note that
$\Omega^*\setminus \Omega^*(2\kappa_{J_n})\subseteq \{t\in\Omega^*:|\beta^*(t)|\le 2\kappa_{J_n}\}$, and hence
\[
  |\Omega^*(2\kappa_{J_n})\Delta \Omega^*|
  =
  |\Omega^*\setminus \Omega^*(2\kappa_{J_n})|
  \le
  \bigl|\{t\in\Omega^*:|\beta^*(t)|\le 2\kappa_{J_n}\}\bigr|
  \to 0
\]
by Assumption~\ref{assumption_smooth_coefficient}.
Combining the above bounds gives $|\Omega_{J_n}^*\Delta \Omega^*|\to 0$, and therefore
\[
  |\widehat{\Omega}\Delta \Omega^*|
  \le
  |\widehat{\Omega}\Delta \Omega_{J_n}^*| + |\Omega_{J_n}^*\Delta \Omega^*|
  \xrightarrow{P} 0,
\]
which completes the proof of (iv).
}

\section{Additional Simulation Analysis}
\subsection{Simulation Details and Supplementary Results (Data from Section~6.1)}

Following the functional-covariate generation mechanism in AdaFNN, we generate $X_i(\cdot)$ from a truncated cosine expansion on $\mathcal T=[0,1]$.
Let $\phi_1(t)\equiv 1$ and $\phi_k(t)=\sqrt{2}\cos\!\big((k-1)\pi t\big)$ for $k=2,\ldots,K$ with $K=50$, and set
$X_i(t)=\sum_{k=1}^{K} c_{ik}\phi_k(t)$.
We draw $c_{ik}=z_k r_{ik}$ with $r_{ik}\overset{\text{i.i.d.}}{\sim}\mathrm{Unif}(-\sqrt3,\sqrt3)$ and
$z_1=20$, $z_2=z_3=15$, and $z_k=1$ for $k\ge 4$.
We observe each curve on a discrete grid over $[0,1]$.

We consider three localized settings for $\beta(\cdot)$, which we label as
\emph{Simple}, \emph{Medium}, and \emph{Complex} according to increasing difficulty in region recovery.
Let $T(t;a,b):=(t-a)(b-t)\mathbf 1\{t\in[a,b]\}$ and let $\tilde T(t;a,b)$ denote its normalized version on $[a,b]$
(so that $\max_{t\in[a,b]}\tilde T(t;a,b)=1$). Thus, on each active interval, $\beta(\cdot)$ has a smooth quadratic bump
that vanishes at the endpoints, and multiple active intervals are represented by sums over disjoint bumps.
Specifically, we use:
(i) (\textbf{Simple}) a single centered bump on $[0.4,0.6]$, $\beta(t)=5\,\tilde T(t;0.4,0.6)$;
(ii) (\textbf{Medium}) a single boundary-adjacent bump on $[0.1,0.3]$, $\beta(t)=5\,\tilde T(t;0.1,0.3)$; and
(iii) (\textbf{Complex}) two separated narrow bumps with within-region oscillation.
Let $W=[a_1,b_1]\cup[a_2,b_2]$ with $(a_1,b_1)=(0.05,0.15)$ and $(a_2,b_2)=(0.75,0.85)$, and set
$\beta(t)=2.5\sum_{m=1}^{2}\tilde T(t;a_m,b_m)\sin\!\bigl(2\pi(t+0.1)\bigr)$. The shapes of the three $\beta(t)$ types are illustrated in Figure \ref{beta}.

\begin{figure}[htbp]
    \centering
    \includegraphics[width=1\textwidth]{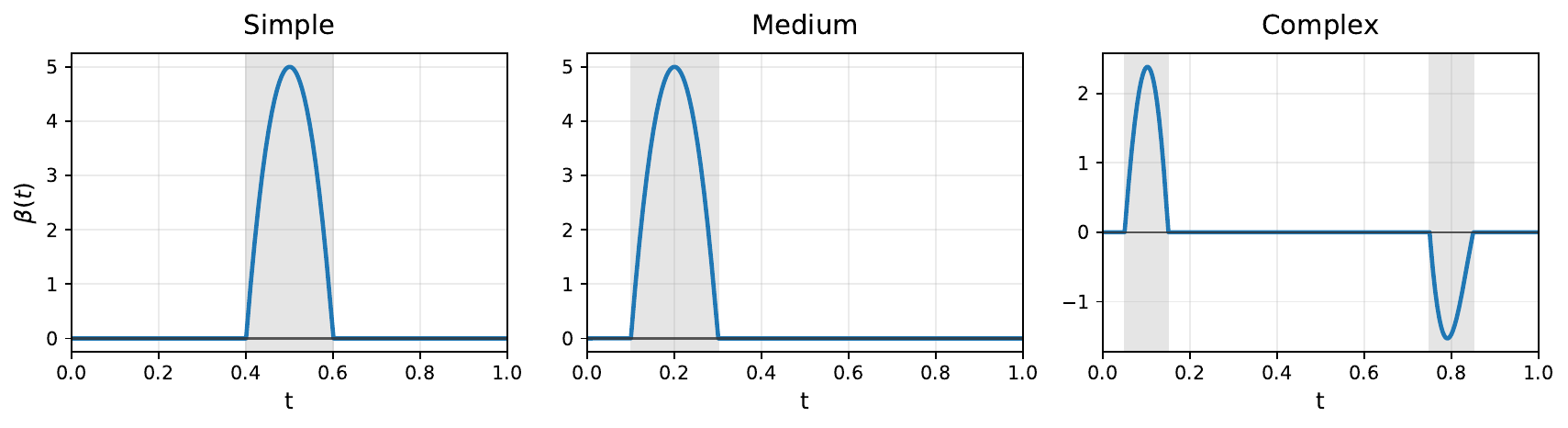} 
    \caption{Illustration of the three types of coefficient functions $\beta(t)$ used in simulations: Simple (single interior bump), Medium (single boundary bump), and Complex (two separated oscillating bumps).}
    \label{beta}
\end{figure}

To vary the nonlinearity of the response mechanism, we consider four choices of the link function $g^*$ with increasing complexity:
(i) linear $g^*(u)=u$;
(ii) logistic-type $g^*(u)=\{1+\exp(u)\}^{-1}$;
(iii) sinusoidal $g^*(u)=\sin(u)$;
and (iv) a composite link $g^*(u)=\tanh(u)+\sin(4u)\exp(-0.01u^2)$.

Finally, given a draw of $X_i(\cdot)$ and a choice of $(\beta,g^*)$, we generate the response as
\[
Y_i = g^*\!\left(\int_0^1 X_i(t)\beta(t)\,dt\right) + \varepsilon_i,
\qquad \varepsilon_i\overset{\text{i.i.d.}}{\sim}\mathcal N(0,\sigma_\varepsilon^2).
\]
To make noise levels comparable across settings, we calibrate additive Gaussian noise by a target signal-to-noise ratio,
$\mathrm{SNR}=\mathrm{Var}(\text{signal})/\mathrm{Var}(\text{noise})$.
In all simulations, the latent curves $X_i(\cdot)$ generate the responses above, but the learning algorithms observe only discretely sampled noisy curves obtained by adding i.i.d.\ Gaussian measurement noise at $\mathrm{SNR}=10$ on the observation grid.
The response noise variance $\sigma_\varepsilon^2$ is chosen so that the resulting responses attain the target response SNR (we consider $\mathrm{SNR}\in\{5,10\}$) in each scenario.
When required, we apply a denoising step to the noisy curve observations before constructing spline features. For each scenario, we simulate 100 replicates and summarize the results in Figures \ref{fig:f1-total}, \ref{fig:rmse-total}, \ref{fig:recall-total}, and \ref{fig:precision-total}.

\begin{figure*}[htbp]
  \centering
  \includegraphics[width=\textwidth]{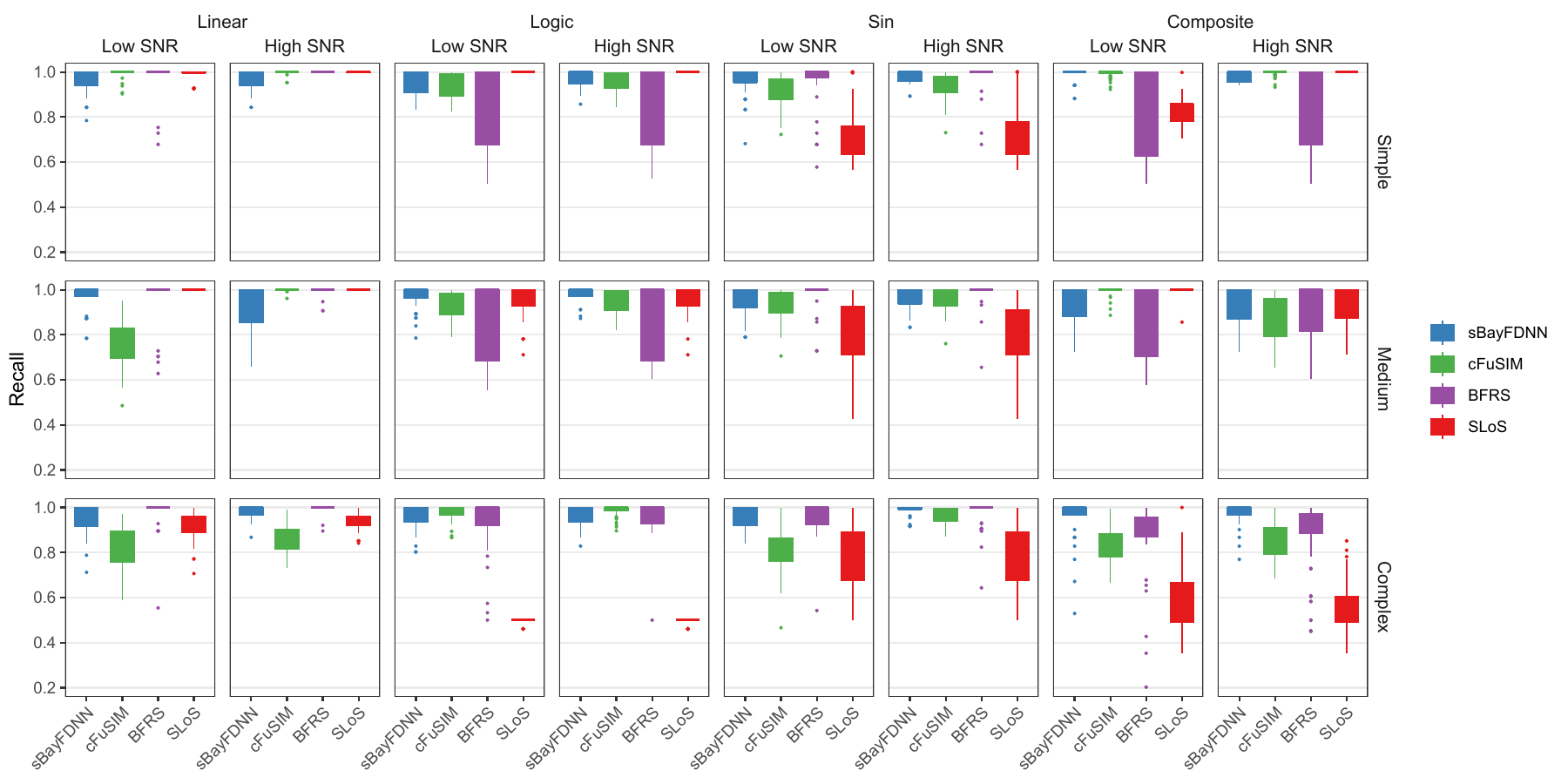}
  \caption{Recall values across $g$ functions, SNR settings, and $\beta(t)$ scenarios.}
  \label{fig:recall-total}
\end{figure*}

\begin{figure*}[htbp]
  \centering
  \includegraphics[width=\textwidth]{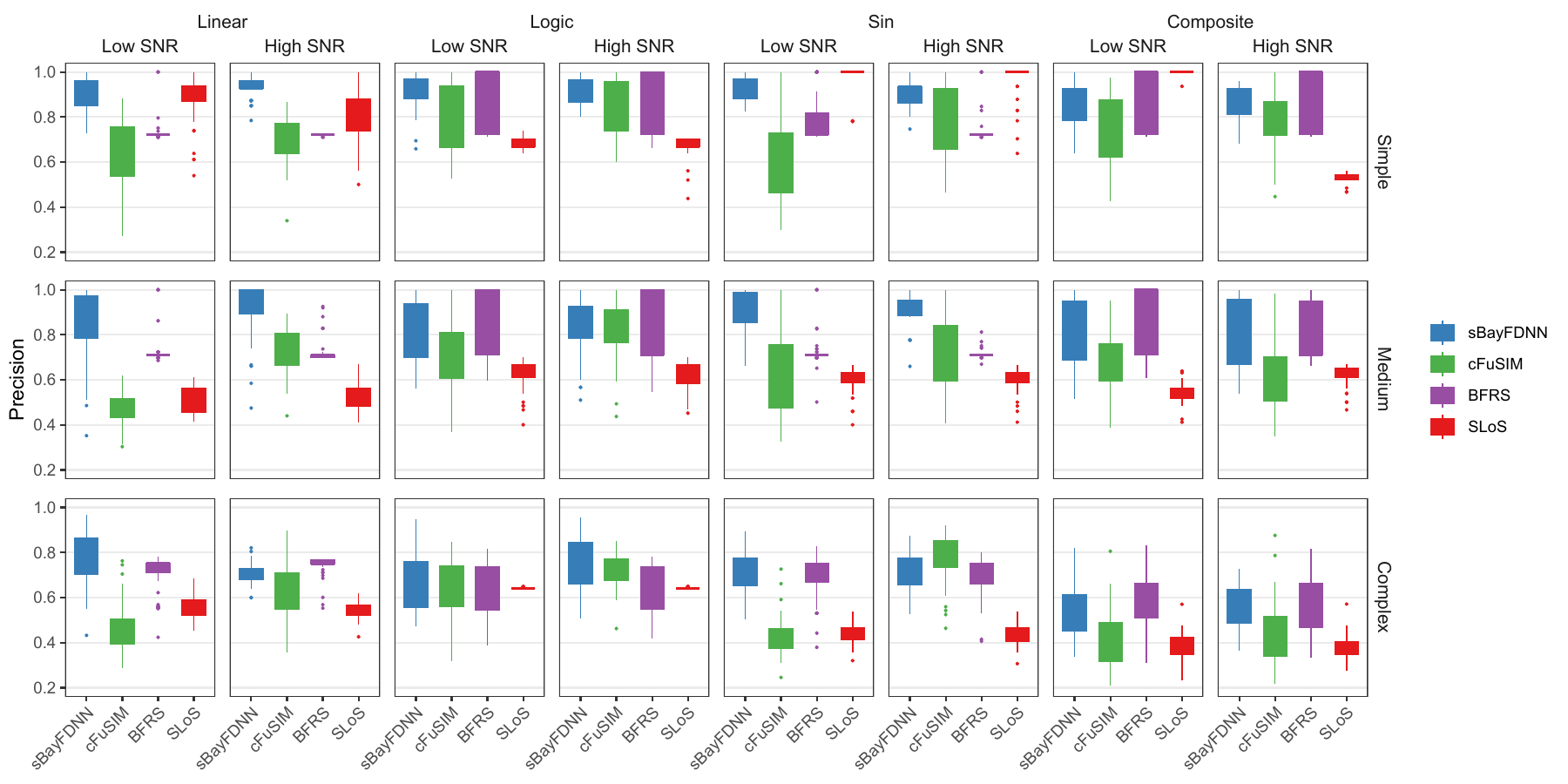}
  \caption{Precision values across $g$ functions, SNR settings, and $\beta(t)$ scenarios.}
  \label{fig:precision-total}
\end{figure*}
%%%%%%%%%%%%%%%%%%%%%%%%%%%%%%%%%%%%%%%%%%%%%%%%%%%%%%%%%%%%%%%%%%%%%%%%%%%%%%%
%%%%%%%%%%%%%%%%%%%%%%%%%%%%%%%%%%%%%%%%%%%%%%%%%%%%%%%%%%%%%%%%%%%%%%%%%%%%%%%

\subsection{Additional Simulation Analysis with Lower SNR}
\label{app:low-snr}
To further evaluate the proposed method under more challenging noise conditions, we conduct an additional analysis focusing on low-SNR scenarios (SNR = 2 and 3). Both prediction metrics and region‑selection metrics, reported as mean (SD) over the 50 replications, are shown in Table~\ref{tab:lowsnr_combined_vertical} for the scenario with composite link function $g$ and various settings of $\beta(t)$. As expected, all methods become more challenging, yet our approach remains stable and continues to outperform the baselines.
% The goal was to assess whether the method can still maintain reasonable prediction accuracy and support recovery performance when the signal is weak and the distinction between relevant and irrelevant regions becomes less pronounced.

\begin{table*}[htbp]
\centering
\small
\setlength{\tabcolsep}{6pt}
\caption{Simulation results under scenarios with composite $g$ and lower SNR levels. Prediction and selection metrics are reported as mean(standard deviation).}
\label{tab:lowsnr_combined_vertical}
\begin{tabular}{lllcccc}
\toprule
SNR & $\beta(t)$ Setting & Method & RMSE & F1 & Recall & Precision \\
\midrule
\multirow{18}{*}{2.0}
& \multirow{6}{*}{Complex}
& sBayFDNN & \textbf{1.019(0.071)} & \textbf{0.641(0.066)} & 0.811(0.238) & \textbf{0.601(0.169)} \\
& & FNN      & 1.100(0.108) & -- & -- & -- \\
& & AdaFNN   & 1.175(0.209) & -- & -- & -- \\
& & cFuSIM   & 1.023(0.068) & 0.488(0.053) & 0.766(0.124) & 0.358(0.079) \\
& & BFRS     & 1.073(0.081) & 0.517(0.268) & 0.503(0.307) & 0.583(0.270) \\
& & SLoS     & 1.072(0.078) & 0.621(0.060) & \textbf{0.912(0.106)} & 0.472(0.044) \\
\cmidrule(lr){2-7}
& \multirow{6}{*}{Medium}
& sBayFDNN & \textbf{0.853(0.042)} & \textbf{0.850(0.081)} & 0.964(0.060) & 0.771(0.132) \\
& & FNN      & 0.949(0.070) & -- & -- & -- \\
& & AdaFNN   & 0.887(0.144) & -- & -- & -- \\
& & cFuSIM   & 1.113(0.324) & 0.606(0.049) & \textbf{1.000(0.000)} & 0.437(0.048) \\
& & BFRS     & 0.946(0.051) & 0.735(0.260) & 0.721(0.298) & \textbf{0.798(0.310)} \\
& & SLoS     & 0.944(0.055) & 0.578(0.094) & 0.735(0.174) & 0.493(0.111) \\
\cmidrule(lr){2-7}
& \multirow{6}{*}{Simple}
& sBayFDNN & 0.866(0.058) & \textbf{0.924(0.053)} & 0.979(0.028) & 0.882(0.095) \\
& & FNN      & 0.944(0.054) & -- & -- & -- \\
& & AdaFNN   & \textbf{0.855(0.171)} & -- & -- & -- \\
& & cFuSIM   & 0.883(0.116) & 0.855(0.077) & \textbf{0.988(0.027)} & 0.754(0.124) \\
& & BFRS     & 0.959(0.065) & 0.767(0.098) & 0.691(0.194) & 0.930(0.117) \\
& & SLoS     & 0.960(0.065) & 0.901(0.032) & 0.831(0.049) & \textbf{0.986(0.044)} \\
\midrule
\multirow{18}{*}{3.0}
& \multirow{6}{*}{Complex}
& sBayFDNN & \textbf{0.908(0.063)} & \textbf{0.651(0.068)} & 0.818(0.127) & 0.549(0.069) \\
& & FNN      & 0.998(0.080) & -- & -- & -- \\
& & AdaFNN   & 0.964(0.185) & -- & -- & -- \\
& & cFuSIM   & 0.926(0.078) & 0.530(0.071) & 0.805(0.080) & 0.395(0.090) \\
& & BFRS     & 0.974(0.076) & 0.562(0.256) & 0.606(0.323) & \textbf{0.556(0.225)} \\
& & SLoS     & 0.972(0.068) & 0.609(0.057) & \textbf{0.912(0.102)} & 0.459(0.047) \\
\cmidrule(lr){2-7}
& \multirow{6}{*}{Medium}
& sBayFDNN & 0.742(0.040) & \textbf{0.843(0.079)} & 0.964(0.060) & \textbf{0.761(0.131)} \\
& & FNN      & 0.845(0.064) & -- & -- & -- \\
& & AdaFNN   & \textbf{0.729(0.070)} & -- & -- & -- \\
& & cFuSIM   & 0.839(0.245) & 0.636(0.137) & 0.835(0.129) & 0.540(0.189) \\
& & BFRS     & 0.847(0.052) & 0.724(0.257) & 0.741(0.305) & 0.748(0.292) \\
& & SLoS     & 0.847(0.051) & 0.455(0.082) & \textbf{1.000(0.000)} & 0.298(0.070) \\
\cmidrule(lr){2-7}
& \multirow{6}{*}{Simple}
& sBayFDNN & 0.753(0.055) & \textbf{0.908(0.055)} & \textbf{0.989(0.023)} & 0.844(0.093) \\
& & FNN      & 0.837(0.058) & -- & -- & -- \\
& & AdaFNN   & 0.802(0.187) & -- & -- & -- \\
& & cFuSIM   & \textbf{0.734(0.062)} & 0.662(0.119) & \textbf{0.989(0.029)} & 0.497(0.137) \\
& & BFRS     & 0.864(0.064) & 0.826(0.032) & 0.787(0.141) & 0.910(0.124) \\
& & SLoS     & 0.865(0.062) & 0.874(0.088) & 0.839(0.075) & \textbf{0.947(0.168)} \\
\bottomrule
\end{tabular}
\end{table*}

\subsection{Additional Simulation Analysis with Model-Misspecification and Non-Gaussian Noise}
\label{app:stress-tests}

To further assess robustness beyond the main simulation settings, we consider two additional scenarios.
The first is an FGAM-type mean misspecification setting: 
$Y_i^{\mathrm{clean}}=\int_0^1 F(X_i(t),t)\,dt$ where $F(x,t)= -0.5+\exp\!\left[-\left(\frac{x}{5.0}\right)^2-\left(\frac{t-0.5}{0.3}\right)^2\right]$,
and the second is a heavy-tailed error setting with Student-$t$ noise (degrees of freedom $=3$) under the original single-index mean with simple $\beta(t)$ and composite link function $g$.
For both scenarios, we evaluate performance at SNR $=5$ and SNR $=10$.
Tables~\ref{tab:stress-fgam} and~\ref{tab:stress-nongaussian} report the corresponding prediction and selection results as mean(standard deviation) over these 50 replications.
\begin{table*}[htbp]
\centering
\small
\setlength{\tabcolsep}{4pt}
\caption{Simulation results for the stress setting with FGAM mean under two noise levels. Prediction and selection metrics are reported as mean(standard deviation).}
\label{tab:stress-fgam}
\begin{tabular}{c lcccc}
\toprule
SNR & Method & RMSE & F1 & Recall & Precision \\
\midrule
\multirow{6}{*}{5}
& sBayFDNN & \textbf{0.049(0.006)} & \textbf{0.907(0.196)} & \textbf{0.873(0.268)} & \textbf{1.000(0.000)} \\
& FNN      & 0.049(0.004) & -- & -- & -- \\
& AdaFNN   & 0.051(0.002) & -- & -- & -- \\
& cFuSIM   & 0.051(0.003) & 0.824(0.006) & 0.701(0.009) & \textbf{1.000(0.000)} \\
& BFRS     & 0.053(0.003) & 0.000(0.000) & 0.000(0.000) & 0.000(0.000) \\
& SLoS     & 0.053(0.003) & 0.000(0.000) & 0.000(0.000) & 0.000(0.000) \\
\cmidrule(lr){1-6}
\multirow{6}{*}{10}
& sBayFDNN & \textbf{0.040(0.006)} & \textbf{0.806(0.246)} & \textbf{0.736(0.331)} & \textbf{1.000(0.000)} \\
& FNN      & 0.041(0.003) & -- & -- & -- \\
& AdaFNN   & 0.049(0.002) & -- & -- & -- \\
& cFuSIM   & 0.042(0.006) & 0.700(0.061) & 0.542(0.074) & \textbf{1.000(0.000)} \\
& BFRS     & 0.051(0.003) & 0.000(0.000) & 0.000(0.000) & 0.000(0.000) \\
& SLoS     & 0.051(0.003) & 0.000(0.000) & 0.000(0.000) & 0.000(0.000) \\
\bottomrule
\end{tabular}
\end{table*}
\begin{table*}[htbp]
\centering
\small
\setlength{\tabcolsep}{4pt}
\caption{Simulation results for the stress setting with non-Gaussian error, simple $\beta(t)$, and composite link function $g$ under two noise levels. Prediction and selection metrics are reported as mean(standard deviation).}
\label{tab:stress-nongaussian}
\begin{tabular}{c lcccc}
\toprule
SNR & Method & RMSE & F1 & Recall & Precision \\
\midrule
\multirow{6}{*}{5}
& sBayFDNN & \textbf{0.658(0.077)} & \textbf{0.921(0.044)} & 0.953(0.053) & \textbf{0.900(0.092)} \\
& FNN      & 0.704(0.084) & -- & -- & -- \\
& AdaFNN   & 0.796(0.295) & -- & -- & -- \\
& cFuSIM   & 0.661(0.093) & 0.736(0.138) & 0.988(0.035) & 0.602(0.169) \\
& BFRS     & 0.777(0.076) & 0.802(0.054) & 0.852(0.196) & 0.811(0.137) \\
& SLoS     & 0.777(0.077) & 0.625(0.021) & \textbf{1.000(0.000)} & 0.455(0.022) \\
\cmidrule(lr){1-6}
\multirow{6}{*}{10}
& sBayFDNN & \textbf{0.553(0.066)} & \textbf{0.944(0.025)} & 0.973(0.038) & \textbf{0.921(0.061)} \\
& FNN      & 0.614(0.067) & -- & -- & -- \\
& AdaFNN   & 0.579(0.055) & -- & -- & -- \\
& cFuSIM   & 0.585(0.063) & 0.751(0.125) & 0.989(0.032) & 0.620(0.158) \\
& BFRS     & 0.712(0.066) & 0.830(0.045) & 0.854(0.166) & 0.858(0.150) \\
& SLoS     & 0.713(0.065) & 0.389(0.039) & \textbf{1.000(0.001)} & 0.242(0.030) \\
\bottomrule
\end{tabular}
\end{table*}

Under the FGAM mean misspecification setting, sBayFDNN achieves the best prediction performance at both SNR levels and also delivers the strongest support recovery overall.
In particular, it attains the highest F1 among the competing region-selection methods, indicating that the proposed approach remains effective even when the true mean structure departs from the assumed single-index form.
The results therefore suggest a useful degree of robustness to mean misspecification.

Under the heavy-tailed error setting, sBayFDNN again performs strongly across both prediction and selection metrics.
It achieves the best or nearly best prediction accuracy and yields the highest F1 and precision among the region-selection competitors.
Although performance differences are naturally smaller in some prediction metrics, the overall pattern indicates that the proposed method remains robust when the Gaussian noise assumption is violated by heavy-tailed disturbances.

\section{Details for the Real-world Datasets}

\paragraph{Implementation details.}
Relative to the simulation defaults in Appendix~\ref{app:hyperparams}, we keep the same network architecture and use mini-batch size 32.
For smaller-sample real datasets, we recommend a less aggressive first-layer sparsification prior to mitigate underfitting; accordingly, we use $(\lambda_n,\sigma_{0,n}^2,\sigma_{1,n}^2)=(10^{-1},10^{-5},5\times 10^{-2})$ for Tecator, Bike, and IHPC.
(Unless otherwise stated, other hyperparameters follow Appendix~\ref{app:hyperparams}.)

\textbf{ECG.}
The ECG dataset is from the EchoNext dataset on PhysioNet \citep{PhysioNet-echonext-1.1.0,goldberger2000physiobank} and downloaded 
from https://physionet.org/content/echonext/1.1.0/.
To reduce phase variability, we detect R-peaks and align each waveform to the detected R-peak, extract a fixed-length beat-centered window, and resample it to a common grid of length $L=256$ at 250\,Hz, following standard ECG preprocessing practice \cite{kachuee2018ecg,makowski2021neurokit2}.
To better approximate an i.i.d.\ sample, we remove repeated measurements from the same subject; after deduplication, the final sample sizes are 26{,}192/4{,}618/5{,}434 for train/validation/test, equal to the numbers of unique patients in each split.
The window length is selected using the training split only by screening candidate pre/post windows and choosing the configuration that minimizes variability of the aligned R-peak location across subjects (trimmed standard deviation below $0.03$ on the normalized phase), which yields a symmetric window of 0.3\,s before and 0.3\,s after the R-peak.

To assess region identification, we define silver-standard intervals on the original domains and then map them to the normalized domain $\mathcal T=[0,1]$ induced by our fixed-window. For ECG, we use a $120$\,ms window \citep{yu2003high,HUMMEL2009553} centered at the R peak, i.e., $[-0.06,0.06]$ seconds relative to the R peak, as a silver-standard proxy for the QRS complex extent. 

To adjust for scalar covariates, we residualize the response by fitting an OLS regression on the training split (age, sex, acquisition year, location setting, and race/ethnicity) and applying the fitted adjustment to validation/test splits.
For ECG, we use a higher-capacity network (7 hidden layers of width 512; total depth 8) trained with mini-batch size 512 and learning rate $5\times 10^{-4}$; we consider $\mathcal J=\{180,200,220\}$ and use B-splines of degree 8 to accommodate the larger sample size and sharply localized QRS morphology. Here, a larger $J_n$ is used because the signal contains a narrow peak; smaller $J_n$ would yield overly broad selections. Table~\ref{tab:ecg-jn-detailed} reports the ECG sensitivity results over different spline resolutions $J_n$; all metrics are computed as described in Appendix~\ref{app:jn-exploration-case3}. It is observed that the selected set of candidate $J_n$ values achieves a good balance between recall and precision.

\begin{table*}[htbp]
\centering
\small
\setlength{\tabcolsep}{4pt}
\caption{Results for ECG dataset under different $J_n$ settings.}
\label{tab:ecg-jn-detailed}
\begin{tabular}{cccccccc}
\toprule
$J_n$ & MeanLen & Curve Roughness & RMSE & F1 & Recall & Precision & MinLen/Mesh \\
\midrule
20  & 1.0000 & 0.0000 & 14.2331 & 0.3333 & 1.0000 & 0.2000 & 20.0000 \\
40  & 1.0000 & 0.0410 & 12.8696 & 0.3333 & 1.0000 & 0.2000 & 40.0000 \\
60  & 1.0000 & 0.1559 & 12.3688 & 0.3333 & 1.0000 & 0.2000 & 60.0000 \\
80  & 1.0000 & 0.1063 & 12.3784 & 0.3333 & 1.0000 & 0.2000 & 80.0000 \\
100 & 0.2935 & 0.1818 & 12.1960 & 0.5083 & 1.0000 & 0.3407 & 1.0870 \\
120 & 0.1845 & 0.2235 & 12.1353 & 0.5308 & 1.0000 & 0.3613 & 1.0714 \\
140 & 0.1591 & 0.1899 & 12.0042 & 0.5906 & 1.0000 & 0.4190 & 1.0606 \\
160 & 0.1066 & 0.1660 & 12.1815 & 0.5458 & 1.0000 & 0.3753 & 1.0526 \\
180 & 0.1570 & 0.1251 & 11.9560 & 0.6359 & 1.0000 & 0.4758 & 20.9302 \\
200 & 0.0955 & 0.1296 & 11.9265 & 0.6039 & 0.7344 & 0.4216 & 9.3750 \\
220 & 0.1486 & 0.1233 & 12.0542 & 0.5979 & 0.8302 & 0.4156 & 22.8302 \\
240 & 0.0733 & 0.0887 & 12.0183 & 0.4764 & 0.5000 & 0.4549 & 9.3103 \\
\bottomrule
\end{tabular}
\end{table*}

%  (ECG) and wavelength normalization (Tecator)

% To assess region identification, we define widely-used silver-standard intervals on the original domains: for ECG, a $120$ms window centered at the R-peak (i.e., $[-0.06, 0.06]$ seconds) as a proxy for the QRS complex \citep{yu2003high, HUMMEL2009553}; for Tecator, the water-absorption band around 970–980 nm, specifically the interval $[965, 985]$nm \citep{VANKOLLENBURG2021121865}. These intervals are then mapped to the normalized domain $\mathcal{T} = [0,1]$ via fixed-window (ECG) and wavelength (Tecator) normalization. Further details on the datasets are provided in the Appendix. 

\textbf{Tecator.}
The Tecator dataset is downloaded from https://lib.stat.cmu.edu/datasets/tecator.
It contains near-infrared absorbance spectra of 240 meat samples measured on 100 wavelength channels ranging from 850\,nm to 1{,}050\,nm, together with moisture (water), fat and protein percentages determined by analytic chemistry.
We use water as the response and follow the official split (129 training, 43 monitoring/validation, and 43 testing samples).
Relative to the simulation defaults in Appendix~\ref{app:hyperparams}, we keep the same network architecture but use mini-batch size 32. 
We use $\mathcal J=\{80,100,120\}$. We use the water-related absorption band around 970–980 nm\citep{VANKOLLENBURG2021121865} and define the silver-standard wavelength interval as $[965,985]$\,nm. The interval is then mapped to the normalized domain $\mathcal{T} = [0,1]$ via  wavelength normalization.

\paragraph{\textbf{Bike.}}
We use the Bike Sharing dataset \citep{bike_sharing_275} in its hourly-resolution form and represent each day as a functional observation with $T=24$ equally spaced time points on $\mathcal T=[0,1]$. The dataset  can be obtained from {https://archive.ics.uci.edu/dataset/275/bike\%2Bsharing\%2Bdataset}.
We define the response as the total demand over the next 7 days and adopt a chronological train/validation/test split with sample sizes $453/97/98$ (total $N=648$), covering the date range 2011-01-16 to 2012-12-30 (train end: 2012-06-12; validation end: 2012-09-17). We use $\mathcal J=\{8,10,12,14\}$.

\paragraph{\textbf{IHPC.}}
We use the Individual Household Electric Power Consumption (IHPC) dataset from the UCI Machine Learning Repository https://archive.ics.uci.edu/dataset/235/individual+household+electric+power+consumption \citep{individual_household_electric_power_consumption_235}. 
We use the minute-averaged global active power trajectory as the functional input, yielding daily curves of length $T=1440$ on $\mathcal T=[0,1]$. 
After restricting to complete days and constructing next-day prediction pairs, we obtain $N=1290$ samples and use a chronological train/validation/test split with sample sizes $672/329/289$, spanning 2006-12-17 to 2010-11-24.We use $\mathcal J=\{15,20,30,40\}$.

Table~\ref{tab:runtime-ihpc-methods} reports runtime on the IHPC dataset. 
Under this high-resolution setting, sBayFDNN remains practically feasible, with an average runtime of about 674 seconds on the reported hardware configuration.
% Although the proposed procedure includes multiple random restarts for more stable optimization, the additional computational burden remains manageable in practice.
Although sBayFDNN is not the fastest method overall, but it is faster than BFRS and substantially more practical than the available CPU-only FNN implementation and cFuSIM under default dense-grid settings. This suggests that the proposed projection-based framework remains computationally feasible for high-resolution functional inputs, even when repeated fitting is incorporated.

\begin{table}[h]
\centering
\small
\setlength{\tabcolsep}{5pt}
\caption{IHPC runtime (mean (sd), seconds) comparison across methods under (AMD Ryzen 5 5600H CPU, NVIDIA RTX 3050 Ti Laptop GPU with 4GB VRAM, and 16GB RAM).}
\label{tab:runtime-ihpc-methods}
\begin{tabular}{lcccccc}
\toprule
& sBayFDNN & AdaFNN & FNN & cFuSim & BFRS & SLoS \\
\midrule
Runtime  & 673.641 (27.851) & 513.336 (23.642) & $>$20 min & $>$30 min & 819.805 (37.151) & 104.144 (12.219) \\
\bottomrule
\end{tabular}
% \vspace{2pt}

% \footnotesize{FNN (available R/Keras implementation) is CPU-only and substantially slower (typically $>20$ min/fit); cFuSIM becomes prohibitively slow on dense grids under default basis settings.}
\end{table}

\end{document}